\documentclass[10pt]{article}
\usepackage[margin = 1in]{geometry}
\usepackage{hyperref}
\usepackage{nameref}
\hypersetup{colorlinks,
            linkcolor=blue,
            citecolor=blue,
            urlcolor=magenta,
            linktocpage,
            plainpages=false}
   
\usepackage{authblk}  
\usepackage{xkeyval}
\usepackage{calc}
\usepackage{etoolbox}
\PassOptionsToPackage{x11names}{xcolor}
\usepackage{xcolor}
\renewenvironment{abstract}
  {{\centering\large\bfseries Abstract\par}\vspace{0.7ex}%
    \bgroup
       \leftskip 20pt\rightskip 20pt\small\noindent\ignorespaces}%
  {\par\egroup\vskip 0.25ex}

\newenvironment{keywords}
{\bgroup\leftskip 20pt\rightskip 20pt \small\noindent{\bfseries
Keywords:} \ignorespaces}%
{\par\egroup\vskip 0.25ex}

\newlength\aftertitskip     \newlength\beforetitskip
\newlength\interauthorskip  \newlength\aftermaketitskip

\newcommand{\artappendix}{\par
  \setcounter{section}{0}
  \setcounter{subsection}{0}
  \def\thesection{\Alph{section}}
  \def\presectionnum{Appendix~}%
}

\let\appendix\artappendix
\usepackage{mathtools}
\usepackage{enumitem}
\usepackage{physics}
\usepackage[ruled,algo2e]{algorithm2e}
\usepackage{tikz}
\usepackage{breqn}
\usepackage{soul}
\usepackage{mwe}
\usepackage{cancel}
\usepackage{algpseudocode}
\usepackage{thmtools}
\usepackage{thm-restate}
\usepackage{cleveref}
\usepackage{mathtools}
\usepackage{bbm}
\usepackage{enumitem}
\usepackage{amsmath, amsthm}
\usepackage{amssymb}
\usepackage{natbib}
\usepackage{graphicx}
\usepackage{url}
\DeclareMathOperator*{\argmin}{arg\,min}

\DeclareMathOperator*{\mix}{mix}
\newcommand{\nosemic}{\SetEndCharOfAlgoLine{\relax}}
\newtheorem{assumption}{Assumption}

\newtheorem{rem}{Remark}

\newcommand{\at}[2][]{#1|_{#2}}

\newcommand{\E}{\mathbb{E}}

\newcommand{\Pc}{\mathcal{P}}

\newcommand{\Rc}{\mathcal{R}}
\newcommand{\Sc}{\mathcal{S}}
\newcommand{\Ac}{\mathcal{A}}

\newcommand{\Prob}{\mathbb{P}}

\usepackage{times}

\title{\vspace{-2cm}{\Large{\bfseries{A Finite Time Analysis of Temporal Difference Learning With Linear Function Approximation}} }}
\author{Jalaj Bhandari}
\author{Daniel Russo}
\author{Raghav Singal}
\affil{Columbia University}
\date{}

\begin{document}
% \begin{center}
% {\large{\textbf{A Finite Time Analysis of Temporal Difference Learning With Linear Function Approximation}} \vspace*{10pt}}
% \end{center}

\maketitle

\begin{abstract}
\noindent Temporal difference learning (TD) is a simple iterative algorithm used to estimate the value function corresponding to a given policy in a Markov decision process. Although TD is one of the most widely used algorithms in reinforcement learning, its theoretical analysis has proved challenging and few guarantees on its statistical efficiency are available. In this work, we provide a \emph{simple and explicit finite time analysis} of temporal difference learning with linear function approximation. Except for a few key insights, our analysis mirrors standard techniques for  analyzing stochastic gradient descent algorithms, and therefore inherits the simplicity and elegance of that literature. Final sections of the paper show how all of our main results extend to the study of TD learning with eligibility traces, known as TD($\lambda$), and to Q-learning applied in high-dimensional optimal stopping problems.
\end{abstract}


\section{Introduction} 
\label{sec:introduction}
Originally proposed by \cite{sutton1988learning}, temporal difference  learning (TD) is one of the most widely used reinforcement learning algorithms and a foundational idea on which more complex methods are built. %\cite{sutton1998reinforcement, sutton2009convergent, sutton2009fast, bhatnagar2009convergent, seijen2014true, mnih2015human, mnih2016asynchronous}. 
The algorithm operates on a stream of data generated by applying some policy to a poorly understood Markov decision process. The goal is to learn an approximate value function, which can then be used to track the net present value of future rewards as a function of the system's evolving state. TD maintains a parametric approximation to the value function, making a simple incremental update to the estimated parameter vector each time a state transition occurs.  

While easy to implement, theoretical analysis of TD is subtle. Reinforcement learning researchers in the 1990s gathered both limited convergence guarantees \citep{jaakkola1994convergence} and examples of divergence \citep{baird1995residual}. Many issues were then clarified in the work of \cite{tsitsiklis1997analysis}, which establishes precise conditions for the asymptotic convergence of TD with linear function approximation and gives examples of divergent behavior when key conditions are violated. With guarantees of asymptotic convergence in place, a natural next step is to understand the algorithm's statistical efficiency. How much data is required to guarantee a given level of accuracy? Can one give uniform bounds on this, or could data requirements explode depending on the problem instance? Twenty years after the work of \cite{tsitsiklis1997analysis}, such questions remain largely unsettled. 

\subsection{Contributions}
This paper develops a \emph{simple and explicit non-asymptotic analysis of TD with linear function approximation}. The resulting guarantees provide assurances of robustness. They explicitly bound the worst-case dependence on problem features like the discount factor, the conditioning of the feature covariance matrix, and the mixing time of the underlying Markov chain. Our analysis reveals rigorous connections between TD and stochastic gradient descent algorithms, provides a template for finite time analysis of incremental algorithms with Markovian noise, and applies without modification to analyzing a class of high-dimensional optimal stopping problems. 
% Additionally, several insights and lemmas in our analysis may be of broader interest. 
We elaborate on these contributions below.  

\begin{itemize}
\item \emph{Links with gradient descent:}  Despite a cosmetic connection to stochastic gradient descent (SGD), incremental updates of TD are not (stochastic) gradient steps with respect to any fixed loss function. It is therefore difficult to show that it makes consistent, quantifiable, progress toward its asymptotic limit point. Nevertheless, Section \ref{sec:mean_path} shows that expected TD updates obey crucial properties mirroring those of gradient descent on a particular quadratic loss function. In a model where the observations are corrupted by i.i.d.\,noise, these gradient-like properties of TD allow us to give state-of-the-art convergence bounds by essentially mirroring standard analyses of stochastic gradient descent (SGD). This approach may be of broader interest as SGD analyses are commonly taught in machine learning courses and serve as a launching point for a much broader literature on first-order optimization. Rigorous connections with the optimization literature can facilitate research on principled improvements to TD. 
\item \emph{Non-asymptotic treatment with Markovian noise:} TD is usually applied online to a single Markovian data stream. However, to our knowledge, there has been no successful\footnote{This was previously attempted by \cite{korda2015td}, but critical errors were shown by \cite{2017kordaissues}.} non-asymptotic analysis in the setting with Markovian observation noise. Instead, many papers have studied such algorithms under the  simpler i.i.d noise model mentioned earlier \citep{sutton2009convergent, sutton2009fast, liu2015finite, touati2017convergent, dalal2017concentration, lakshminarayanan2018linear}.  One reason is that the dependent nature of the data introduces a substantial technical challenge: the algorithm's updates are not only noisy, but can be severely biased. We use information theoretic techniques to control the magnitude of bias, yielding bounds that are essentially scaled by a factor of the mixing time of the underlying Markov process relative to those attained for i.i.d.\,model. Our analysis in this setting applies only to a variant of TD that projects the iterates onto a norm ball. This projection step imposes a uniform bound on the noise of gradient updates, which is needed for tractability. For similar reasons, projection operators are widely used throughout the stochastic approximation literature \citep[Section 2]{kushner2010stochastic}.
\item \emph{An extendable approach:} Much of the paper focuses on analyzing the most basic temporal difference learning algorithm, known as TD(0). We also extend this analysis to other algorithms. First, we  establish convergence bounds for temporal difference learning with eligibility traces, known as TD($\lambda$). This is known to often outperform TD(0) \citep{sutton1998reinforcement}, but a finite time analysis is more involved. Our analysis also applies without modification to Q-learning for a class of high-dimensional optimal stopping problems. Such problems have been widely studied due to applications in the pricing of financial derivatives \citep{tsitsiklis1999optimal, andersen2004primal, haugh2004pricing, desai2012pathwise, goldberg2018beating}. For our purposes, this example illustrates more clearly the link between value prediction and decision-making. It also shows our techniques extend seamlessly to analyzing an instance of non-linear stochastic approximation. To our knowledge, no prior work has provided non-asymptotic guarantees for either TD($\lambda$) or Q-learning with function approximation. 
\end{itemize}

\subsection{Related Literature}
\paragraph{Non-asymptotic analysis of TD(0):}
There has been very little non-asymptotic analysis of TD(0). To our knowledge, \cite{korda2015td} provided the first finite time analysis. However, several serious errors in their proofs were pointed out by \cite{2017kordaissues}. A very recent work by \cite{dalal2017finite} studies TD(0) with linear function approximation in an i.i.d.\,observation model, which assumes sequential observations used by the algorithm are drawn independently from their steady-state distribution. They focus on analysis with problem independent step-sizes of the form $1/T^{\sigma}$ for a fixed $\sigma\in(0,1)$ and establish that mean-squared error convergences at a rate\footnote{In personal communication, the authors have told us their analysis also yields a $O(1/T)$ rate of convergence for problem dependent step-sizes, though we have not been able to easily verify this.} of $O(1/T^{\sigma})$. Unfortunately, while the analysis is technically non-asymptotic, the constant factors in the bound display a complex dependence on the problem instance and even scale exponentially with the eigenvalues of certain matrices. \cite{dalal2017finite} also give a high-probability bound, a nice feature that we do not address in this work. 

This paper was accepted at the 2018 Conference on Learning Theory (COLT) and published in the proceedings as a two-page extended abstract. While the paper was under review, an interesting paper by \cite{lakshminarayanan2018linear} appeared. They study linear stochastic approximation algorithms under i.i.d\,noise, including TD(0), with constant step-sizes and iterate averaging. This line of work dates back to \cite{gyorfi1996averaged}, who show that the iterates of a constant step-size linear stochastic approximation algorithm form an ergodic Markov chain and, \emph{in the case of i.i.d.\,observation noise}, their expectation in steady-state is equal to the true solution of the linear system. By a central limit theorem for ergodic sequences, the average iterate converges to the true solution, with mean-squared error decaying at rate $O(1/T)$.  \cite{bach2013non} give a sophisticated non-asymptotic analysis of the least-mean-squares algorithm with constant step-size and iterate-averaging. \cite{lakshminarayanan2018linear} aim to understand whether such guarantees extend to linear stochastic approximation algorithms more broadly. In the process, their work provides $O(1/T)$ bounds for iterate-averaged TD(0) with constant step-size. A remarkable feature of their approach is that the choice of step-size is independent of the conditioning of the features (although the bounds themselves do degrade if features become ill-conditioned). It is worth noting that these results rely critically on the assumption that noise is i.i.d. In fact, \cite{gyorfi1996averaged} provide a very simple example of failure under correlated noise. In this example, under a linear stochastic approximation algorithm applied with any constant step-size, the averaged-iterate will converge to the wrong limit. 

The recent works of \cite{dalal2017finite} and \cite{lakshminarayanan2018linear} give bounds for TD(0) only under i.i.d.\,observation noise. Therefore their results are most comparable to what is presented in Section \ref{sec:iid_sampling}. For the i.i.d.\,noise model, the main argument in favor of our approach is that it allows for extremely simple proofs, interpretable constant terms, and illuminating connections with SGD. Moreover, it is worth emphasizing that our approach gracefully extends to more complex settings, including more realistic models with Markovian noise, the analysis of TD with eligibility traces, and the analysis of Q-learning for optimal stopping problems as shown in Sections \ref{sec:markov_chain_analyis}, \ref{sec:td_lambda} and \ref{sec:opt_stopping}. 

While not directly comparable to our results, we point the readers to the excellent work of \cite{schapire1996worst}. To facilitate theoretical analysis, they consider a slightly modified version of the TD($\lambda$). % algorithm, called the TD$^{*}(\lambda)$ algorithm\footnote{TD($\lambda$) and TD$^{*}(\lambda)$ only differ for non-zero $\lambda$ and thus TD$^{*}(0)$ is identical to TD(0).}.  
The authors provide a finite time analysis for this algorithm in an adversarial model where the goal is to predict the discounted sum of future rewards from each state. Performance is measured relative to the best fixed linear predictor in hindsight. The analysis is creative, but results depend on a several unknown constants and on the specific sequence of states and rewards on which the algorithm is applied. \cite{schapire1996worst} also apply their techniques to study value function approximation in a Markov decision process. In that case, the bounds are much weaker than what is established here. Their bound scales with the size of the state space--which is enormous in most practical problems--and applies only to TD(1)--a somewhat degenerate special case of TD($\lambda$) in which it is equivalent to Monte Carlo policy evaluation \citep{sutton1998reinforcement}. 

%for value function approximation in a Markov process, they show an $\mathcal{O}(1/\sqrt{T})$ convergence rate for TD$^{*}(0)$, which is identical to TD(0). However, their rate depends on an unknown constant and also scales with the size of the state space.

\paragraph{Asymptotic analysis of stochastic approximation:}
There is a well developed asymptotic theory of stochastic approximation, a field that studies noisy recursive algorithms like TD \citep{kushner2003stochastic, borkar2009stochastic, benveniste2012adaptive}. Most asymptotic convergence proofs in reinforcement learning use a technique known as the ODE method \citep{borkar2000ode}. Under some technical conditions and appropriate decaying step-sizes, this method ensures the almost-sure convergence of stochastic approximation algorithms to the invariant set of a certain `mean' differential equation. The technique greatly simplifies asymptotic convergence arguments, since it completely circumvents issues with noise in the system and issues of step-size selection. But this  also makes it a somewhat coarse tool, unable to generate insight into an algorithm's sensitivity to noise, ill-conditioning, or step-size choices. A more refined set of techniques begin to address these issues. Under fairly broad conditions, a central limit theorem for stochastic approximation algorithms characterizes their limiting variance. Such a central limit theorem has been specifically provided for TD by \cite{konda2002thesis} and \cite{devraj2017fastest}. 

Despite the availability of such asymptotic techniques, the modern literature on first-order stochastic optimization focuses heavily on non-asymptotic analysis \citep{bottou2016optimization, bubeck2015convex, jain2017non}. One reason is that such asymptotic analysis necessarily focuses on a regime where step-sizes are arbitrarily small relative to problem features and the iterates have already converged to a small neighborhood of the optimum. However, the use of a first-order method in the first place signals that a practitioner is mostly interesting in cheaply reaching a reasonably accurate solution, rather than the rate of convergence in the neighborhood of the optimum. In practice, it is common to use constant step-sizes, so iterates never truly converge to the optimum. A non-asymptotic analysis requires grappling with the algorithm's behavior in practically relevant regimes where step-sizes are still relatively large and iterates are not yet close to the true solution.  

\paragraph{Analysis of related algorithms:}
A number of papers analyze algorithms related to and inspired by the classic TD algorithm. First, among others, \cite{antos2008learning,lazaric2010finite, ghavamzadeh2010lstd, pires2012statistical,prashanth2014fast} and \cite{tu2017least} analyze least-squares temporal difference learning (LSTD). \cite{yu2009convergence} study the related least-squares policy iteration algorithm.  The asymptotic limit point of TD is a minimizer of a certain population loss, known as the mean-squared projected Bellman error. LSTD solves a least-squares problem, essentially computing the exact minimizer of this loss on the empirical data. It is easy to derive a central limit theorem for LSTD. Finite time bounds follow from establishing uniform convergence rates of the empirical loss to the population loss. Unfortunately, such techniques appear to be quite distinct from those needed to understand the online TD algorithms studied in this paper.  Online TD has seen much wider use due to significant computational advantages \citep{sutton1998reinforcement}. 

Gradient TD methods are another related class of algorithms. These were derived by \cite{sutton2009convergent, sutton2009fast}  to address the issue that TD can diverge in so-called “off-policy” settings, where data is collected from a policy different from the one for which we want to estimate the value function.  
Unlike the classic TD(0) algorithm, gradient TD methods are designed to mimic gradient descent with respect to the mean squared projected Bellman error. \cite{sutton2009convergent, sutton2009fast} propose asymptotically convergent two-time scale stochastic approximation schemes based on this and more recently \cite{dalal2017concentration} give a finite time analysis of two time scale stochastic approximation algorithms, including several variants of gradient TD algorithms. A creative paper by \cite{liu2015finite} reformulates the original optimization as a primal-dual saddle point problem and leverages convergence analysis form that literature  to give a non-asymptotic analysis. This work was later revisited by \cite{touati2017convergent}, who established a faster rate of convergence. The works of \cite{dalal2017concentration, liu2015finite} and \cite{touati2017convergent} all consider only i.i.d.\,observation noise. One interesting open question is whether our techniques for treating the Markovian observation model will also apply to these analyses. Finally, it is worth highlighting that, to the best of our knowledge, substantial new techniques are needed to analyze the widely used TD(0), TD($\lambda$) and the Q-learning algorithm studied in this paper. Unlike gradient TD methods, they do not mimic noisy gradient steps with respect to any fixed objective\footnote{This can be formally verified for TD(0) with linear function approximation. If the TD step were a gradient with respect to a fixed objective, differentiating it should give the Hessian and hence a symmetric matrix. Instead, the matrix one attains is typically not a symmetric one.}.

%--------------------------------------------------------
\section{Problem formulation} \label{sec:formulation}
%--------------------------------------------------------

\paragraph{Markov reward process.}
%\label{subsec:mrp}
We consider the problem of evaluating the value function $V_{\mu}$ of a given policy $\mu$ in a Markov decision process (MDP). We work in the \emph{on policy} setting, where data is generated by applying the policy $\mu$ in the MDP. Because the policy $\mu$ is applied automatically to select actions, such problems are most naturally formulated as value function estimation in a Markov reward process (MRP). A MRP\footnote{We avoid $\mu$ from notation for simplicity.} comprises of $(\mathcal{S}, \mathcal{P}, \mathcal{R}, \gamma$) \citep{sutton1998reinforcement} where $\mathcal{S}$ is the set of states, $\mathcal{P}$ is the Markovian transition kernel, $\mathcal{R}$ is a reward function, and $\gamma<1$ is the discount factor. For a discrete state-space $\Sc$, $\mathcal{P}(s'|s)$ specifies the probability of transitioning from a state $s$ to another state $s'$. The reward function $\mathcal{R}(s,s')$ associates a reward with each state transition. We denote by $\Rc(s)= \sum_{s'\in \mathcal{S}} \mathcal{P}(s'|s) \mathcal{R}(s,s')$ the expected instantaneous reward generated from an initial state $s$. 
%------Dan's comment----
%Removed sentence "As a notational simplification, we will sometimes use the notation $r=\Rc(s,s')$ to denote these rewards.". This is not uniquely defined. What are s,s'? 

The value function associated with this MRP, $V_{\mu}$, specifies the expected cumulative discounted future reward as a function of the state of the system. In particular,
\begin{eqnarray*}
V_\mu(s) = \mathbb{E}\left[ \sum_{t=0}^\infty \gamma^t \Rc(s_t) \mid s_0=s \right],
\end{eqnarray*}
where the expectation is over sequences of states generated according to the transition kernel $\Pc$. This value function obeys the Bellman equation $T_\mu V_\mu = V_\mu$, where the Bellman operator $T_\mu$ associates a value function $V:\Sc \to \mathbb{R}$ with another value function $T_{\mu}V$ satisfying
\[
(T_\mu V)(s)= \Rc(s)+ \gamma \sum_{s' \in \mathcal{S}} \mathcal{P}(s'|s) V(s') \quad \forall \,\, s\in \Sc. 
\]
We assume rewards are bounded uniformly such that
\[
\abs{\Rc(s,s')} \leq r_{\max} \quad \forall \,\, s,s' \in \mathcal{S}.
\]
Under this assumption, value functions are assured to exist and are the unique solution to Bellman's equation \citep{bertsekas2012dynamic}. 
We also assume that the Markov reward process induced by following the policy $\mu$ is ergodic with a unique stationary distribution $\pi$. For any two states $s, s'$: $\pi(s') = \lim_{t\to \infty} \Prob(s_t= s'| s_0=s)$. 

Following common references \citep{bertsekas2012dynamic,dann2014policy, de2003linear}, we will simplify the presentation by assuming the state space $\Sc$ is a finite set of size $n=|\Sc|$. We elaborate on this choice in the remark below. 
\begin{rem}
Working with a finite state space allows for the use of compact matrix notation, which is the convention in work on linear value function approximation. It also avoids measure theoretic notation for conditional probability distributions. Our proofs extend in an obvious way to problems with countably infinite state-spaces. For problems with general state-space, even the core results in dynamic programming hold only under suitable technical conditions \citep{bertsekas1978stochastic}. 
\end{rem}

\paragraph{Value function approximation.}
Given a fixed policy $\mu$, the problem is to efficiently estimate the corresponding value function $V_{\mu}$ using only the observed rewards and state transitions. Unfortunately, due to the curse of dimensionality, most modern applications have intractably large state spaces, rendering exact value function learning hopeless. Instead, researchers resort to parametric approximations of the value function, for example by using a linear function approximator \citep{sutton1998reinforcement} or a non-linear function approximation such as a neural network \citep{mnih2015human}. In this work, we consider a linear function approximation architecture where the true value-to-go $V_{\mu}(s)$ is approximated as
$$V_{\mu}(s) \approx V_{\theta}(s) = \phi(s)^\top \theta, $$ 
where $\phi(s) \in \mathbb{R}^d$ is a fixed feature vector for state $s$ and $\theta \in \mathbb{R}^d$ is a parameter vector that is shared across states. When the state space is the finite set $\Sc = \{s_1, \ldots, s_n\}$, $V_{\theta}\in \mathbb{R}^n$ can be expressed compactly as
\begin{align*}
V_{\theta} &= \begin{bmatrix}
                               \phi(s_1)^\top \\
              				   \vdots \\
              				   \phi(s_n)^\top
             				   \end{bmatrix} \theta = \begin{bmatrix}
              				  						   \phi_1(s_1) & \phi_k(s_1) & \phi_d(s_1)  \\
                                                       \vdots & \vdots & \vdots \\
              				   						   \phi_1(s_n) & \phi_k(s_n) & \phi_d(s_n) \\
             				  						   \end{bmatrix} \theta = \Phi \theta,                                       \end{align*}
where $ \Phi \in \mathbb{R}^{n \times d}$ and $\theta \in \mathbb{R}^d$. We assume throughout that the $d$ features vectors $\{\phi_k\}_{k=1}^d$, forming the columns of $\Phi$ are linearly independent. 

%It is worth mentioning that TD will never directly store the matrix $\Phi$. Instead, it requires only the ability to evaluate $V_{\theta}(s)=\phi(s)^\top \theta$ at a given state $s$, which requires only that the feature mapping $\phi$ can be applied efficiently. 

%\subsection{Norms and projections}
%For a value function $V\in \mathbb{R}^n$, we will often consider its projection onto the space $\{\Phi \theta : \theta \in \mathbb{R}^d \}$ of approximate value functions. However, it will often be more appropriate for us to work in a Euclidean norm that is weighted by the long run frequency with which states are visited in the Markov chain. Let $D = \text{diag}(\pi(s_1),\ldots,\pi(s_n)) \in \mathbb{R}^{n \times n}$ denote the diagonal matrix whose elements are given by the entries of the stationary distribution $\pi$. Define the corresponding inner product $\langle x, y\rangle_{D} = x^\top D y$ and the norm $\| x \|_{D} \ = \sqrt{\langle x, x\rangle_{D}}$. We denote by $\Pi_{D}$ the projection 
%\[ 
%\Pi_{D}V = \argmin_{V' \in \text{span}(\Phi)} \norm{V - V'}_{D} \quad \forall \,\, V \in \mathbb{R}^n
%\]
%onto the space of approximate value functions $\{\Phi \theta : \theta \in \mathbb{R}^d\}$. 
\paragraph{Norms in value function and parameter space.}
For a symmetric positive definite matrix $A$, define the inner product $\langle x, y \rangle_{A} = x^\top A  y$ and the associated norm $\| x\|_{A} = \sqrt{ x^\top A x}.$ If $A$ is positive semi-definite rather than positive definite then $\| \cdot \|_{A}$ is called a semi-norm.  Let $D = \text{diag}(\pi(s_1),\ldots,\pi(s_n)) \in \mathbb{R}^{n \times n}$ denote the diagonal matrix whose elements are given by the entries of the stationary distribution $\pi(\cdot)$. Then, for two value functions $V$ and $V'$, 
\[
\| V -V' \|_D = \sqrt{\sum_{s \in \mathcal{S}} \pi(s) \left( V(s)-V'(s) \right)^2},
\]
measures the mean-square difference between the value predictions under $V$ and $V'$, in steady-state. This suggests a natural norm on the space of parameter vectors. In particular, for any $\theta, \theta' \in \mathbb{R}^d$,  
\[ 
\| V_{\theta} - V_{\theta'}\|_D = \sqrt{\sum_{s \in \mathcal{S}} \pi(s)\left(\phi(s)^\top ( \theta - \theta')\right)^2 } = \| \theta - \theta' \|_{\Sigma}
\]
where 
\[
\Sigma := \Phi^\top D \Phi = \sum_{s\in \Sc} \pi(s) \phi(s) \phi(s)^\top
\]
is the steady-state feature covariance matrix. 

\paragraph{Feature regularity.}
We assume that any entirely redundant or irrelevant features have been removed, so $\Sigma$ has full rank. Additionally, we also assume that $\|\phi(s)\|_2^2 \leq 1$ for all $s\in \Sc$, which can be ensured through feature normalization. This also ensures that $\Sigma$ exists. Let $\omega>0$ be the minimum eigenvalue of $\Sigma$.  From our bound on the feature vectors, the maximum eigenvalue of $\Sigma$ is less than $1$\footnote{Let  $\lambda_{\max}(A) = \max_{\|x\|_2=1} x^\top A x$ denote the maximum eigenvalue of a symmetric positive-semidefinite matrix. Since this is a convex function, $\lambda_{\max}(\Sigma) \leq \sum_{s\in \Sc} \pi(s) \lambda_{\max}(\phi(s)\phi(s)^\top) \leq \sum_{s\in \Sc} \pi(s) =1$.}   %We use Weyl's inequality along with the fact that $X=\phi(s)\phi(s)^\top$ is a Hermitian matrix with maximum eigenvalue less that 1.}
, so $1/\omega$ bounds the condition number of the feature covariance matrix. The following lemma is an immediate consequence of our assumptions.
\begin{restatable}[Norm equivalence]{lem}{}
\label{lemma:strong_conv}
For all $\theta \in \mathbb{R}^d$, $
\sqrt{\omega} \| \theta \|_2 \leq \| V_{\theta} \|_D \leq  \|\theta \|_2.$
\end{restatable}
\noindent While we assume it has full rank, we will establish some finite time bounds that are independent of the conditioning of the feature covariance matrix. 

%--------------------------------------------------------------
\section{Temporal difference learning}\label{sec:algorithm}
%--------------------------------------------------------------
We consider the classic temporal difference learning algorithm \citep{sutton1988learning}. The algorithm starts with an initial parameter estimate $\theta_0$ and at every time step $t$, it observes one data tuple $O_t=(s_t, r_t=\Rc(s_t,s_t'), s'_t)$ consisting of the current state, the current reward and the next state reached by playing policy $\mu$ in the current state. This tuple is used to define a loss function, which is taken to be the squared sample Bellman error. It then proceeds to compute the next iterate $\theta_{t+1}$ by taking a gradient step. Some of our bounds guarantee accuracy of the average iterate, denoted by $\bar{\theta}_t= t^{-1}\sum_{i=0}^{t-1} \theta_i$. The version of TD presented in Algorithm \ref{algo:increment_TD} also makes online updates to the averaged iterate.  

We present in Algorithm \ref{algo:increment_TD} the simplest variant of TD, which is known as TD(0). % An extension of to TD with eligibility traces, known as TD($\lambda$) is presented in Section \ref{sec:td_lambda}. 
It is also worth highlighting that here we study online temporal difference learning, which makes incremental gradient-like updates to the parameter estimate based on the most recent data observations only. Such algorithms are widely used in practice, but harder to analyze than so-called batch TD methods like the LSTD algorithm of \cite{bradtke1996linear}. 

 \begin{algorithm2e}
     \SetAlgoLined
     \SetKwInOut{Input}{Input}
     \SetKwInOut{Output}{Output}
     \Input{initial guess $\theta_0$, step-size sequence $\{\alpha_{t}\}_{t\in \mathbb{N}}$.}
     Initialize: $\bar{\theta}_{0} \gets \theta_0$. \\
     \For{$t=0,1,\ldots$}{
        	\nosemic Observe tuple: $O_t = (s_t,r_t=\Rc(s_t,s_t'),s'_t)$\;
            Define target: $y_t = \Rc(s_t,s_t') + \gamma V_{\theta_t}(s'_t)$\tcc*[r]{sample Bellman operator}
            Define loss function: $\frac{1}{2} (y_t - V_{\theta}(s_t))^2$\tcc*[r]{sample Bellman error squared}
            Compute negative gradient: $g_t(\theta_t) = - \frac{\partial}{\partial \theta} \frac{1}{2}  (y_t - V_{\theta}(s_t))^2\at[\big]{\theta = \theta_t}$\;
            Take a gradient step: $\theta_{t+1} = \theta_t + \alpha_t g_t(\theta_t)$
        \tcc*[r]{$\alpha_t$:step-size}
        Update averaged iterate: $\bar{\theta}_{t+1} \gets \left(\frac{t}{t+1}\right)\bar{\theta}_{t}+ \left(\frac{1}{t+1}\right)\theta_{t+1} $ 
        \tcc*[r]{$\bar{\theta}_{t+1} = \frac{1}{t+1} \sum_{\ell=0}^{t} \theta_\ell$}
        }
           \caption{TD(0) with linear function approximation} 
           \label{algo:increment_TD}
    \end{algorithm2e}
    
     At time $t$, TD takes a step in the direction of the negative gradient $g_{t}(\theta_t)$ evaluated at the current parameter. As a general function of $\theta$ and the tuple $O_t=(s_t, r_t, s_t')$, the negative gradient can be written as
    \begin{eqnarray}
    \label{eq:grad_TD(0)}
     g_t(\theta) = \Big( r_t + \gamma \phi(s'_{t})^\top \theta - \phi(s_t)^\top \theta \Big)\phi(s_t) .
    \end{eqnarray}
    The long-run dynamics of TD are closely linked to the expected negative gradient step when the tuple $O_t=(s_t, r_t, s'_t)$ follows its \emph{steady-state} behavior:
    \[ 
    \bar{g}(\theta) := \sum_{s, s' \in \Sc} \pi(s) \Pc(s'|s) \left(\Rc(s,s')+\gamma \phi(s')^\top \theta - \phi(s)^\top \theta \right)\phi(s) \quad \forall \,\, \theta \in \mathbb{R}^d. 
    \]
    This can be rewritten more compactly in several useful ways. One such way is, 
    \begin{equation}
    \label{eq: expectation form of gbar}
    \bar{g}(\theta) = \E\left[ \phi r\right] + \E\left[\phi(\gamma\phi' - \phi)^\top \right] \theta,  
    \end{equation}
    where $\phi = \phi(s)$ is the feature vector of a random initial state $s\sim \pi$, $\phi'=\phi(s')$ is the feature vector of a random next state drawn according to $s' \sim \Pc(\cdot \mid s)$, and $r=\mathcal{R}(s,s')$. In addition, since $\sum_{s' \in \mathcal{S}} \Pc(s'| s)\left(\Rc(s, s') + \gamma \phi(s')^\top \theta  \right)= (T_{\mu} \Phi \theta)(s)$, we can recognize that
    \begin{equation}
    \label{eq:Expected grad matrix form}
    \bar{g}(\theta)=\Phi^\top D (T_{\mu}\Phi\theta - \Phi \theta).
    \end{equation}
See \cite{tsitsiklis1997analysis} for a derivation of this fact. 
\section{Asymptotic convergence of temporal difference learning}
\label{sec:asymptotic_TD(0)}
The main challenge in analyzing TD is that the gradient steps $g_{t}(\theta)$ are not true stochastic gradients with respect to any fixed objective. 
% At every iteration, the step taken at time $t$ \jb{does pull the iterate $\theta_{t+1}$ closer to $y_t$, but $y_t$ itself depends on $\theta_t$}.
The gradient step taken at time $t$ pulls the value prediction $V_{\theta_{t+1}}(s_t)$ closer to $y_t$, but $y_t$ itself depends on $V_{\theta_t}$. So does this circular process converge? The key insight of \cite{tsitsiklis1997analysis} was to interpret this as a stochastic approximation scheme for solving a fixed point equation known as the projected Bellman equation. Contraction properties together with general results from stochastic approximation theory can then be used to show convergence.  

Should TD converge at all, it should be to a stationary point. Because the feature covariance matrix $\Sigma$ is full rank there is a unique\footnote{This follows formally as a consequence of Lemma \ref{lemma:expected_gradient} in this paper.} vector $\theta^*$ with $\bar{g}(\theta^*)=0$.  We briefly review results that offer insight into $\theta^*$ and proofs of the asymptotic convergence of TD. 

\paragraph{Understanding the TD limit point.} 
\cite{tsitsiklis1997analysis} give an interesting characterization of the limit point $\theta^*$. They show it is the unique solution to the \emph{projected} Bellman equation
\begin{equation}
\label{eq:projected_Bellmann_eq}
\Phi \theta = \Pi_{D} T_{\mu} \Phi \theta,
\end{equation}
where $\Pi_{D}(\cdot)$ is the projection operator onto the subspace $\{\Phi x \mid x\in \mathbb{R}^d  \}$ spanned by these features in the inner product $\langle \cdot , \cdot \rangle_D$. To see why this is the case, note that by using $\bar{g}(\theta^*)=0$ along with Equation \eqref{eq:Expected grad matrix form}, 
\[ 
0 = x^\top \bar{g}(\theta^*) = \langle \Phi x \, , \, T_{\mu}\Phi \theta^* - \Phi \theta^* \rangle_D \quad \forall \,\, x \in \mathbb{R}^d.
\]
That is, the Bellman error at $\theta^*$, given by $(T_{\mu}\Phi \theta^* - \Phi \theta^*)$, is orthogonal to the space spanned by the features in the inner product $\langle \cdot , \cdot \rangle_D$. By definition, this means $ \Pi_{D} \left(T_{\mu}\Phi \theta^* - \Phi \theta^*\right) = 0$ and hence $\theta^*$ must satisfy the projected Bellman equation.  

The following lemma shows the projected Bellman operator, $\Pi_{D} T_{\mu}(\cdot)$ is a contraction, and so in principle, one could converge to the approximate  value function $\Phi \theta^*$ by repeatedly applying it. 
TD appears to serve a simple stochastic approximation scheme for solving the projected-Bellman fixed point equation.
\begin{restatable}[]{lem}{}
	\label{lemma:BO_prop}
	[\textbf{\cite{tsitsiklis1997analysis}}]
	$\Pi_{D} T_\mu (\cdot)$ is a contraction with respect to $\| \cdot \|_{D}$ with modulus $\gamma$, that is,
	\begin{eqnarray*}
		\norm{\Pi_{D} T_\mu V_{\theta} - \Pi_{D} T_\mu V_{\theta'}}_{D} \leq \gamma \norm{V_{\theta} - V_{\theta'}}_{D} \quad \forall \, \theta, \theta' \in \mathbb{R}^d. 
	\end{eqnarray*}
\end{restatable}
\noindent Finally, the limit of convergence comes with some competitive guarantees. From Lemma \ref{lemma:BO_prop}, a short argument shows 
\begin{equation}\label{eq: approximation error bound}
\norm{V_{\theta^*} - V_{\mu} }_{D}   \leq \frac{1}{\sqrt{1-\gamma^2}} \norm{ \Pi_{D} V_{\mu} - V_{\mu}  }_{D}. 
\end{equation}
See Chapter 6 of \cite{bertsekas2012dynamic} for a proof. The left hand side of Equation \eqref{eq: approximation error bound} measures the root-mean-squared deviation between the value predictions of the limiting TD value function and the true value function. On the right hand side, the projected value function $\Pi_{D} V_{\mu}$ minimizes root-mean-squared prediction errors among all value functions in the span of $\Phi$. If $V_{\mu}$ actually falls within the span of the features, there is no approximation error at all and TD converges to the true value function.

\paragraph{Asymptotic convergence via the ODE method.}
Like many analyses in reinforcement learning, the convergence proof of \cite{tsitsiklis1997analysis} appeals to a powerful technique from the stochastic approximation literature known as the  ``ODE method''. Under appropriate conditions, and assuming a decaying step-size sequence satisfying the Robbins-Monro conditions, this method establishes the asymptotic convergence of the stochastic recursion $\theta_{t+1}=\theta_t + \alpha_t g_{t}(\theta_t)$ as a consequence of the global asymptotic stability of the deterministic ODE: $\dot{\theta}_t = \bar{g}(\theta_t)$. The critical step in the proof of \cite{tsitsiklis1997analysis} is to use the contraction properties of the Bellman operator to establish this ODE is globally asymptotically stable with the equilibrium point $\theta^*$.  
	
The ODE method vastly simplifies convergence proofs. First, because the continuous dynamics can be easier to analyze than discretized ones, and more importantly, because it avoids dealing with stochastic noise in the problem. At the same time, by side-stepping these issues, the method offers little insight into the critical effect of step-size sequences, problem conditioning, and mixing time issues on algorithm performance.

\section{Outline of analysis}\label{sec:outline}
The remainder of the paper focuses on a finite time analysis of TD. Broadly, we establish two types of finite time bounds. We first derive bounds that depend on the condition number of the feature covariance matrix. In that case, we state explicit bounds on the expected distance $\E\left[\| \theta_T - \theta^*  \|_2^2\right]$ of the iterate from the TD fixed-point, $\theta^*$. These mirror what one might expect from the literature on stochastic optimization of strongly convex functions: results showing that TD with constant step-sizes converges to within a radius of $\theta^*$ at an exponential rate, and $O(1/T)$ convergence rates with appropriate decaying step-sizes. Note that by Lemma \ref{lemma:strong_conv}, $\| V_{\theta_T} - V_{\theta^*}\|_D^2 \leq \| \theta_T - \theta^*  \|_2^2$, so  bounds on the distance of the iterate to the TD fixed point also imply bounds on the distance between value predictions. 
	
These results establish fast rates of convergence, but only if the problem is well conditioned. The choice of step-sizes is also very sensitive to problem conditioning. Work on robust stochastic approximation \citep{nemirovski2009robust} argues instead for the use of comparatively large step-sizes together with iterate averaging. Following the spirit of this work, we also give explicit bounds on  $\E\left[\| V_{\bar{\theta}_T} - V_{\theta^*}\|_D^2\right]$, which measures the mean-squared gap between the predictions under the averaged-iterate $\bar{\theta}_T$ and under the TD limit point $\theta^*$. These yield slower $O(1/\sqrt{T})$ convergence rates, but both the bounds and step-sizes are completely independent of problem conditioning. 
	
Our approach is to start by developing insights from simple, stylized settings, and then incrementally extend the analysis to more complex settings. The analysis is outlined below.
	
	\begin{description}
		\item[Noiseless case:] Drawing inspiration from the ODE method discussed above, we start by analyzing the Euler discretization of the ODE $\dot{\theta}_t = \bar{g}(\theta_t)$, which is the deterministic recursion $\theta_{t+1} =  \theta_t + \alpha \bar{g}(\theta_t)$. We call this method ``mean-path TD''. As motivation, the section first considers a fictitious gradient descent algorithm designed to converge to the TD fixed point. We then develop striking analogues for mean-path TD of the key properties underlying the convergence of gradient descent. Easy proofs then yield two bounds mirroring those given for gradient descent. 
		\item[Independent noise:] Section \ref{sec:iid_sampling} studies TD under an i.i.d.\,observation model, where the data-tuples used by TD are drawn i.i.d.\,from the stationary distribution. The techniques used to analyze mean-path TD(0) extend easily to this setting, and the resulting bounds mirror standard guarantees for stochastic gradient descent. %including bounds for constant and decaying step-sizes that depend on the minimum eigenvalue of the feature covariance matrix, and a bound for TD with large step-sizes and iterate averaging that allows for feature covariance matrices that are ill conditioned. 
		\item[Markov noise:]  In Section \ref{sec:markov_chain_analyis}, we analyze TD in the more realistic setting where the data is collected from a single sample path of an ergodic Markov chain. This setting introduces significant challenges due to the highly dependent nature of the data. For tractability, we assume the Markov chain satisfies a certain uniform bound on the rate at which it mixes, and study a variant of TD that uses a projection step to ensure uniform boundedness of the iterates. In this case,  our results essentially scale by a factor of the mixing time relative to the i.i.d.\,case. 
		\item[Extension to TD($\lambda)$:] In Section \ref{sec:td_lambda}, we extend the analysis under the Markov noise to TD with eligibility traces, popularly known as TD($\lambda$). Eligibility traces are known to often provide performance gains in practice, but theoretical analysis is more complex. Such analysis offers some insight into the subtle tradeoffs in the selection of the parameter $\lambda\in [0,1]$. %The limit point of TD($\lambda$) offers better error guarantees for larger values of $\lambda$. However, as is reflected by our bounds for TD($\lambda$) with constant step-size, the rate of convergence to this limit point can degrade as $\lambda$ increases.
        \item[Approximate optimal stopping:] A final section extends our results to a class of high dimensional optimal stopping problems. We analyze Q-learning with linear function approximation. Building on observations of \cite{tsitsiklis1999optimal}, we show the key properties used in our analysis of TD continue to hold for Q-learning in this setting. The convergence bounds shown in Sections \ref{sec:iid_sampling} and \ref{sec:markov_chain_analyis} therefore apply \emph{without any modification}. 
	\end{description}

\section{Analysis of mean-path TD} \label{sec:mean_path}
%-------------------------------------------------------------------------------
All practical applications of TD involve observation noise. However, a great deal of insight can be gained by investigating a natural deterministic analogue of the algorithm. Here we study the recursion
\[
\theta_{t+1} = \theta_t + \alpha \bar{g}(\theta_t)  \qquad t\in \mathbb{N}_0 = \{0,1, 2,\ldots\},
\]
which is the Euler discretization of the ODE described in Section \ref{sec:asymptotic_TD(0)}. We will refer to this iterative algorithm as \emph{mean-path TD.} In this section, we develop key insights into the dynamics of mean-path TD that allow for a remarkably simple finite time analysis of its convergence. Later sections of the paper show how these ideas extend gracefully to analyses with observation noise. 

The key to our approach is to develop properties of mean-path TD that closely mirror those of gradient descent on a particular quadratic loss function. To this end, in the next subsection, we review a  simple analysis of gradient descent. In Subsection \ref{subsec: mean path properties}, we establish key properties of mean-path TD mirroring those used to analyze this gradient descent algorithm. Finally, Subsection \ref{subsec:mean path convergence rate} gives convergence rates of mean-path TD, with proofs and rates mirroring those given for gradient descent except for a constant that depends on the discount factor, $\gamma$.  

\subsection{Gradient descent on a value function loss}
Consider the cost function 
\[
f(\theta) = \frac{1}{2} \| V_{\theta^*} - V_{\theta}  \|_D^2 = \frac{1}{2} \| \theta^* - \theta \|_{\Sigma}^2,
\]
which measures the mean-squared gap between the value predictions under $\theta$ and those under the stationary point of TD, $\theta^*$. Consider as well a hypothetical algorithm that performs gradient descent on $f$, iterating $\theta_{t+1} = \theta_t - \alpha \grad f(\theta_t)$ for all $t\in \mathbb{N}_0$. Of course, this algorithm is not implementable, as one does not know the limit point $\theta^*$ of TD. However, reviewing an analysis of such an algorithm will offer great insights into our eventual analysis of TD.  
%To start note $f(\theta) = \| \theta^* -\theta \|_{\Phi^\top D \Phi}^2$ and 
%\[ 
%\omega I  \preceq \Phi^\top D \Phi \preceq I
%\]
%where $I$ is the $d\times d$ identity matrix. From this, we arrive at the following lemma, which will be useful throughout the paper. 
%\begin{lemma}[Norm equivalence] \label{lemma:strong_conv}
%\[ 
%\omega \|  \theta^* - \theta   \|_2^2  \leq  \| V_{\theta^*} - V_{\theta} \|_D^2  \leq \|  \theta^* - \theta   \|_2^2.
%\]
%\end{lemma}

To start, a standard decomposition characterizes the evolution of the error at iterate $\theta_t$: 
\[
\| \theta^* - \theta_{t+1} \|_2^2 = \| \theta^* - \theta_t \|_2^2 + 2\alpha \grad f(\theta_t)^\top (\theta^*-\theta_t) + \alpha^2 \|\grad f(\theta_t)\|_2^2. 
\]
To use this decomposition, we need two things. First, some understanding of $\grad f(\theta_t)^\top (\theta^*-\theta_t)$, capturing whether the gradient points in the direction of $(\theta^*-\theta_t)$. And second, we need an upper bound on the norm of the gradient $\|\grad f(\theta_t)\|_2^2$. 
% To use this decomposition, we need some understanding of $(\theta^*-\theta_t)^\top \grad f(\theta_t)$, capturing whether the gradient points in the direction of $\theta^*-\theta$, as well as of the norm of the gradient $\|\grad f(\theta_t)\|_2^2$. 
In this case,  $\grad f(\theta) = \Sigma (\theta - \theta^*)$, from which we conclude 
\begin{equation}\label{eq: quadratic gradient angle} 
\grad f(\theta)^\top (\theta^* - \theta) = - \|  \theta^*-\theta \|^2_{\Sigma} = - \| V_{\theta^*} - V_{\theta}  \|_D^2. 
\end{equation}
In addition, one can show\footnote{This can be seen from the fact that for any vector $u$ with $\|u\|_2 \leq 1$, 
$$u^\top \grad  f(\theta) = \langle  u\, , \, \theta - \theta^* \rangle_{\Sigma} \leq \|u \|_{\Sigma} \| \theta^* -\theta \|_{\Sigma}\leq \| \theta^* -\theta \|_{\Sigma} = \| V_{\theta^*} - V_{\theta}  \|_D.$$
}
\begin{equation}\label{eq: quadratic gradient norm}
\| \grad f(\theta)\|_2 \leq \| V_{\theta^*} - V_{\theta}  \|_D. 
\end{equation}
Now, using \eqref{eq: quadratic gradient angle} and \eqref{eq: quadratic gradient norm}, we have that for step-size $\alpha=1$, 
\begin{equation}\label{eq: quadratic descent rate}
\| \theta^* - \theta_{t+1} \|_2^2 \leq \| \theta^* - \theta_t \|_2^2 -  \| V_{\theta^*} - V_{\theta_t}  \|_D^2.
\end{equation}
The distance to $\theta^*$ decreases in every step, and does so more rapidly if there is a large gap between the value predictions under $\theta$ and $\theta^*$.  Combining this with Lemma \ref{lemma:strong_conv} gives
\begin{equation}\label{eq: GD bound 1}
\| \theta^* - \theta_{t+1} \|_2^2 \leq (1-\omega)\| \theta^* - \theta_t \|_2^2\leq \ldots \leq (1-\omega)^{t+1}\| \theta^* - \theta_{0} \|_2^2.
\end{equation}
Recall that $\omega$ denotes the minimum eigenvalue of $\Sigma$. This shows that error converges at a fast geometric rate. However the rate of convergence degrades if the minimum eigenvalue $\omega$ is close to zero. Such a convergence rate is therefore only meaningful if the feature covariance matrix is well conditioned. 

By working in the space of value functions and performing iterate averaging, one can also give a guarantee that is independent of $\omega$. Recall the notation  $\bar{\theta}_T = T^{-1}\sum_{t=0}^{T-1} \theta_t$ for the averaged iterate. A simple proof from \eqref{eq: quadratic descent rate} shows 
\begin{equation}\label{eq: GD bound 2}
\| V_{\theta^*} - V_{\bar{\theta}_T} \|_{D}^2 \leq \frac{1}{T}\sum_{t=0}^{T-1} \| V_{\theta^*} - V_{\theta_t} \|_{D}^2  \leq \frac{\| \theta^* - \theta_0 \|_2^2}{T}.  
\end{equation}

\subsection{Key properties of mean-path TD}\label{subsec: mean path properties}
This subsection establishes analogues for mean-path TD of the key properties  \eqref{eq: quadratic gradient angle} and \eqref{eq: quadratic gradient norm} used to analyze gradient descent. 
First, to characterize the gradient update, our analysis builds on Lemma 7 of \cite{tsitsiklis1997analysis}, which uses the contraction properties of the projected Bellman operator to conclude
\begin{equation}\label{eq: positive angle}
\bar{g}(\theta)^\top (\theta^* - \theta) >0 \quad \forall \,\, \theta \neq \theta^*.
\end{equation}
That is, the expected update of TD always forms a positive angle with $(\theta^*-\theta)$. Though only Equation \eqref{eq: positive angle} was stated in their lemma, \cite{tsitsiklis1997analysis} actually reach a much stronger conclusion in their proof itself. This result, given in Lemma \ref{lemma:expected_gradient} below, establishes that the expected updates of TD point in a descent direction of $\| \theta^* - \theta\|^2_2$, and do so more strongly when the gap between value functions under $\theta$ and $\theta^*$ is large. We will show that this more quantitative form of \eqref{eq: positive angle} allows for elegant finite time-bounds on the performance of TD. 
% It seems the power of the result has not been appreciated in the literature. We also provide a new and more elementary proof of Lemma \ref{lemma:expected_gradient}. 

Note that this lemma mirrors the property in Equation \eqref{eq: quadratic gradient angle}, but with a smaller constant of $(1-\gamma)$. This reflects that expected TD must converge to $\theta^*$ by bootstrapping \citep{sutton1988learning} and may follow a less direct path to $\theta^*$ than the fictitious gradient descent method considered in the previous subsection. Recall that the limit point $\theta^*$ solves $\bar{g}(\theta^*)=0$. 
\begin{restatable}[]{lem}{expectedGradientAngle}
\label{lemma:expected_gradient}
For any $\theta \in \mathbb{R}^d$,
\begin{equation*}
(\theta^* - \theta)^\top \bar{g}(\theta) \,\, \geq \,\, (1-\gamma) \norm{V_{\theta^*} - V_{\theta}}^2_{D}.
\end{equation*}
\end{restatable}
\begin{proof} We use the notation described in Equation \eqref{eq: expectation form of gbar} of Section \ref{sec:algorithm}. Consider a stationary sequence of states with random initial state $s \sim \pi$ and subsequent state $s'$, which, conditioned on $s$, is drawn from $\Pc(\cdot|s)$. Set $\phi=\phi(s)$, $\phi'=\phi(s')$ and $r = \Rc(s,s')$. Define $\xi= V_{\theta^*}(s)-V_{\theta}(s)=(\theta^*-\theta)^\top \phi$ and $\xi'=V_{\theta^*}(s')-V_{\theta}(s')=(\theta^*-\theta)^\top \phi'$.  By stationarity, $\xi$ and $\xi'$ are two correlated random variables with the same same marginal distribution. By definition, $\pi$, $\E[\xi^2]= \| V_{\theta^*}-V_{\theta}\|_D^2$ since $s$ is drawn from $\pi$. 
% Consider a stationary sequence of states $(s_0, s_1,\ldots)$ and set $\phi=\phi(s)$ and $\phi'=\phi(s')$. Similarly, define 
% $\xi= (\theta^*-\theta)^\top \phi$ and $\xi'=(\theta^*-\theta)^\top \phi'$.  By stationarity, these are two correlated random variables with the same same marginal distribution. 
%A simple argument shows that\[\E[\xi^2] = = \|V_{\theta^*}-V_{\theta}\|_D^2.\]

Using the expression for $\bar{g}(\theta)$ in Equation \eqref{eq: expectation form of gbar}, 
\begin{equation}\label{eq: simplified gbar}
\bar{g}(\theta)=\bar{g}(\theta)-\bar{g}(\theta^*)  = \E[\phi (\gamma \phi' - \phi)^\top (\theta-\theta^*)]=\E[\phi(\xi - \gamma \xi')].
\end{equation}
Therefore 
\begin{eqnarray*}
(\theta^*-\theta)^\top \bar{g}(\theta) = \E\left[ \xi(\xi - \gamma \xi') \right] =\E[\xi^2] - \gamma \E[\xi' \xi] \geq  (1-\gamma)\E[\xi^2] = (1-\gamma) \norm{V_{\theta^*}-V_{\theta}}_D^2.
\end{eqnarray*}
The inequality above uses Cauchy-Schwartz inequality together with the fact that $\xi$ and $\xi'$ have the same marginal distribution to conclude $\E[\xi \xi']\leq \sqrt{\E[\xi^2]}\sqrt{\E[(\xi')^2]} = \E[\xi^2].$
\end{proof}

Lemma \ref{lemma:grad_norm_mean_path} is the other key ingredient to our results. It upper bounds the norm of the expected negative gradient, providing an analogue of Equation \eqref{eq: quadratic gradient norm}.
\begin{restatable}[]{lem}{}
\label{lemma:grad_norm_mean_path}
$\|\bar{g}(\theta) \|_2 \leq  2 \| V_{\theta}-V_{\theta^*}\|_D \,\, \forall \,\, \theta \in \mathbb{R}^d$. 
\end{restatable}
\begin{proof}
Beginning from \eqref{eq: simplified gbar} in the Proof of Lemma \ref{lemma:expected_gradient}, we have
\begin{eqnarray*}
\|\bar{g}(\theta) \|_2 = \| \E[\phi(\xi - \gamma \xi')] \|_2 \leq  \sqrt{\E\left[ \| \phi \|_2^2 \right]} \sqrt{\E\left[(\xi - \gamma \xi')^2 \right] } \leq \sqrt{\E[\xi^2]} + \gamma \sqrt{\E[(\xi')^2]} = (1+\gamma)\sqrt{\E[\xi^2]},
\end{eqnarray*}
where the second inequality uses the assumption that $\| \phi \|_2 \leq 1$ and the final equality uses that $\xi$ and $\xi'$ have the same marginal distribution. We conclude by recalling that $\E[\xi^2]= \| V_{\theta^*} - V_{\theta} \|_D^2$ and $1+\gamma\leq 2$. 
% \begin{eqnarray*}
% \|\bar{g}(\theta) \|_2 = \|\bar{g}(\theta-\theta^*) \|_2 
% &=& \left\|\mathbb{E}\left[ \phi\left(\gamma \phi' - \phi \right)^\top (\theta-\theta^*)   \right]   \right\|_2 \\
% &\leq & \left\|\mathbb{E}[\phi] \right\|_2   \left| \mathbb{E}\left[\left(\gamma \phi' - \phi \right)^\top (\theta-\theta^*) \right] \right| \\
% &\leq& \left| \mathbb{E}\left[\left(\gamma \phi' - \phi \right)^\top (\theta-\theta^*) \right] \right|\\
% &\leq& \left| \mathbb{E}\left[(\phi')^\top (\theta-\theta^*) \right]\right| + \left| \mathbb{E}\left[\phi^\top (\theta-\theta^*) \right]\right| \\
% &=&  2\left| \mathbb{E}\left[\phi^\top (\theta-\theta^*) \right]\right|\\
% &\leq & 2\sqrt{\mathbb{E} \left[ | \phi^\top (\theta-\theta^*) |^2 \right] }\\
% %&=& 2 \sqrt{(\theta - \theta^*)^\top \Phi D \Phi^\top (\theta - \theta^*)} \\
% &=& 2\| \Phi (\theta-\theta^*)\|_D.
% \end{eqnarray*}
\end{proof}

Lemmas \ref{lemma:expected_gradient} and \ref{lemma:grad_norm_mean_path} are quite powerful when used in conjunction.
As in the analysis of gradient descent reviewed in the previous subsection, our analysis starts with a recursion for the error term, $\|\theta_{t} - \theta^* \|^2$. See Equation \eqref{eq:error_rec_mean_path} in Theorem \ref{thrm:mean_path} below. Lemma \ref{lemma:expected_gradient} shows the first order term in this recursion reduces the error at each time step, while using the two lemmas in conjunction shows the first order term dominates a constant times the second order term. Precisely,   
% that the angle $\bar{g}(\theta)$ makes with the vector $(\theta^*-\theta)$ is lower bounded by a constant that depends on the discount factor. This implies a lower bound on the component of $\bar{g}(\theta)$ that is aligned with $(\theta^* - \theta)$ as shown in in Figure \ref{fig: angle of td update}.
\begin{eqnarray*}
\bar{g}(\theta)^\top (\theta^* - \theta) \geq (1-\gamma) \norm{V_{\theta^*} - V_{\theta}}^2_{D} \geq \frac{(1-\gamma)}{4} \| \bar{g}(\theta) \|^2_2.
\end{eqnarray*}
This leads immediately to conclusions like Equation \eqref{eq: noiseless key inequality}, from which finite time convergence bounds follow. 

It is also worth pointing out that as TD(0) is an instance of linear stochastic approximation, these two lemmas can be interpreted as statements about the eigenvalues of the matrix driving its behavior\footnote{Recall from Section \ref{sec:algorithm} that $\bar{g}(\theta)$ is an affine function. That is, it can be written as $A\theta-b$ for some $A\in \mathbb{R}^{d\times d}$ and $b\in \mathbb{R}^d$. Lemma \ref{lemma:expected_gradient} shows that $A\preceq - (1-\gamma)\Sigma$, i.e. that $A+(1-\gamma)\Sigma$ is negative definite. It is easy to show that $\|\bar{g}(\theta)\|_2^2=(\theta-\theta^*)^{\top} (A^{\top} A) (\theta-\theta^*)$, so Lemma \ref{lemma:grad_norm_mean_path} shows that $A^\top A \preceq \Sigma$. Taking this perspective, the important part of these lemmas is that they allows us to understand TD in terms of feature covariance matrix $\Sigma$ and the discount factor $\gamma$ rather than the more mysterious matrix $A$.}.

\subsection{Finite time analysis of mean-path TD}\label{subsec:mean path convergence rate}
We now combine the insights of the previous subsection to establish convergence rates for mean-path TD. These mirror the bounds for gradient descent given in Equations \eqref{eq: GD bound 1} and \eqref{eq: GD bound 2}, except for an additional dependence on the discount factor.  The first result bounds the distance between the value function under an averaged iterate and under the TD stationary point. This gives a comparatively slow $O(1/T)$ convergence rate, but does not depend at all on the conditioning of the feature covariance matrix. When this matrix is well conditioned, so the minimum eigenvalue $\omega$ of $\Sigma$ is not too small, the geometric convergence rate given in the second part of the theorem dominates. Note that by Lemma \ref{lemma:strong_conv}, bounds on $\| \theta_t - \theta^*\|_2$ always imply bounds on $\| V_{\theta_t} - V_{\theta^*}\|_D$.    
\begin{restatable}[]{thm}{}
\label{thrm:mean_path}
Consider a sequence of parameters $(\theta_0,\theta_1, \ldots)$  obeying the recursion 
\[
\theta_{t+1} = \theta_t + \alpha \bar{g}(\theta_t) \qquad t \in \mathbb{N}_0 = \{0,1,2,\ldots\},
\]
where $\alpha=(1-\gamma)/4$. Then, 
\[
\|V_{\theta^*} - V_{\bar{\theta}_T} \|_D^2 \,\, \leq  \frac{4\| \theta^* - \theta_0   \|_2^2}{T(1-\gamma)^2}
\]
and
\[
\| \theta^* - \theta_T \|_2^2  \,\, \leq  \exp\left\{ - \left(\frac{(1-\gamma)^2 \omega}{4}\right) T \right\}  \| \theta^* - \theta_0 \|_2^2. 
\]
\end{restatable}
\begin{proof}
With probability 1, for every $t\in \mathbb{N}_0$, we have 
\begin{eqnarray}
\label{eq:error_rec_mean_path}
\| \theta^* - \theta_{t+1}  \|_2^2  = \| \theta^* - \theta_t  \|_2^2 - 2\alpha (\theta^* - \theta_t)^\top\bar{g}(\theta_t)  +  \alpha^2 \| \bar{g}(\theta_t) \|^2_2.
\end{eqnarray} 
Applying Lemmas \ref{lemma:expected_gradient} and \ref{lemma:grad_norm_mean_path} and using a constant step-size of $\alpha=(1-\gamma)/4$, we get
\begin{eqnarray}
\| \theta^* - \theta_{t+1} \|_2^2  &\leq& \| \theta^* - \theta_t \|_2^2  - \left(2 \alpha(1-\gamma) - 4\alpha^2 \right)\| V_{\theta^*} - V_{\theta_t} \|_D^2 \nonumber \\
&=& \| \theta^* - \theta_t \|_2^2  - \left( \frac{(1-\gamma)^2}{4} \right)\| V_{\theta^*} - V_{\theta_t} \|_D^2. \label{eq: noiseless key inequality}
\end{eqnarray}
Then, 
\[
 \left(\frac{(1-\gamma)^2}{4} \right) \sum_{t=0}^{T-1} \| V_{\theta^*} - V_{\theta_t} \|_D^2 \leq \sum_{t=0}^{T-1} \left(\| \theta^* - \theta_t \|_2^2 - \| \theta^* - \theta_{t+1} \|_2^2\right)  \leq \| \theta^* - \theta_0 \|_2^2.
\]
Applying Jensen's inequality gives the first result:
\[ 
\| V_{\theta^*} - V_{\bar{\theta}_T} \|_D^2 \leq \frac{1}{T} \sum_{t=0}^{T-1} \| V_{\theta^*} - V_{\theta_t} \|_D^2 \leq  \frac{4\| \theta^* - \theta_0 \|_2^2}{(1-\gamma)^2 T}.
\]
Now, returning to \eqref{eq: noiseless key inequality}, and applying Lemma \ref{lemma:strong_conv} implies 
\begin{eqnarray*}
\|\theta^* - \theta_{t+1} \|_2^2 \leq \| \theta^* - \theta_t \|_2^2  - \left( \frac{(1-\gamma)^2}{4} \right) \omega \| \theta^* - \theta_t\|_2^2 &=& \left( 1- \frac{\omega (1-\gamma)^2}{4} \right)\| \theta^* - \theta_t \|_2^2 \\
&\leq& \exp\left\{ - \frac{\omega (1-\gamma)^2}{4} \right\} \| \theta^* - \theta_t \|_2^2,
\end{eqnarray*}
where the final inequality uses that $\left(1- \frac{\omega(1-\gamma)^2}{4}\right) \leq e^{\frac{-\omega(1-\gamma)^2}{4}}$. Repeating this inductively gives the desired result. 
\end{proof}

%-------------------------------------------------------------------------------
\section{Analysis for the i.i.d.\,observation model} \label{sec:iid_sampling}
%-------------------------------------------------------------------------------
	This section studies TD under an i.i.d.\,observation model, and establishes three explicit guarantees that mirror standard finite time bounds available for SGD. Specifically, we study a model where the random tuples observed by the TD algorithm are sampled i.i.d.\,from the stationary distribution of the Markov reward process. This means that for all states $s$ and $s'$,  
	\begin{equation}\label{eq: stochastic model}
	\Prob \left[(s_t, r_t, s'_t) = (s, \Rc(s,s'), s')\right] = \pi(s)\Pc(s'|s),
	\end{equation}
	and the tuples $\{(s_t, r_t, s'_t)\}_{t\in \mathbb{N}}$ are drawn independently across time. Note that the probabilities in Equation \eqref{eq: stochastic model} correspond to a setting where the first state $s_t$ is drawn from the stationary distribution, and then $s'_t$ is drawn from $\Pc(\cdot | s_t)$. This model is widely used for analyzing RL algorithms. See for example \cite{sutton2009convergent}, \cite{sutton2009fast}, \cite{korda2015td}, and \cite{dalal2017finite}.

	%This section  establishes two bounds on the the expected distance between the estimated value function of TD after $T$ iterations and its limiting value. Specifically, we bound $\mathbb{E} \left[ \norm{V_{\bar{\theta}_T} - V_{\theta^*}}^2_{D} \right]$, where the use of the norm $\|\cdot \|_{D}$ corresponds to measuring the the mean-squared gap between the value predictions under $V_{\bar{\theta}_T}$ and those under $V_{\theta^*}$. Recall that $\theta^*$ is defined as the unique solution to the projected Bellman equation $\Phi \theta^* = \Pi_{D} T_{\mu} \Phi \theta^*$, and it is known that the iterates of TD converge to $\theta^*$ under some technical conditions. 
	% $O\left(\frac{1}{\sqrt{T}(1-\gamma)^2}\right)$

Theorem \ref{thrm_iid} follows from a unified analysis that combines the techniques of the previous section with typical arguments used in the SGD literature. All bounds depend on $\sigma^2=\E[\|g_{t}(\theta^*)\|_2^2]=\E[\|g_{t}(\theta^*)-\bar{g}(\theta^*)\|_2^2]$, which roughly captures the variance of TD updates at the stationary point $\theta^*$. The bound in part (a) follows the spirit of work on so-called \emph{robust stochastic approximation} \citep{nemirovski2009robust}. It applies to TD with iterate averaging and relatively large step-sizes. The result is a simple bound on the mean-squared gap between the value predictions under the averaged iterate and the TD fixed point. The main strength of this result is that the step-sizes and the bound do not depend at all on the condition number of the feature covariance matrix. Note that the requirement that $\sqrt{T}\geq 8/(1-\gamma)$ is not critical; one can carry out analysis using the step-size $\alpha_0 = \min\{ (1-\gamma)/8, \sqrt{T}  \}$, but the bounds we attain only become meaningful in the case where $T$ is sufficiently large, so we chose to simplify the exposition. 
	
Parts (b) and (c) provide faster convergence rates in the case where the feature covariance matrix is well conditioned. Part (b) studies TD applied with a constant step-size, which is common in practice. In this case, the iterate $\theta_t$ will never converge to the TD fixed point, but our results show the expected distance to $\theta^*$ converges at an exponential rate below some level that depends on the choice of step-size. This is sometimes referred to as the rate at which the initial point $\theta_0$ is ``forgotten''. Bounds like this justify the common practice of starting with large step-sizes, and sometimes dividing the step-sizes in half once it appears error is no-longer decreasing.  Part (c) attains an $\mathcal{O}(1/T)$ convergence rate for a carefully chosen decaying step-size sequence. This step-size sequence requires knowledge of the minimum eigenvalue of the feature covariance matrix $\Sigma$, which plays a role similar to a strong convexity parameter in the optimization literature. In practice, this would need to be estimated, possibly by constructing a sample average approximation to the feature covariance matrix. The proof of part (c) closely follows an inductive argument presented in \cite{bottou2016optimization}. Recall that $\bar{\theta}_T= T^{-1} \sum_{t=0}^{T-1} \theta_t$ denotes the averaged iterate.
	\begin{restatable}[]{thm}{}
		\label{thrm_iid}
		Suppose TD is applied under the i.i.d.\,observation model and set $\sigma^2 = \E\left[ \|g_{t}(\theta^*)  \|^2_2 \right]$.  
		\begin{enumerate}[label=(\alph*)]
			\item For any $T \geq (8/(1-\gamma))^2$ and a constant step-size sequence $\alpha_0=\cdots = \alpha_T=\frac{1}{\sqrt{T}}$,
			\[
			\mathbb{E} \left[  \| V_{\theta^*} - V_{\bar{\theta}_T} \|_D^2 \right] \leq \frac{\| \theta^* - \theta_0 \|_2^2+2\sigma^2}{\ \sqrt{T}(1-\gamma)}. 
			\]
			\item For any constant step-size sequence $\alpha_0 = \cdots =\alpha_T  \leq \omega (1-\gamma)/8$, 
			\[
			\E \left[\|\theta^*-\theta_{T} \|_2^2\right] \leq \left(e^{- \alpha_0(1-\gamma) \omega T} \right) \|\theta^* - \theta_0\|_2^2 \,+\, \alpha_0 \left(\frac{2 \sigma^2}{(1-\gamma)\omega}\right).
			\]
			\item For a decaying step-size sequence $\alpha_t=\frac{\beta}{\lambda+t}$ with $\beta=\frac{2}{(1-\gamma)\omega}$ and $\lambda=\frac{16}{(1-\gamma)^2 \omega}$, 
			\[ 
			\E \left[\|\theta^*-\theta_{T} \|_2^2\right] \leq \frac{\nu}{\lambda+T} \quad \text{where} \quad \nu = \max\left\{ \frac{8 \sigma^2 }{ (1-\gamma)^2 \omega^2} \, , \, \frac{16\| \theta^*-\theta_0 \|_2^2}{(1-\gamma)^2 \omega}  \right\}.
			\]
		\end{enumerate}
	\end{restatable}
		Our proof is able to directly leverage Lemma \ref{lemma:expected_gradient}, but the analysis requires the following extension of Lemma \ref{lemma:grad_norm_mean_path} which gives an upper bound to the expected norm  of the stochastic gradient.

    \begin{restatable}[]{lem}{}
    \label{lemma:norm_bd_grad} 
        For any fixed $\theta\in\mathbb{R}^d$,
		 $\E\left[\|g_t(\theta) \|^2_2\right] \leq  2 \sigma^2 + 8 \| V_{\theta}-V_{\theta^*}\|^2_D$ where $\sigma^2 = \mathbb{E}\left[ \| g_t(\theta^*) \|^2_2 \right]$. 
	\end{restatable}
    \begin{proof}
		For brevity of notation, set $\phi= \phi(s_t)$ and $\phi'=\phi(s'_t)$. Define $\xi = (\theta^*-\theta)^\top \phi$ and $\xi' = (\theta^*-\theta)^\top \phi'$. By stationarity, $\xi$ and $\xi'$ have the same marginal distribution and $\E[\xi^2] = \| V_{\theta^*} - V_{\theta}\|_D^2$, following the same argument as in Lemma \ref{lemma:expected_gradient}. Using the formula for $g_{t}(\theta)$ in Equation \eqref{eq:grad_TD(0)}, we have
		\begin{eqnarray*}
			\E\left[ \|g_t(\theta) \|^2_2\right] &\leq&  \E\left[ \left(\|g_t(\theta^*) \|_2 + \|g_t(\theta)-g_t(\theta^*) \|_2 \right)^2 \right]  \\ 
            &\leq & 2 \E \left[ \|g_t(\theta^*)\|^2_2 \right] + 2 \E \left[ \|g_t(\theta)-g_t(\theta^*) \|^2_2 \right] \\ 
			&=& 2 \sigma^2 + 2 \mathbb{E} \left[ \left\| \phi\left(\phi - \gamma \phi' \right)^\top (\theta^*-\theta)  \right\|^2_2 \right]\\
			&=& 2 \sigma^2 + 2 \mathbb{E} \left[ \left\| \phi (\xi - \gamma \xi') \right\|^2_2 \right]\\
			&\leq& 2 \sigma^2 + 2 \E\left[ |\xi - \gamma \xi'|^2 \right] \\
			&\leq& 2 \sigma^2 + 4 \left( \E\left[ |\xi|^2 \right] + \gamma^2 \E\left[ |\xi'|^2 \right] \right) \\
			&\leq& 2 \sigma^2 + 8 \| V_{\theta^*} - V_{\theta} \|^2_D,
		\end{eqnarray*}
where we used the assumption that $\| \phi \|^2_2 \leq 1$. The second inequality uses the basic algebraic identity $(x+y)^2 \leq 2\max\{x,y\}^2 \leq 2x^2 + 2y^2$, along with the monotonicity of expectation operators. % along with the fact that $\E[(X+Y)^2] \leq 2 \E[X^2] + \E[Y^2]$ for any two random variables $X, Y$ which follows via the $c_r$ inequality\footnote{For any two random variables $X, Y$: $\E[\abs{(X+Y)}^r] \leq 2^{r-1} \E[\abs{X}^r] + \E[\abs{Y}^r]$ for $r \geq 1$}.
	\end{proof}
 
% \jb{We pause here to remark on a consequence of the i.i.d\,model which considerably simplifies the proof.} Let $\{\mathcal{H}_t\}_{t \geq 0}$ denote the natural filtration with respect to $\{\theta_t\}$. As will be discussed further in next section, the proof below relies critically on the fact that the random tuple $O_t = (s_t, r_t, s'_t)$ is independent of the iterate, $\theta_t$ which is $\Hc_{t-1}$ measurable. Thus, 
% 	\[
%     \E[ g_{t}(\theta_t) | \Hc_{t-1}] = \bar{g}(\theta_t). 
%     \]
% %     \[
% %     \E[ g_{t}(\theta_t)] =\E\left[ \E\left[  g_{t}(\theta_t)  \mid \Hc_{t-1}\right]\right]= \bar{g}(\theta_t).
% %     \]
% This implies Martingale noise, $\E\left[ g_{t}(\theta_t) - \bar{g}(\theta_t) \right] = \E\left[ \E\left[ g_{t}(\theta_t)  \mid \Hc_{t-1}\right] - \bar{g}(\theta_t) \right] = 0.$
%     In an analogous way, independence of the tuple $O_t$ and the iterate $\theta_t$ under the i.i.d.\,model lets us conclude by Lemma \ref{lemma:norm_bd_grad} that $\E\left[\| g_{t}(\theta_t) \|^2_2\right] \leq 2\sigma^2 + 8\E\left[\| V_{\theta^*} - V_{\theta_t} \|^2_D\right]$. 

Using this we give a proof of Theorem \ref{thrm_iid} below. Let us remark here on a consequence of the i.i.d noise model that considerably simplifies the proof. Until now, we have often developed properties of the TD updates $g_{t}(\theta)$ applied to an arbitrary, but fixed, vector $\theta\in \mathbb{R}^d$. For example, we have given an expression for $\bar{g}(\theta):=\E[g_{t}(\theta)]$, where this expectation integrates over the random tuple $O_t=(s_t, r_t, s'_t)$ influencing the TD update. In the i.i.d noise model, the current iterate, $\theta_t$, is independent of the tuple $O_t$, and so $\E[g_{t}(\theta_t) | \theta_t] = \bar{g}(\theta_t)$. In a similar manner, after conditioning on $\theta_t$, we can seamlessly apply  Lemmas \ref{lemma:expected_gradient} and \ref{lemma:norm_bd_grad}, as is done in inequality \eqref{eq:using iid noise}  of the proof below.
\begin{proof} 
 		The TD algorithm updates the parameters as: $\theta_{t+1} = \theta_t + \alpha_t g_t(\theta_t)$. Thus, for each $t\in \mathbb{N}_0$, we have, 
 		\[ 
 		\|\theta^*-\theta_{t+1} \|_2^2 \,\, = \, \| \theta^* -\theta_t \|_2^2  - 2\alpha_t   g_{t}(\theta_t)^\top (\theta^* - \theta_t) +  \alpha_t^2 \| g_{t}(\theta_t) \|_2^2.
 		\]
 		Under the hypotheses of (a), (b) and (c), we have that $\alpha_t \leq (1-\gamma)/8$. %\jb{Applying Lemma \ref{lemma:expected_gradient} and Lemma \ref{lemma:norm_bd_grad} implies that with probability 1,
%         \[ 
% 		\|\theta^*-\theta_{t+1} \|_2^2 \,\, \leq \,\, \| \theta^* -\theta_t \|_2^2 - \left( 2\alpha_t (1-\gamma) - 8\alpha_t^2 \right) \| V_{\theta^*} - V_{\theta_t} \|_D^2 + 2\alpha_t^2 \sigma^2
% 		\]}
        Taking expectations and applying Lemma \ref{lemma:expected_gradient} and Lemma \ref{lemma:norm_bd_grad} implies,
		\begin{eqnarray}
		\E \left[\|\theta^*-\theta_{t+1} \|_2^2\right] &=& \E \left[\| \theta^* -\theta_t \|_2^2 \right] - 2\alpha_t \E\left[ g_{t}(\theta_t)^\top (\theta^* - \theta_t) \right]  + \alpha_t^2 \E\left[ \| g_{t}(\theta_t) \|_2^2  \right] \nonumber \\
        &=& \E \left[\| \theta^* -\theta_t \|_2^2 \right] - 2\alpha_t \E\left[\E\left[  g_{t}(\theta_t)^\top (\theta^* - \theta_t)  \,|\, \theta_{t}\right] \right]  + \alpha_t^2 \E\left[ \E\left[ \| g_{t}(\theta_t) \|_2^2 \,|\, \theta_t \right] \right] \nonumber \\
        &\leq& \E \left[\| \theta^* -\theta_t \|_2^2 \right] - \left( 2\alpha_t (1-\gamma) - 8\alpha_t^2 \right)\mathbb{E}\left[ \| V_{\theta^*} - V_{\theta_t} \|_D^2 \right]  + 2\alpha_t^2 \sigma^2 \label{eq:using iid noise} \\
		& \leq & \E \left[\| \theta^* -\theta_t \|_2^2 \right] - \alpha_t (1-\gamma) \mathbb{E}\left[ \| V_{\theta^*} - V_{\theta_t} \|_D^2 \right]  + 2\alpha_t^2 \sigma^2. \label{eq:error decomposition iid}
		\end{eqnarray}
        The inequality \eqref{eq:using iid noise} follows from Lemmas \ref{lemma:expected_gradient} and \ref{lemma:norm_bd_grad}. The application of these lemmas uses that the random tuple $O_t = (s_t, r_t, s'_t)$ influencing $g_{t}(\cdot)$ is independent of the iterate, $\theta_t$.    
		\paragraph{Part (a).} Consider a constant step-size of $\alpha_T = \cdots =\alpha_0 = 1/\sqrt{T}$. Starting with Equation \eqref{eq:error decomposition iid} and summing over $t$ gives
		\[ 
		\mathbb{E} \left[  \sum_{t=0}^{T-1} \| V_{\theta^*} - V_{\theta_t} \|_D^2 \right] \leq  \frac{\| \theta^* - \theta_0 \|_2^2}{\alpha_0 (1-\gamma) } + \frac{2 \alpha_0 T \sigma^2}{(1-\gamma)} =  \frac{\sqrt{T} \| \theta^* - \theta_0 \|_2^2}{(1-\gamma) }+ \frac{2 \sqrt{T} \sigma^2}{(1-\gamma)}. 
		\]
		We find 
		\[ 
		\mathbb{E} \left[  \| V_{\theta^*} - V_{\bar{\theta}_T} \|_D^2 \right] \leq \frac{1}{T} \mathbb{E} \, \left[  \sum_{t=0}^{T-1} \| V_{\theta^*} - V_{\theta_t} \|_D^2 \right] \leq \frac{\| \theta^* - \theta_0 \|_2^2+2\sigma^2}{\ \sqrt{T}(1-\gamma) }. 
		\]
		\paragraph{Part (b).}
		Consider a constant step-size of $\alpha_0 \leq \omega (1-\gamma)/8$. Applying Lemma \ref{lemma:strong_conv} to Equation \eqref{eq:error decomposition iid} implies
		\begin{equation}
		\E \left[\|\theta^*-\theta_{t+1} \|_2^2\right] \leq  \left(1- \alpha_0 (1-\gamma)\omega \right) \E \left[\| \theta^* -\theta_t \|_2^2 \right]  + 2 \alpha_0^2 \sigma^2.
		\end{equation}
		Iterating this inequality establishes that for any $T \in \mathbb{N}_0$,
		\[
		\E \left[\|\theta^*-\theta_T \|_2^2\right] \leq \left(1- \alpha_0 (1-\gamma)\omega \right)^{T} \E \left[ \| \theta^* - \theta_0 \|_2^2 \right] + 2 \alpha_0^2 \sigma^2 \sum_{t=0}^{\infty} \left(1- \alpha_0 (1-\gamma)\omega \right)^{t}.
		\]
		The result follows by solving the geometric series and using that $\left(1- \alpha_0 (1-\gamma)\omega\right) \leq e^{-\alpha_0(1-\gamma)\omega}$.
		\paragraph{Part (c).}
		Note that by the definitions of $\nu,\lambda$ and $\beta$, we have 
		\[ 
		\nu = \max\{ 2\beta^2 \sigma^2 \, , \,  \lambda \| \theta^* - \theta_0 \|_2^2 \}.
		\]
		We then have $\|\theta^* - \theta_0 \|_2^2 \,\, \leq \frac{\nu}{\lambda}$ by the definition of $\nu$. Proceeding by induction, suppose $\E\left[\|\theta^* - \theta_t \|_2^2\right] \leq \frac{\nu}{\lambda+t}$. Then, 
		\begin{eqnarray*}
			\E \left[\|\theta^*-\theta_{t+1} \|_2^2\right] &\leq& \left(1- \alpha_t (1-\gamma)\omega \right) \E \left[\| \theta^* -\theta_t \|_2^2 \right]  + 2\alpha_t^2 \sigma^2 \\ 
			& \leq & \left(1- \frac{(1-\gamma)\omega \beta}{\hat{t}} \right) \frac{\nu}{\hat{t}} + \frac{2\beta^2 \sigma^2}{\hat{t}^2} \qquad [\text{where  } \hat{t}\equiv \lambda+t ] \\
			&=& \left(\frac{\hat{t}- (1-\gamma)\omega \beta}{\hat{t}^2} \right) \nu + \frac{2\beta^2 \sigma^2}{\hat{t}^2} \\
			&=& \left(\frac{\hat{t}- 1}{\hat{t}^2} \right) \nu + \frac{2\beta^2 \sigma^2 - \left((1-\gamma)\omega \beta -1 \right)\nu}{\hat{t}^2} \\
			&=& \left(\frac{\hat{t}- 1}{\hat{t}^2} \right) \nu + \frac{2\beta^2 \sigma^2 - \nu}{\hat{t}^2} \qquad \qquad \,\, [\text{using  } \beta = \frac{2}{(1-\gamma)\omega} ] \\
			&\leq & \frac{\nu }{\hat{t} +1 },
		\end{eqnarray*}
		where the final inequality uses that $2\beta^2 \sigma^2 - \nu \leq 0$, which holds by the definition of $\nu$ and the fact that $\hat{t}^2 \geq (\hat{t}-1)(\hat{t}+1)$. 
\end{proof}

\section{Analysis for the Markov chain observation model: Projected TD algorithm} \label{sec:markov_chain_analyis}
%----------------------------------------------------------------------------------------------------------
In Section \ref{sec:iid_sampling}, we developed a method for analyzing TD under an i.i.d.\,sampling model in which tuples are drawn independently from the stationary distribution of the underlying MDP. But a more realistic setting is one in which the observed tuples used by TD are gathered from a single trajectory of the Markov chain.  In particular, if for a given sample path the Markov chain visits states $(s_0, s_1,\ldots s_t, \ldots)$, then these are processed into tuples $O_t= (s_t, r_t=\Rc(s_t, s_{t+1}), s_{t+1})$ that are fed into the TD algorithm. Mathematical analysis is difficult since the tuples used by the algorithm can be highly correlated with each other. We outline the main challenges below.
%--------------------- Gradient bias ------------------%

\paragraph{Challenges in the Markov chain noise model.}
In the i.i.d.\,observation setting, our analysis relied heavily on a Martingale property of the noise sequence. This no longer holds in the Markov chain model due to strong dependencies between the noisy observations. To understand this, recall the expression of the negative gradient,
\begin{equation}\label{eq:mc_gradient_expression}
g_t(\theta) = \Big( r_t + \gamma \phi(s_{t+1})^\top \theta - \phi(s_t)^\top \theta \Big) \phi(s_t).
\end{equation}
To make the statistical dependencies more transparent, we can overload notation to write this as $g(\theta, O_t )\equiv g_{t}(\theta)$, where $O_t=(s_t, r_t, s_{t+1})$. Assuming the sequence of states is stationary, we have defined the function $\bar{g} : \mathbb{R}^d \to \mathbb{R}^d$ by $\bar{g}(\theta)= \E[g(\theta, O_t) ]$, where, since $\theta$ is non-random, this expectation integrates over the marginal distribution of the tuple $O_t$. However, $\E[g(\theta_t, O_t) \mid \theta_t=\theta ] \neq \bar{g}(\theta)$ because $\theta_{t}$ is a function of past tuples $\{O_1, \ldots, O_{t-1}\}$, potentially introducing strong dependencies between $\theta_{t}$ and $O_t$. Similarly, in general $\E[g(\theta_t, O_t) - \bar{g}(\theta_t) ] \neq 0$, indicating bias in the algorithm's gradient evaluation. A related challenge arises in trying to control the norm of the gradient step, $\E[\|g_{t}(\theta_t)\|_2^2]$. Lemma \ref{lemma:norm_bd_grad} does not yield a bound due to coupling between the iterate $\theta_t$ and the observation $O_t$. 

Our analysis uses an information-theoretic technique to control for this coupling and explicitly account for the gradient basis. This technique may be of broader use in analyzing reinforcement learning and stochastic approximation algorithms. However, our analysis also requires some strong regularity conditions, as outlined below.

\paragraph{Projected TD algorithm.}
Our technique for controlling the gradient bias relies critically on a condition that, when step-sizes are small, the iterates $(\theta_t)_{t\in \mathbb{N}_0}$ do not change too rapidly. This is the case as long as norms of the gradient steps do not explode. For tractability, we modify the TD algorithm itself by adding a projection step that ensures gradient norms are uniformly bounded across time. In particular, starting with an initial guess of $\theta_0$ such that $\norm{\theta_0}_2 \leq R$, we consider the Projected TD algorithm, which iterates 
\begin{equation}\label{eq: projected TD recursion}
\theta_{t+1} = \Pi_{2,R} \left(\theta_t + \alpha_t g_{t}(\theta_t) \right) \,\, \forall \,\, t \in \mathbb{N}_0,
\end{equation}
where 
\begin{eqnarray*}
\Pi_{2,R}(\theta) = \argmin_{\theta': \norm{\theta'}_2 \leq R} \, \norm{\theta - \theta'}_2
\end{eqnarray*}
is the projection operator onto a norm ball of radius $R<\infty$. The subscript $2$ on the operator indicates that the projection is with respect the unweighted Euclidean norm. This should not be confused with the projection operator $\Pi_{D}$ used earlier, which projects onto the subspace of approximate value functions with respect to a weighted norm.   

One may wonder whether this projection step is practical. We note that, from a computational perspective, it only involves rescaling of the iterates, as $\Pi_{2,R}(\theta) = R\theta/\|\theta\|$ if $\|\theta\|_2 > R$ and is simply $\theta$ otherwise. In addition, Subsection \ref{sec:comment_projection} suggests that by using aprori bounds on the value function, it should be possible to estimate a projection radius containing the TD fixed point. However, at this stage, we view this mainly as a tool that enables clean finite time analysis, rather than a practical algorithmic proposal. 

It is worth mentioning that projection steps have a long history in the stochastic approximation literature, and many of the standard analyses for stochastic gradient descent rely on projections steps to control the norm of the gradient \citep{kushner2010stochastic, lacoste2012simpler,bubeck2015convex,nemirovski2009robust}. 

%and we need to account for this \emph{gradient bias}: $\mathbb{E} \left[ g_t(\theta_t)|\mathcal{H}_{t-1} \right] - \bar{g}(\theta_t)$. Our analysis in this Section uncovers key insights to upper bound it. Essentially, our result in Theorem \ref{thrm:mc_ana} shows that this gradient bias leads to an additional factor in the convergence bounds relative to those shown in Theorem \ref{thrm_iid}.
%
%Another challenge for analyzing the Markov sampling model is to bound the variance of the stochastic gradient analogous to Lemma \ref{lemma:norm_bd_grad}, which shows an upper bound to the variance under the i.i.d\, sampling model. However, recall that the claim: $\E[\xi^2] = \| V_{\theta^*} - V_{\theta}\|_D^2$ which is used to prove Lemma \ref{lemma:norm_bd_grad} only holds when the tuples are sampled from the steady-state distribution. To overcome this challenge under the Markov model, we modify the TD algorithm itself with an additional projection step.  The boundedness of iterates, $\norm{\theta_t} \leq R \,\, \forall \,\, t$, enables a uniform bound on the norm of the gradient, as shown in Lemma \ref{lemma:mc_norm_grad} which is important to our analysis. 

\paragraph{Structural assumptions on the Markov reward process.}
To control the statistical bias in the gradient updates--which is the main challenge under the Markov observation model--we assume that the Markov chain mixes at a uniform geometric rate, as stated below.
\begin{assumption}\label{as:6}
There are constants $m > 0$ and $\rho \in (0,1)$ such that
\begin{eqnarray*}
\sup_{s \in \mathcal{S}} \,\, \text{d}_{\text{TV}} \left(\mathbb{P}(s_t \in \cdot | s_0 = s), \pi\right) \leq m \rho^t \quad \forall \,\, t \in \mathbb{N}_0,
\end{eqnarray*}
where $\text{d}_{\text{TV}}(P,Q)$ denotes the total-variation distance between probability measures $P$ and $Q$. In addition, the initial distribution of $s_0$ is the steady-state distribution $\pi$, so $(s_0,s_1,\ldots)$ is a stationary sequence. 
\end{assumption}
This uniform mixing assumption always holds for irreducible and aperiodic Markov chains \citep{levin2017markov}. We emphasize that the assumption that the chain begins in steady-state is not essential: given the uniform mixing assumption, we can always apply our analysis after the Markov chain has approximately reached its steady-state. However, adding this assumption allows us to simplify many mathematical expressions. Another useful quantity for our analysis is the mixing time which we define as
\begin{eqnarray}
\label{eq:tau_mix}
\tau^{\mix}(\epsilon) = \min \{ t \in \mathbb{N}_0 \,\, | \,\, m \rho^t \leq \epsilon \}.
\end{eqnarray}
For interpreting the bounds, note that from Assumption \ref{as:6},
\[
\tau^{\mix}(\epsilon) \sim \frac{\log(1/\epsilon)}{\log(1/\rho)} \qquad \text{as } \epsilon \to 0.
\]
We can therefore evaluate the mixing time at very small thresholds like $\epsilon=1/T$ while only contributing a logarithmic factor to the bounds. 

\paragraph{A bound on the norm of the gradient:}
Before proceeding, we also state a bound on the euclidean norm of the gradient under TD(0) that follows from the uniform bound on rewards, along with feature normalization\footnote{Recall that we assumed $\norm{\phi(s)}_2 \leq 1$ for all $s \in \mathcal{S}$ and $\abs{\Rc(s,s')} \leq r_{\max}$ for all $s,s' \in \mathcal{S}$} and boundedness of the iterates through the projection step. Under projected TD(0) with projection radius $R$, this lemma implies that $\|g_{t}(\theta_t)\|_2 \leq (r_{\max}+2R)$. This gradient bound plays an important role in our convergence bounds. 
\begin{restatable}[]{lem}{MCGradientNormBound}
\label{lemma:mc_grad_norm}
For all $\theta \in \mathbb{R}^d$,  $\norm{g_t(\theta)}_2 \leq r_{\max}+2\|\theta\|_2$ with probability 1.
\end{restatable}
\begin{proof}
Using the expression of $g_{t}(\theta)$ in Equation \eqref{eq:mc_gradient_expression}, we have
\begin{eqnarray*}
\|g_{t}(\theta)\|_2 \leq |r_t + (\gamma\phi(s'_{t})-\phi(s_t))^\top \theta | \, \| \phi(s_t)\|  
&\leq& r_{\max} + \| \gamma \phi(s'_{t}) - \phi(s_t) \|_2 \|\theta\|_2  \\ 
&\leq& r_{\max}+2\|\theta\|.
\end{eqnarray*}
\end{proof}

%--------------------- Main Result ------------------%
\subsection{Finite time bounds}
Following Section \ref{sec:iid_sampling}, we state several finite time bounds on the performance of the Projected TD algorithm. As before, in the spirit of robust stochastic approximation \citep{nemirovski2009robust}, the bound in part (a) gives a comparatively slow convergence rate of $\tilde{\mathcal{O}}(1/\sqrt{T})$, but where the bound and step-size sequence are independent of the conditioning of the feature covariance matrix $\Sigma$. The bound in part (c) gives a faster convergence rate in terms of the number of samples $T$, but the bound and as well as the step-size sequence depend on the minimum eigenvalue $\omega$ of $\Sigma$. Part (b) confirms that for sufficiently small step-sizes, the iterates converge at an exponential rate to within some radius of the TD fixed-point, $\theta^*$. 

It is also instructive to compare the bounds for the Markov model vis-a-vis the i.i.d.\,model. One can see that in the case of part(b) for the Markov chain setting, a $\mathcal{O}\left( G^2 \tau^{\mix}(\alpha_0) \right)$ term controls the limiting error due to gradient noise. This scaling by the mixing time is intuitive, reflecting that roughly every cycle of $\tau^{\mix}(\cdot)$ observations provides as much information as a single independent sample from the stationary distribution. We can also imagine specializing the results to the case of Projected TD under the i.i.d.\,model, thereby eliminating all terms depending on the mixing time. We would attain bounds that mirror those in Theorem \ref{thrm_iid}, except that the gradient noise term $\sigma^2$ there would be replaced by $G^2$. This is a consequence using $G$ as a uniform upper bound on the gradient norm in the proof, which is possible because of the projection step. 

\begin{restatable}[]{thm}{MCBound}
\label{thrm:mc_ana}
Suppose the Projected TD algorithm is applied with parameter $R \geq \| \theta^*\|_2$ under the Markov chain observation model with Assumption \ref{as:6}. %et $\bar{\theta}_T= T^{-1} \sum_{t=1}^{T} \theta_t$ denote the averaged iterate and 
Set $G= (r_{\max} + 2R)$. Then the following claims hold. 
\begin{enumerate}[label=(\alph*)]
\item With a constant step-size sequence $\alpha_0 = \cdots= \alpha_T = 1/\sqrt{T}$,
\begin{eqnarray*}
\mathbb{E}\left[\norm{V_{\theta^*} - V_{\bar{\theta}_T}}^2_{D}\right] \leq \frac{\norm{\theta^* - \theta_0}_2^2 + G^2 \left(9+12\tau^{\mix}(1/\sqrt{T})\right)}{2\sqrt{T}(1-\gamma)}.
\end{eqnarray*}

\item With a constant step-size sequence $\alpha_0 =\cdots =\alpha_T < 1/(2 \omega (1-\gamma))$, 
\begin{eqnarray*}
\mathbb{E} \left[\norm{\theta^*-\theta_{T}}_2^2\right] &\leq& \left(e^{-2\alpha_0(1-\gamma)\omega T}\right) \norm{\theta^*-\theta_0}_2^2 + \alpha_0\left(\frac{ G^2 \left(9+12\tau^{\mix}(\alpha_0) \right)}{2(1-\gamma)\omega}\right).
\end{eqnarray*}

\item With a decaying step-size sequence $\alpha_t = 1/(\omega (t+1) (1-\gamma))$ for all $t\in \mathbb{N}_0$,
\begin{eqnarray*}
\mathbb{E}\left[\norm{V_{\theta^*} - V_{\bar{\theta}_T}}^2_{D}\right] \leq \frac{ G^2 \left(9+24\tau^{\mix}(\alpha_T)\right)}{T(1-\gamma)^2\omega} \left(1+\log T\right),
\end{eqnarray*}
\end{enumerate}
%where $\tau^{\mix}(\alpha_T) := \tau^{\mix}\left(\frac{1}{(T+1)(1-\gamma)\omega}\right)$. 
\end{restatable}

\paragraph{Remark 1:} The proof of part (c) also implies an $\tilde{\mathcal{O}}(1/T)$ convergence rate for the iterate $\theta_{T}$ itself; similar to the $\mathcal{O}(1/T)$ convergence shown for the i.i.d.\,case, in part (c) of Theorem \ref{thrm_iid}.

\paragraph{Remark 2:} It is likely possible to eliminate the $\log T$ term in the numerator of part (c) to get a $\mathcal{O}(1/T)$ convergence rate. One approach is to use a different weighting of the iterates when averaging, as in \cite{lacoste2012simpler}. For brevity and simplicity, we do not pursue this direction.

\subsection{Choice of the projection radius} 
\label{sec:comment_projection}
We briefly comment on the choice of the projection radius, $R$. Note that Theorem \ref{thrm:mc_ana} assumes that $\norm{\theta^*}_2 \leq R$, so the TD limit point lies within the projected ball. How do we choose such an $R$ when $\theta^*$ is unknown? It turns out we can use Lemma \ref{lemma:BO_prop}, which relates the value function at the limit of convergence $V_{\theta^*}$ to the true value function, to give a conservative upper bound. This is shown in the proof of the following lemma. 
\begin{restatable}[]{lem}{}
%Let $\theta^*$ be the limit point of TD(0). Then, 
$\| \theta^* \|_{\Sigma} \leq \frac{2r_{\rm max}}{(1-\gamma)^{3/2}}$ and hence $\|\theta^*\|_2 \leq \frac{2 r_{\rm max}}{\sqrt{\omega}(1-\gamma)^{3/2}}$. 
\end{restatable}
\begin{proof}
Because rewards are uniformly bounded, $|V_\mu (s)|\leq r_{\max} / (1-\gamma)$ for all $s \in \Sc$. Recall that $V_{\mu}$ denotes the true value function of the Markov reward process. This implies that
\begin{eqnarray*}
\norm{V_{\mu}}_D \leq \norm{V_{\mu}}_\infty \leq \frac{r_{\max}}{(1-\gamma)}.
\end{eqnarray*}
Lemma \ref{lemma:BO_prop} along with simple matrix inequalities enable a simple upper bound on $\norm{\theta^*}_2$. We have 
\begin{eqnarray*}
\norm{V_{\theta^*} - V_{\mu} }_{D} &\leq& \frac{1}{\sqrt{1-\gamma^2}} \norm{V_{\mu} - \Pi_{D} V_{\mu}}_{D} \leq \frac{1}{\sqrt{1-\gamma^2}} \norm{ V_{\mu}}_{D} \leq \frac{1}{\sqrt{1-\gamma}} \norm{ V_{\mu}}_{D},
\end{eqnarray*}
where the penultimate inequality holds by the Pythagorean theorem. By the reverse triangle inequality we have $\abs{ \norm{V_{\theta^*}}_D - \norm{V_{\mu}}_D } \leq \norm{V_{\theta^*} - V_{\mu} }_{D}$. Thus, 
\begin{eqnarray*}
\norm{V_{\theta^*}}_D  &\leq& \norm{V_{\theta^*} - V_{\mu} }_{D} + \norm{V_{\mu}}_D
\leq \frac{2}{\sqrt{1-\gamma}} \norm{ V_{\mu} }_{D} 
\leq \frac{2}{\sqrt{1-\gamma}} \frac{r_{\max}}{(1-\gamma)}.
\end{eqnarray*}
Recall from Section \ref{sec:formulation} we have, $\|V_{\theta^*}\|_D = \| \theta^*\|_{\Sigma}$ which establishes first part of the claim. The second claim uses that $\| \theta^*\|_{\Sigma} \geq \omega \| \theta^* \|_2$ which follows by Lemma \ref{lemma:strong_conv}. 
\end{proof}
It is important to remark here that this bound is \emph{problem dependent} as it depends on the minimum eigenvalue  $\omega$ of the steady-state feature covariance matrix $\Sigma$. We believe that estimating $\omega$ online would make the projection step practical to implement.

%------Gradient Bias---------------------------------------------------%

\subsection{Analysis}
We now present the key analysis used to establish Theorem \ref{thrm:mc_ana}. Throughout, we assume the conditions of the theorem hold:
we consider the Markov chain observation model with Assumption \ref{as:6} and study the Projected TD algorithm applied with parameter $R \geq \| \theta^*\|_2$ and some step-size sequence $(\alpha_0,\cdots, \alpha_T)$.  

We fix some notation throughout the scope of this subsection. Define the set $\Theta_{R}= \{\theta \in \mathbb{R}^d : \| \theta\|_2 \leq R \}$, so $\theta_t \in \Theta_R$ for each $t$ because of the algorithm's projection step. Set $G = (r_{\max} + 2R)$, so $\| g_{t}(\theta)\|_2 \leq G$ for all $\theta \in \Theta_R$ by Lemma \ref{lemma:mc_grad_norm}. Finally, we set
\[
\zeta_t(\theta) \equiv  \left(g_t(\theta) - \bar{g}(\theta)\right)^\top(\theta - \theta^*) \quad \forall \,\, \theta \in \Theta_R,
\]
which can be thought of as the error in the evaluation of gradient-update under parameter $\theta$ at time $t$.

Referring back to the analysis of the i.i.d.~observation model, one can see that an error decomposition given in Equation \eqref{eq:error decomposition iid} is the crucial component of the proof.  The main objective in this section is to establish two key lemmas that yield a similar decomposition in the Markov chain observation model. The result can be stated cleanly in the case of a constant step-size. If $\alpha_0= \dots= \alpha_T=\alpha$, we show
\begin{eqnarray} \nonumber
 \E \left[ \|\theta^*-\theta_{t+1} \|_2^2 \right] &\leq& \E \left[ \| \theta^* -\theta_t \|_2^2 \right]  - 2\alpha (1-\gamma)\E \left[ \| V_{\theta^*}-V_{\theta_t} \|_D^2 \right] +\E[\alpha \zeta_t(\theta_t)] +\alpha^2 G^2 \\  
 &\leq& \E \left[ \| \theta^* -\theta_t \|_2^2 \right]  - 2\alpha (1-\gamma)\E \left[ \| V_{\theta^*}-V_{\theta_t} \|_D^2 \right] + \alpha^2 \left(5+6\tau^{\mix}(\alpha)\right) G^2. \label{eq:mc_clean_error_decomp}
\end{eqnarray}
The first inequality follows from Lemma \ref{lemma:mc_grad_decomp}. The second follows from Lemma \ref{lemma:mc_grad_bias_bd}, which in the case of a constant step-size $\alpha$ shows $\E[\alpha \zeta_t(\theta_t)]\leq G^2(4+6\tau^{\mix}(\alpha)) \alpha^2$. Notice that bias in the gradient enters into the analysis as if by scaling the magnitude of the noise in gradient evaluations by a factor of the mixing time. From this decomposition, parts (a) and (b) of Theorem \ref{thrm:mc_ana} follow by essentially copying the proof of Theorem \ref{thrm_iid}. Similar, but messier, inequalities hold for any decaying step-size sequence, which allows us to establish part (c).

\subsubsection{Error decomposition under Projected TD}
The next lemma establishes a recursion for the error under projected TD(0) that hold for each sample path. 
\begin{restatable}[]{lem}{lemmamcgraddecomp}
\label{lemma:mc_grad_decomp} With probability 1, for every $t\in \mathbb{N}_0$,
\[
\| \theta^* - \theta_{t+1} \|_2^2 \leq \| \theta^* - \theta_t \|_2^2 - 2\alpha_t (1-\gamma) \| V_{\theta^*}-V_{\theta_t} \|_D^2 + 2\alpha_t\zeta_t(\theta_t) +\alpha_t^2 G^2.
\]
\end{restatable}
\begin{proof}
From the projected TD(0) recursion in Equation \eqref{eq: projected TD recursion}, for any $t\in \mathbb{N}_0$,
\begin{eqnarray*}
\| \theta^* - \theta_{t+1} \|_2^2&=& \| \theta^* - \Pi_{2,R}(\theta_{t}+\alpha_t g_{t}(\theta_t))  \|_2^2  \\
 &=&  \|\Pi_{2,R}(\theta^*) - \Pi_{2,R}(\theta_{t}+\alpha_t g_{t}(\theta_t)) \|_2  \\
&\leq & \| \theta^*-\theta_{t}-\alpha_t g_{t}(\theta_t)  \|_2^2 \\
&=& \| \theta^* - \theta_t \|_2^2 - 2\alpha_t g_{t}(\theta_t)^\top (\theta^*-\theta_t) + \alpha_t^2 \| g_{t}(\theta_t)\|_2^2 \\
&\leq &  \| \theta^* - \theta_t \|_2^2 - 2\alpha_t g_{t}(\theta_t)^\top (\theta^*-\theta_t) + \alpha_t^2 G^2. \\
&=&  \| \theta^* - \theta_t \|_2^2 - 2\alpha_t \bar{g}(\theta_t)^\top (\theta^*-\theta_t) + 2\alpha_t\zeta_{t}(\theta_t) + \alpha_t^2 G^2.\\
&\leq&  \| \theta^* - \theta_t \|_2^2 - 2\alpha_t (1-\gamma) \| V_{\theta^*}-V_{\theta_t} \|_D^2 + 2\alpha_t\zeta_t(\theta_t) +\alpha_t^2 G^2.
\end{eqnarray*}
The first inequality used that orthogonal projection operators onto a convex set are non-expansive\footnote{Let $\Pc_{\mathcal{C}}(x)= \argmin_{x'\in \mathcal{C}} \|x'-x\|$ denote the projection operator onto a closed, non-empty, convex set $\mathcal{C}\subset \mathbb{R}^d$.  Then $\|\Pc_{\mathcal{C}}(x) - \Pc_{\mathcal{C}}(y) \| \leq \|x - y\|$ for all vectors $x$ and $y$.}, the second used Lemma \ref{lemma:mc_grad_norm} together with the fact that $\|\theta_t\|_2 \leq R$ due to projection, and the third used Lemma \ref{lemma:expected_gradient}. 
\end{proof}
By taking expectation of both sides, this inequality could be used to produce bounds in the same manner as in the previous section, except that in general $\E[\zeta_t(\theta_t)] \neq 0$ due to bias in the gradient evaluations.

\subsubsection{Information--theoretic techniques for controlling the gradient bias}\label{subsec: info}
The uniform mixing condition in Assumption \ref{as:6} can be used in conjunction with some information theoretic inequalities to control the magitude of the gradient bias. This section presents a general lemma, which is the key to this analysis. We start by reviewing some important properties of information-measures.  %in the next subsection. We then state a general lemma that allows us to control for correlations in the Markov chain sequence on expectations of bounded functions of the Markov chain. A final subsection illustrates how this can be used to bound the norm of the gradient bias. 

\paragraph{Information theory background.}
The total-variation distance between two probability measures is a special case of the more general $f$-divergence defined as
\[
d_f(P||Q) = \int f\left(\frac{dP}{dQ}\right) dQ,
\]
where $f$ is a convex function such that $f(1)=0$. By choosing $f(x) = \abs{x-1}/2$, one recovers the total-variation distance. A choice of $f(x)=x\log(x)$ yields the Kullback-Leibler divergence. This yields a generalization of the mutual information between two random variables $X$ and $Y$. The $f$-information between $X$ and $Y$ is the $f$-divergence between their joint distribution and the product of their marginals:
\begin{eqnarray*}
I_f(X,Y) = d_f(\mathbb{P}(X=\cdot, Y=\cdot)\,, \, \mathbb{P}(X=\cdot) \otimes \mathbb{P}(Y=\cdot)).
\end{eqnarray*} 
This measure satisfies several nice properties. By definition it is symmetric, so 
$I_{f}(X, Y)=I_{f}(Y,X)$. It can be expressed in terms of the expected divergence between conditional distributions:
\begin{equation}\label{eq: f information to conditional divergence}
I_f(X,Y) = \sum_x \mathbb{P}(X=x) d_f(\mathbb{P}(Y=\cdot|X=x), \mathbb{P}(Y=\cdot)).
\end{equation}
Finally, it satisfies the following data-processing inequality. If $X\rightarrow Y \rightarrow Z$ forms a Markov chain, then 
\[
I_{f}(X, Z) \leq I_{f}(X, Y).
\]
Here, we use the notation $X\rightarrow Y \rightarrow Z$, which is standard in information theory and the study of graphical models, to indicate that the random variables $Z$ and $X$ are independent conditioned on $Y$. Note that by symmetry we also have 
$I_{f}(X, Z) \leq I_{f}(Y, Z)$. To use these results in conjunction with Assumption \ref{as:6}, we can specialize to total-variation distance $(d_{\text{TV}})$ and total-variation mutual information $(I_{\text{TV}})$ using $f(x) = \abs{x-1}/2$. The total-variation is especially useful for our purposes because of the following variational representation. 
\begin{equation}\label{eq: variational form of tv}
d_{\rm TV}(P, Q) = \sup_{v: \|v\|_{\infty}\leq \frac{1}{2}} \left| \intop v dP - \intop v dQ\right|. 
\end{equation}
In particular, if $P$ and $Q$ are close in total-variation distance, then the expected value of any bounded function under $P$ will be close to that under $Q$.

\paragraph{Information theoretic control of coupling.}
With this background in place, we are ready to establish a general lemma, which is central to our analysis. We use $\|f\|_{\infty} = \sup_{x\in \mathcal{X}} |f(x)|$ to denote the supremum norm of a function $f: \mathcal{X} \to \mathbb{R}$. 
\begin{restatable}[Control of couping]{lem}{}
\label{lemma:mc_exp_bd}
Consider two random variables $X$ and $Y$ such that 
\[ 
X \rightarrow s_t \rightarrow s_{t+\tau} \rightarrow Y
\]
for some fixed $t\in \{0,1,2,\ldots\}$ and $\tau>0$. Assume the Markov chain mixes uniformly, as stated in Assumption \ref{as:6}. Let $X'$ and $Y'$ denote independent copies drawn from the marginal distributions of $X$ and $Y$, so $\Prob(X'=\cdot, Y'=\cdot) = \Prob(X = \cdot) \otimes \Prob(Y = \cdot)$. Then, for any bounded function $v$, 
\[ 
\left|\E \left[v(X, Y) \right] - \E \left[v(X', Y') \right]\right| \leq 2\|v \|_{\infty} (m\rho^{\tau}).
\]
 \end{restatable}
\begin{proof}
Let $P = \Prob(X \in \cdot, Y\in \cdot)$ denote the joint distribution of $X$ and $Y$ and $Q=\Prob(X\in \cdot) \otimes \Prob(Y\in \cdot)$ denote the product of the marginal distributions. Let $h = \frac{v}{2\|v\|_{\infty}}$, which is the function $v$ rescaled to take values in $[-1/2, 1/2]$. Then, by Equation \eqref{eq: variational form of tv}
\[ 
\E[h(X,Y)]-\E[h(X',Y')] = \intop h dP - \intop h dQ \leq d_{\rm TV}(P,Q) = I_{\rm TV}(X,Y), 
\]
where the last equality uses the definition of the total-variation mutual information, $I_{TV}$. Then,
\begin{eqnarray*}
I_{\rm TV}(X,Y) \leq I_{\rm TV}(s_t, s_{t+\tau}) &=& \sum_{s\in \Sc}\Prob(s_t=s) d_{\rm TV}\left(\Prob(s_{t+\tau} = \cdot \mid s_{t} = s))\, , \, \Prob(s_{t+\tau}=\cdot)\right)\\
&\leq& \sup_{s \in \mathcal{S}} \, d_{\rm TV}\left(\Prob(s_{t+\tau} = \cdot \mid s_{t} = s)\, , \, \pi \right)\\
&\leq & m \rho^{\tau},
\end{eqnarray*}
where the three steps follow, respectively, from the data-processing inequality, the property in Equation \eqref{eq: f information to conditional divergence}, the stationarity of the Markov chain, and the the uniform mixing condition in Assumption \ref{as:6}. Combining these steps gives 
\[
\left|\E \left[v(X, Y) \right] - \E \left[v(X', Y') \right]\right| \leq 2 \|v\|_{\infty}\, I_{\rm TV}(X,Y) \leq 2 \|v\|_{\infty} \,m \rho^{\tau}. 
\]
\end{proof}

%\subsubsection{Application to Bounding the Gradient Bias}
\subsubsection{Bounding the gradient bias.}
We are now ready to bound the expected gradient error $\E[\zeta_{t}(\theta_t)]$. First, we establish some basic regularity properties of the function $\zeta_{t}(\cdot)$. 
\begin{restatable}[Gradient error is bounded and Lipschitz]{lem}{lemmagradbias}
\label{lemma:mc_noise_bd}
With probability 1, 
\[
|\zeta_t(\theta)| \leq 2G^2 \qquad \text{for all } \, \theta \in \Theta_{R}
\]
and 
\[ 
\abs{\zeta_t (\theta) - \zeta_t (\theta')} \leq 6G \norm{(\theta - \theta^')}_2 \qquad \text{for all } \, \theta, \theta' \in \Theta_{R}. 
\]
\end{restatable}
\begin{proof}
The result follows from a straightforward application of the bounds $\|g_{t}(\theta)\|_2 \leq G$ and $\| \theta\|_2 \leq R \leq G/2$, which hold for each $\theta \in \Theta_R$. A full derivation is given in Appendix \ref{appendix_a:mc_noise_bd}.
\end{proof}
We now use Lemmas \ref{lemma:mc_exp_bd} and \ref{lemma:mc_noise_bd} to establish a bound on the expected gradient error. 
 \begin{restatable}[Bound on gradient bias]{lem}{MCGradientBound}
 \label{lemma:mc_grad_bias_bd}
Consider a non-increasing step-size sequence, $\alpha_0 \geq \alpha_1 \ldots \geq \alpha_T$. Fix any $t<T$, and set $t^*\equiv \max\{0, t-\tau^{\mix}(\alpha_T)\}$. Then, 
\[
\mathbb{E}\left[\zeta_t(\theta_t)\right] \leq G^2 \left(4 + 6\tau^{\mix}(\alpha_T)\right) \alpha_{t^*}.
\]
The following bound also holds:
\[
\mathbb{E}\left[\zeta_t(\theta_t)\right] \leq 6G^2\sum_{i=0}^{t-1} \alpha_i.
\]
\end{restatable}
\begin{proof} We break the proof down into three steps.\vspace{2mm}  \\ 
\underline{\it Step 1: Relate $\zeta_{t}(\theta_t)$ and $\zeta_{t}(\theta_{t-\tau})$.} \vspace{1mm}\\ 
Note that for any $i\in \mathbb{N}_0$,  
\[
\norm{\theta_{i+1} - \theta_i}_2 = \norm{\Pi_{2,R}\left( \theta_i + \alpha_i g_i(\theta_i) \right) - \Pi_{2,R}(\theta_i)}_2 \leq \norm{\theta_i + \alpha_i g_i(\theta_i) - \theta_i}_2 = \alpha_i \norm{g_i(\theta_i)}_2 \leq \alpha_i G.
\]
Therefore, 
\[
\norm{\theta_{t}-\theta_{t-\tau}}_2 \leq \sum_{i=t-\tau}^{t-1} \norm{\theta_{i+1} - \theta_i}_2 \leq G \sum_{i=t-\tau}^{t-1} \alpha_i. 
\]
Applying Lemma \ref{lemma:mc_noise_bd}, we conclude  
\begin{equation}\label{eq: key decomposition for bounding gradient bias}
\zeta_{t}(\theta_t)  \leq \zeta_{t}(\theta_{t-\tau})+ 6G^2\sum_{i=t-\tau}^{t-1} \alpha_i \qquad \text{for all } \tau \in \{0,\cdots, t\}.
\end{equation}
\underline{\it Step 2: Bound $\E[\zeta_{t}(\theta_{t-\tau})]$ using Lemma \ref{lemma:mc_exp_bd}.}\vspace{1mm}\\
 Recall that the gradient $g_{t}(\theta)$ depends implicitly on the observed tuple $O_{t}=(s_t, \Rc(s_t, s_{t+1}), s_{t+1})$. Let us overload notation to make this statistical dependency more transparent. Put
\[
g(\theta, O_t) := g_{t}(\theta)=\Big( r_t + \gamma \phi(s_{t+1})^\top \theta - \phi(s_t)^\top \theta \Big) \phi(s_t) \qquad  \theta \in \Theta_R
\]
and
\[ 
\zeta(\theta, O_t) := \zeta_{t}(\theta) = \left(g(\theta, O_t) - \bar{g}(\theta)\right)^\top (\theta-\theta^*)   \qquad  \theta \in \Theta_R.
\]
We have defined $\bar{g}: \Theta_R \to \mathbb{R}^d$ as $\bar{g}(\theta) = \E[g(\theta, O_t)]$ for all  $\theta \in \Theta_R$, where this expectation integrates over the marginal distribution of $O_t$. Then, by definition, for any fixed (non-random) $\theta \in \Theta_{R}$, 
\[
\E[\zeta(\theta, O_t)]=\left(\E[g(\theta, O_t)] - \bar{g}(\theta)\right)^\top (\theta-\theta^*)=0.
\]
Since $\theta_0\in \Theta_{R}$ is non-random, it follows immediately that 
\begin{equation}\label{eq: zero bias at initial point}
\E[\zeta(\theta_0, O_t)]=0. 
\end{equation}
We use Lemma \ref{lemma:mc_noise_bd} to bound $\E[\zeta_{t}(\theta_{t-\tau}, O_t)]$.  First, consider random variables $\theta'_{t-\tau}$ and $O'_{t}$ drawn independently from the marginal distributions of $\theta_{t-\tau}$ and $O_t$, so $\Prob(\theta'_{t-\tau}=\cdot, O'_t=\cdot)=\Prob(\theta_{t-\tau}=\cdot)\otimes \Prob(O_t=\cdot)$. Then $\E[\zeta(\theta'_{t-\tau}, O'_t) ]=\E[ \E[\zeta(\theta'_{t-\tau}, O'_t) \mid \theta'_{t-\tau}] ] =0$. Since  $|\zeta(\theta, O_t)|\leq 2G^2$ for all $\theta \in \Theta_{R}$ by Lemma \ref{lemma:mc_grad_bias_bd} and $\theta_{t-\tau} \rightarrow s_{t-\tau} \rightarrow s_t \rightarrow O_t$ forms a Markov chain, applying Lemma \ref{lemma:mc_noise_bd} gives 
\begin{equation}
\E[\zeta(\theta_{t-\tau}, O_t)] \leq 2 (2G^2)(m\rho^\tau) = 4G^2 m\rho^{\tau}.
\end{equation}
\underline{\it Step 3: Combine terms.}\vspace{1mm}\\
The second claim follows immediately from Equation \eqref{eq: key decomposition for bounding gradient bias} together with Equation \eqref{eq: zero bias at initial point}. We focus on establishing the first claim. Taking the expectation of Equation \eqref{eq: key decomposition for bounding gradient bias} implies
\[
\E[\zeta_{t}(\theta_t)]  \leq \E[\zeta_{t}(\theta_{t-\tau})]+ 6G^2\tau \alpha_{t-\tau} \qquad \forall \tau \in \{0,\cdots, t\}.
\]
For $t\leq \tau^{\mix}(\alpha_T)$, choosing $\tau=t$ gives 
\[
\E[\zeta_{t}(\theta_t)]  \leq \underbrace{\E[\zeta_{t}(\theta_{0})]}_{=0}+ 6G^2t \alpha_{0} \leq 6G^2 \tau^{\mix}(\alpha_T) \alpha_0.
\]
For $t> \tau^{\mix}(\alpha_T)$, choosing $\tau= \tau_0\equiv \tau^{\mix}(\alpha_T)$ gives 
\[ 
\E[\zeta_{t}(\theta_t)]  \leq 4G^2 m \rho^{\tau_0} + 6G^2\tau_0 \alpha_{t-\tau_0} \leq 4G^2 \alpha_{T} + 6G^2\tau_0 \alpha_{t-\tau} \leq G^2\left( 4+ 6\tau_0  \right) \alpha_{t-\tau_0}. 
\]
where the second inequality used that $m \rho^{\tau_0} \leq \alpha_T$ by the definition of the mixing time $\tau_0\equiv \tau^{\mix}(\alpha_T)$ and the second inequality uses that step-sizes are non-increasing. 
\end{proof}

\subsubsection{Completing the proof of Theorem \ref{thrm:mc_ana}} Combining Lemmas \ref{lemma:mc_grad_decomp} and \ref{lemma:mc_noise_bd} gives the error decomposition in Equation \ref{eq:mc_clean_error_decomp} for the case of a constant step-size. As noted at the beginning of this subsection, from this decomposition, parts (a) and (b) of Theorem \ref{thrm:mc_ana} can be established by essentially copying the proof of Theorem \ref{thrm_iid}. For completeness, this is included in Appendix \ref{appendix_a}. For part (c), we closely follow analysis of SGD with decaying step-sizes presented in \cite{lacoste2012simpler}. However, some headache is introduced because Lemma \ref{lemma:mc_grad_bias_bd} includes terms of the form $\alpha_{t-\tau^{\mix}(\alpha_T)}$ instead of the typical $\alpha_{t}$ terms present in analyses of SGD. A complete proof of part (c) is given in Appendix \ref{appendix_a} as well.

\section{Extension to TD with eligibility traces} \label{sec:td_lambda}
%--------------------------------------------------------------------------
\label{sec:td_lambda}
This section extends our analysis to provide finite time guarantees for temporal difference learning \emph{with eligibility traces}. We study a class of algorithms, denoted by TD($\lambda$) and parameterized by $\lambda \in [0,1]$, that contains as a special case the TD(0) algorithm studied in previous sections\footnote{TD(0) corresponds to $\lambda=0$.}. For $\lambda>0$, the algorithm maintains an eligibility trace vector, which is a geometric weighted average of the negative gradients at all previously visited states, and makes parameter updates in the direction of the eligibility vector rather than the negative gradient. Eligibility traces sometimes provide substantial performance improvements in practice \citep{sutton1998reinforcement}. Unfortunately, they also introduce subtle dependency issues that complicate theoretical analysis; to our knowledge, this section provides the \emph{first} non-asymptotic analysis TD($\lambda$). 

Our analysis focuses on the Markov chain observation model studied in the previous section and we mirror the technical assumptions used there. In particular, we assume that the Markov chain is stationary and mixes at a uniform geometric rate (Assumption \ref{as:6}). As before, for tractability, we study a projected variant of TD($\lambda$).

\subsection{Projected TD($\lambda$) algorithm}
TD($\lambda$) makes a simple, but highly consequential, modification to TD(0). Pseudo-code for the algorithm is presented below in Algorithm  \ref{algo:project_increment_TD_lambda}. As with TD(0), at each time-step $t$ it observes a tuple $(s_t, r_t=\Rc(s_t,s_{t+1}), s_{t+1})$ and computes the TD error $\delta_t(\theta_t) = r_t + \gamma V_{\theta_t}(s_{t+1}) - V_{\theta_t}(s_t)$. However, while TD(0) makes an update $\theta_{t+1}=\theta_t + \alpha_t \delta_{t}(\theta_t) \phi(s_t)$ in the direction of the feature vector at the current state, TD($\lambda$) makes the update $\theta_{t+1}= \theta_t + \alpha_t \delta_t(\theta_t) z_{0:t}$. The vector $z_{0:t} = \sum_{k=0}^t (\gamma \lambda)^k \phi(s_{t-k})$ is called the eligibility trace which is updated incrementally as shown below in Algorithm \ref{algo:project_increment_TD_lambda}. As the name suggests, the components of $z_{0:t}$ roughly capture the extent to which each feature is eligible for receiving credit or blame for an observed TD error \citep{sutton1998reinforcement,seijen2014true}.

 \begin{algorithm2e} 
     \SetAlgoLined
     \SetKwInOut{Input}{Input}
     \SetKwInOut{Output}{Output}
     \Input{radius $R$, initial guess $\{\theta_0\,:\,\norm{\theta_0}_2 \leq R \}$, and step-size sequence $\{\alpha_{t}\}_{t\in \mathbb{N}}$}
     Initialize: $\bar{\theta}_{0} \gets \theta_0$, $z_{-1}=0$, $\lambda \in [0,1]$. \\
     \For{$t=0,1,\ldots$}{
        	\nosemic Observe tuple: $O_t = (s_t,r_t,s_{t+1})$\;
%             Define target: $y_t = r_t + \gamma V_{\theta_t}(s'_t)$\tcc*[r]{sample Bellman operator}
            Get TD error: $\delta_t(\theta_t) = r_t + \gamma V_{\theta_t}(s_{t+1}) - V_{\theta_t}(s_t)$\tcc*[r]{sample Bellman error}
            Update eligibility trace: $z_{0:t} = (\gamma \lambda) z_{0:t-1} + \phi(s_t)$\tcc*[r]{Geometric weighting}
            Compute update direction: $x_t(\theta_t, z_{0:t}) = \delta_t(\theta_t) z_{0:t}$\;
            Take a projected update step: $\theta_{t+1} = \Pi_{2,R} (\theta_t + \alpha_t x_t(\theta_t, z_{0:t}))$
        \tcc*[r]{$\alpha_t$:step-size}
        Update averaged iterate: $\bar{\theta}_{t+1} \gets \left(\frac{t}{t+1}\right)\bar{\theta}_{t}+ \left(\frac{1}{t+1}\right)\theta_{t+1} $ 
        \tcc*[r]{$\bar{\theta}_{t+1} = \frac{1}{t+1} \sum_{\ell=1}^{t+1} \theta_\ell$}
        }
           \caption{Projected TD($\lambda$) with linear function approximation} \label{algo:project_increment_TD_lambda}
    \end{algorithm2e}
    
Some new notation in Algorithm \ref{algo:project_increment_TD_lambda} should be highlighted. We use $x_t(\theta, z_{0:t})=\delta_t(\theta)z_{0:t}$ to denote the update to the parameter vector $\theta$ at time $t$. This plays a role analogous to the negative gradient $g_{t}(\theta)$ in TD(0). 

\subsection{Limiting behavior of TD($\lambda$)}
\label{subsec:asymp_td_lambda}
We now review results on the asymptotic convergence of TD($\lambda$) due to 
\cite{tsitsiklis1997analysis}. This provides the foundation of our finite time analysis and also offers insight into how the algorithm differs from TD(0). 

Before giving any results, let us note that just as the true value function $V_{\mu}(\cdot)$ is the unique solution to Bellman's fixed point equation $V_{\mu}=T_{\mu} V_{\mu}$, it is also the unique solution to a $k$-step Bellman equation
$V_{\mu} = T_{\mu}^{(k)} V_{\mu}$. This can be written equivalently as
\[ 
V_{\mu}(s)= \mathbb{E} \left[ \sum_{t=0}^k \gamma^t \Rc(s_t) + \gamma^{k+1} V(s_{k+1}) \mid s_0 = s \right] \qquad \forall s\in S,
\]
where the expectation is over states sampled when policy $\mu$ is applied to the MDP. The asymptotic properties of TD$(\lambda)$ are closely tied to a geometrically weighted version of the $k$-step Bellman equations described above. Define the averaged Bellman operator 
\begin{eqnarray}
\label{eq:lambda_bellman_op}
(T_{\mu}^{(\lambda)}V)(s) = (1-\lambda) \sum_{k=0}^\infty \lambda^k \mathbb{E} \left[ \sum_{t=0}^k \gamma^t \Rc(s_t) + \gamma^{k+1} V(s_{k+1}) \mid s_0 = s \right].
\end{eqnarray}
One interesting interpretation of this equation is as a $k$-step Bellman equation, but where the horizon $k$ itself is a random geometrically distributed random variable. 

\cite{tsitsiklis1997analysis} showed that under appropriate technical conditions,  the approximate value function $V_{\theta_t} = \Phi\theta_t$ estimated by TD($\lambda)$ converges almost surely to the unique solution, $\theta^*$ of the projected fixed point equation 
\[
\Phi \theta = \Pi_{D} T_{\mu}^{(\lambda)} \Phi \theta.
\]
TD($\lambda)$ is then interpreted as a stochastic approximation scheme for solving this fixed point equation. The existence and uniqueness of such a fixed point $\Phi \theta^*$ is implied by the following lemma, which shows that $\Pi_D T^{\lambda}(\cdot)$ is a contraction operator with respect to the steady-state weighted norm $\|\cdot \|_D$. 
\begin{restatable}[]{lem}{lemmacontractdlambda}
\label{lemmm:contract_td_lambda}
[\textbf{\cite{tsitsiklis1997analysis}}]
$\Pi_{D} T_{\mu}^{(\lambda)} (\cdot)$ is a contraction with respect to $\| \cdot \|_{D}$ with modulus
\begin{eqnarray*}
\kappa = \frac{\gamma (1-\lambda)}{1 - \gamma \lambda} \leq \gamma < 1.
\end{eqnarray*}
% $\kappa$. That is,
% \begin{eqnarray*}
% \norm{\Pi_{D} T_{\mu}^{(\lambda)} V_{\theta} - \Pi_{D} T_{\mu}^{(\lambda)} V_{\theta'}}_{D} \leq \kappa \norm{V_{\theta} - V_{\theta'}}_{D} \quad \forall \,\, (\theta, \theta'). 
% \end{eqnarray*}
% where $\kappa$ is the contraction factor of $\Pi_{D} T_{\mu}^{(\lambda)} (\cdot)$ such that
\end{restatable}
As with TD(0), the limiting value function under TD($\lambda$) comes with some competitive guarantees. A short argument using Lemma \ref{lemmm:contract_td_lambda} shows
\begin{equation}\label{eq: approximation_error_bound_td_lambda}
\norm{V_{\theta^*} - V_{\mu} }_{D}   \leq \frac{1}{\sqrt{1-\kappa^2}} \norm{ \Pi_{D} V_{\mu} - V_{\mu}  }_{D}. 
\end{equation}
See for example Chapter 6 of \cite{bertsekas2012dynamic} for a proof. It is important to note the distinction between the convergence guarantee results for TD($\lambda$) and TD(0) in terms of the contraction factors. The contraction factor $\kappa$ is always less than $\gamma$, the contraction factor under TD(0). In addition, as $\lambda\to 1$, $\kappa\to 0$ implying that the limit point of TD($\lambda$) for large enough $\lambda$ will be arbitrarily close to $\Pi_{D} V_{\mu}$, which minimizes the mean-square error in value predictions among all value functions representable by the features. This calculation suggests a choice of $\lambda=1$ will offer the best performance. However, the \emph{rate of convergence} also depends on $\lambda$, and may degrade as $\lambda$ grows. Disentangling such issues requires also a careful study of the statistical efficiency of TD($\lambda$), which we undertake in the following subsection. 
%As $\kappa$ is less that $\gamma$, a comparison between Equation \eqref{eq: approximation error bound} and Equation \eqref{eq: approximation_error_bound_td_lambda} shows that TD($\lambda$) converges to a fixed point with a smaller mean square error. But what about the convergence rate? Does TD($\lambda$) converge faster that TD(0)? Given that $\kappa$ is less that $\gamma$ and more so as $\lambda$ increases, one might expect TD($\lambda$) to also converge faster as. However, as we explain in the following Section, our main result in Theorem \ref{thrm:lambda_mc_ana} shows otherwise. The contraction factor $\kappa$ is always smaller than $\gamma$, and, for fixed $\gamma$, $\kappa\to 0$ as $\lambda\to 1$. 

\subsection{Finite time bounds for Projected TD($\lambda$)} 
Following Section \ref{sec:markov_chain_analyis}, we establish three finite time bounds on the performance of the Projected TD($\lambda$) algorithm. The first bound in part (a) does not depend on any special regularity of the problem instance but gives a comparatively slow convergence rate of $\tilde{\mathcal{O}}(1/\sqrt{T})$. It applies with the robust (problem independent) and aggressive step-size of $1/\sqrt{T}$. Part (b) illustrates the exponential rate of convergence to within some radius of the TD($\lambda$) fixed-point for sufficiently small step-sizes. Part (c) attains an improved dependence on $T$ of $\tilde{\mathcal{O}}(1/T)$, but the step-size sequence requires knowledge of the minimum eigenvalue $\omega$ of $\Sigma$. 

Compared to the results for TD(0), our bounds depend on a slightly different definition of the mixing time that takes into account the geometric weighting in the eligibility trace term. Define
\begin{eqnarray}
\label{eq:lambda_tau_mix}
\tau^{\mix}_{\lambda}(\epsilon) = \max \{ \tau^{\text{MC}}(\epsilon), \tau^{\text{Algo}}(\epsilon) \},
\end{eqnarray}
where we denote $\tau^{\text{MC}}(\epsilon) = \min \{ t \in \mathbb{N}_0 \,\, | \,\, m \rho^t \leq \epsilon \}$ and $\tau^{\text{Algo}}(\epsilon) = \min \{ t \in \mathbb{N}_0 \,\, | \,\, (\gamma \lambda)^t \leq \epsilon \}$. As we show next, this definition of mixing time enables compact bounds for convergence rates of TD($\lambda$).

\begin{restatable}[]{thm}{LambdaMCBound}
\label{thrm:lambda_mc_ana}
Suppose the Projected TD($\lambda$) algorithm is applied with parameter $R \geq \| \theta^*\|_2$ under the Markov chain observation model with Assumption \ref{as:6}. Set $B=\frac{(r_{\rm max}+2R)}{(1-\gamma \lambda)}$. Then the following claims hold. 
\begin{enumerate}[label=(\alph*)]
\item With a constant step-size $\alpha_t = \alpha_0 = 1/\sqrt{T}$,
\begin{eqnarray*}
\mathbb{E}\left[\norm{V_{\theta^*} - V_{\bar{\theta}_T}}^2_{D}\right] \leq \frac{\norm{\theta^* - \theta_0}_2^2 + B^2 \left(13+28\tau^{\mix}_{\lambda}(1/\sqrt{T})\right)}{2\sqrt{T}(1-\kappa)}.
\end{eqnarray*}

\item With a constant step-size $\alpha_t = \alpha_0 < 1/(2 \omega (1-\kappa))$ and $T > 2 \tau^{\mix}_{\lambda}(\alpha_0)$ ,
\begin{eqnarray*}
\mathbb{E} \left[\norm{\theta^*-\theta_{T}}_2^2\right] \leq \left(e^{-2\alpha_0(1-\kappa)\omega T}\right) \norm{\theta^*-\theta_0}_2^2 + \alpha_0 \Bigg( \frac{B^2 \left(13+24\tau^{\mix}_{\lambda}(\alpha_0) \right)}{2(1-\kappa)\omega} \Bigg).
\end{eqnarray*}

\item With a decaying step-size $\alpha_t = 1/(\omega (t+1) (1-\kappa))$,
\begin{eqnarray*}
\mathbb{E}\left[\norm{V_{\theta^*} - V_{\bar{\theta}_T}}^2_{D}\right] \leq \frac{ B^2 \left(13+52\tau^{\mix}_{\lambda}(\alpha_T)\right)}{T(1-\kappa)^2\omega} \left(1+\log T\right).
\end{eqnarray*}
\end{enumerate}
\end{restatable}
\begin{rem} As was the case for TD(0), the proof of part (c) also implies a $\tilde{\mathcal{O}}(1/T)$ convergence rate for the iterate $\theta_{T}$ itself. Again, a different weighting of the iterates as shown in \cite{lacoste2012simpler} might enable us to eliminate the $\log T$ term in the numerator of part (c) to give a $\mathcal{O}(1/T)$ convergence rate. For brevity, we do not pursue this direction. 
\end{rem}
We now compare the bounds for TD($\lambda$) with that of TD(0) ignoring the constant terms. First, let us look at the results for the constant step-size $\alpha_t = 1/\sqrt{T}$ in part (a) of Theorems \ref{thrm:mc_ana} and \ref{thrm:lambda_mc_ana}. Approximately, for the TD($\lambda$) case, we have the term $\frac{B^2}{\sqrt{T}(1-\kappa)}$ vis-a-vis the term $\frac{G^2}{\sqrt{T}(1-\gamma)}$ for the TD(0) case. A simple argument below clarifies the relationship between these two.
\begin{eqnarray*}
\frac{B^2}{\sqrt{T}(1-\kappa)} &=& \frac{(r_{\max} + 2R)^2}{\sqrt{T}(1-\kappa) (1-\gamma \lambda)^2} = \frac{G^2}{\sqrt{T}(1-\kappa) (1-\gamma \lambda)^2} \\
&\geq& \frac{G^2}{\sqrt{T}(1-\kappa) (1-\gamma \lambda)} = \frac{G^2}{\sqrt{T}(1-\gamma)}.
\end{eqnarray*}
As we will see later, $B$  is an upper bound to the norm of $x_t(\theta_t,z_{0:t})$, the update direction for TD($\lambda$). Correspondingly, from Section \ref{sec:markov_chain_analyis}, we know that $G$ is the upper bound on gradient norm, $g_t(\theta_t)$ for TD(0). Intuitively, for TD($\lambda$), the bound $B$ is larger (due to the presence of the eligibility trace term) and more so as $\lambda \rightarrow 1$. This dominates any benefit (in terms of statistical efficiency) from a smaller contraction factor, $\kappa$. 

However, for decaying step-sizes of $\alpha_t = 1/(\omega (t+1) (1-\kappa))$, the bounds are qualitatively the same. This follows as the terms that dominate part (c) of Theorems \ref{thrm:mc_ana} and \ref{thrm:lambda_mc_ana} are equal:
\begin{eqnarray*}
\frac{B^2}{T(1-\kappa)^2} = \frac{(r_{\max} + 2R)^2}{T(1-\kappa)^2 (1-\gamma \lambda)^2}
= \frac{G^2}{T(1-\kappa)^2 (1-\gamma \lambda)^2} = \frac{G^2}{T(1-\gamma)^2}.
\end{eqnarray*}
In conclusion, for constant step-sizes--which is often how TD algorithms as used in practice--our bounds establish a faster convergence rate for TD($0$) than for TD($\lambda$). Or equivalently, according to our bounds, more data is required to guarantee TD($\lambda)$ is close to its limit point. In this context, however, the trade-off we remarked on in Section \ref{subsec:asymp_td_lambda} is noteworthy as the fixed point for TD($\lambda$) comes with a better error guarantee.

%------------------------------------------------------------------------------------
\section{Extension: Q-learning for high dimensional Optimal Stopping} \label{sec:opt_stopping}
%------------------------------------------------------------------------------------
So far, this paper has dealt with the problem of approximating the value function of a fixed policy in a computationally and statistically efficient manner. The Q-learning algorithm is one natural extension of temporal-difference learning to control problems, where the goal is to learn an effective policy from data. Although it is widely applied in reinforcement learning, in general Q-learning is unstable and its iterates may oscillate forever. An important exception to this was discovered by \cite{tsitsiklis1999optimal}, who showed that $Q$-learning converges asymptotically for optimal stopping problems. In this section, we show how the techniques developed in Sections \ref{sec:iid_sampling} and \ref{sec:markov_chain_analyis} can be applied \emph{in an identical manner} to give finite time bounds for Q-learning with linear function approximation applied to optimal-stopping problems with high dimensional state spaces.  To avoid repetition, we only state key properties satisfied by Q-learning in this setting which establish exactly the same convergence bounds as shown in Theorems \ref{thrm_iid} and \ref{thrm:mc_ana}. 

\subsection{Problem formulation} \label{subsec:opt_stopping_prob_form}
%-------------------------------------------------------------------------

The optimal stopping problem is that of determining the time to terminate a process to maximize cumulative expected rewards accrued. Problems of this nature arise naturally in many settings, most notably in the pricing of financial derivatives \citep{andersen2004primal, haugh2004pricing, desai2012pathwise}. We first give a brief formulation for a class of optimal stopping problems. A more detailed exposition can be found in \cite{tsitsiklis1999optimal}, or Chapter 5 of the thesis work of \cite{van1998learning}. 

Consider a discrete-time Markov chain $\{s_t\}_{t \geq 0}$ with finite state space $\mathcal{S}$ and unique stationary distribution distribution $\pi$. At each time $t$, the decision-maker observes the state $s_t$ and decides whether to stop or continue. Let $\gamma \in [0,1)$ denote the discount factor and let $u(\cdot)$ and $U(\cdot)$ denote the reward functions associated with \emph{continuation} and \emph{termination} decisions respectively. Let the stopping time $\tau$ denote the (random) time at which the decision-maker stops. The expected total discounted reward from initial state $s$ associated with the stopping time $\tau$ is
\begin{equation}\label{eq: objective in optimal stopping}
\mathbb{E} \left[ \sum_{t=0}^{\tau-1} \gamma^t u(s_t) + \gamma^\tau U(s_\tau) \, \bigg\vert \, s_0=s \right],
\end{equation}
where $U(s_\tau)$ is defined to be zero for $\tau = \infty$. We seek an optimal stopping policy, which determines when to stop as a function of the observed states so as to maximize \eqref{eq: objective in optimal stopping}.

% The optimal stopping time $\tau^*$ satisfies
% \begin{eqnarray*}
% \mathbb{E} \left[ \sum_{t=0}^{\tau^*-1} \gamma^t u(s_t) + \gamma^{\tau^*} U(s_{\tau^*}) \right] = \sup_{\tau \in \mathcal{T}} \mathbb{E} \left[ \sum_{t=0}^{\tau-1} \gamma^t u(s_t) + \gamma^\tau U(s_\tau) \right].
% \end{eqnarray*}
% For each stopping time $\tau$, let us define the value function, $V^\tau(\cdot)$ as the total expected discounted reward for $\tau$. 
% \begin{eqnarray*}
% V^\tau(s) = \mathbb{E} \left[ \sum_{t=0}^{\tau-1} \gamma^t u(s_t) + \gamma^\tau U(s_\tau) | s_0 = s \right]. 
% \end{eqnarray*}
For any Markov decision process, the optimal state-action value function $Q^*: \mathcal{S} \times \mathcal{A} \rightarrow \mathbb{R}$ specifies the expected value to go from choosing an action $a\in \Ac$ in a state $s\in \Sc$ and following the optimal policy in subsequent states. In optimal stopping problems, there are only two possible actions at every time step: whether to \emph{terminate} or to \emph{continue}. The value of stopping in state $s$ is just $U(s)$, which allows us to simplify notation by only representing the continuation value. 

For the remainder of this section, we let $Q^* : \Sc\to \mathbb{R}$ denote the continuation-value function. It can be shown that $Q^*$ is the unique solution to the Bellman equation $Q^*= FQ^*$, where the Bellman operator is given by 
\[ 
F Q(s) = u(s) + \gamma \sum_{s' \in \mathcal{S}} P(s'|s) \max{ \{ U(s'), Q(s')\}}.
\]
Given the optimal continuation values $Q^*(\cdot)$, the optimal stopping time is simply
\begin{eqnarray}
\label{eq:opt_stop_def}
\tau^* = \min{ \{t \, | \, U(s_t) \geq Q^*(s_t)\} }.
\end{eqnarray}

%-------------------------------------------------------------------------------------------------------
\subsection{Q-Learning for high dimensional Optimal Stopping} \label{subsec:opt_stopping_approx_alg}
%-------------------------------------------------------------------------------------------------------
In principle, one could generate the optimal stopping time using Equation \eqref{eq:opt_stop_def} by applying exact dynamic programming algorithms to compute the optimal Q-function. However, such methods are only implementable for small state spaces. To scale to high dimensional state spaces, we consider a feature-based approximation of the optimal continuation value function, $Q^*$. We focus on linear function approximation, where $Q^*(s)$ is approximated as
$$Q^*(s) \approx Q_{\theta}(s) = \phi(s)^\top \theta, $$ 
where $\phi(s) \in \mathbb{R}^d$ is a fixed feature vector for state $s$ and $\theta \in \mathbb{R}^d$ is a parameter vector that is shared across states. As shown in Section \ref{sec:formulation}, for a finite state space, $\Sc = \{s_1, \ldots, s_n\}$, $Q_{\theta}\in \mathbb{R}^n$ can be expressed compactly as $Q_{\theta} = \Phi \theta$, where $ \Phi \in \mathbb{R}^{n \times d}$ and $\theta \in \mathbb{R}^d$ . We also assume that the $d$ features vectors $\{\phi_k\}_{k=1}^d$, forming the columns of $\Phi$ are linearly independent. 

We consider the Q-learning approximation scheme in Algorithm \ref{algo:project_TD_opt_stop}. The algorithm starts with an initial parameter estimate of $\theta_0$ and observes a data tuple $O_t = (s_t, u(s_t), s'_t)$. This is used to compute the target $y_t = u(s_t) + \gamma \max{ \{ U(s'_t), Q_{\theta_t}(s'_t) \} }$, which is a sampled version of the $F(\cdot)$ operator applied to the current $Q$--function. The next iterate, $\theta_{t+1}$, is computed by taking a gradient step with respect to a loss function measuring the distance between $y_t$ and predicted value-to-go. An important feature of this method is that problem data is generated by the exploratory policy that chooses to continue at all time-steps.   

 \begin{algorithm2e} 
     \SetAlgoLined
     \SetKwInOut{Input}{Input}
     \SetKwInOut{Output}{Output}
     \Input{initial guess $\theta_0$, step-size sequence $\{\alpha_{t}\}_{t\in \mathbb{N}}$ and radius $R$. }
     Initialize: $\bar{\theta}_{0} \gets \theta_0$. \\
     \For{$t=0,1,\ldots$}{
        	\nosemic Observe tuple: $O_t = (s_t,u(s_t),s'_t)$\;
            Define target: $y_t = u(s_t) + \gamma \max\hspace{-0.5mm}{\{ U(s'_t), Q_{\theta_t}(s'_t)\}}$\hspace{-1mm}\tcc*[r]{sample Bellman operator}
            Define loss function: $\frac{1}{2} (y_t - Q_{\theta}(s_t))^2$\tcc*[r]{sample Bellman error}
            Compute negative gradient: $g_t(\theta_t) = - \frac{\partial}{\partial \theta} \frac{1}{2}  (y_t - Q_{\theta}(s_t))^2\at[\big]{\theta = \theta_t}$\;
            Take a gradient step: $\theta_{t+1} = \theta_t + \alpha_t g_t(\theta_t)$
        \tcc*[r]{$\alpha_t$:step-size}
        Update averaged iterate: $\bar{\theta}_{t+1} \gets \left(\frac{t}{t+1}\right)\bar{\theta}_{t}+ \left(\frac{1}{t+1}\right)\theta_{t+1} $ 
        \tcc*[r]{$\bar{\theta}_{t+1} = \frac{1}{t+1} \sum_{\ell=1}^{t+1} \theta_\ell$}
        }
           \caption{Q-Learning for Optimal Stopping problems.} \label{algo:project_TD_opt_stop}
    \end{algorithm2e}

%---------------------------------------------------------------
\subsection{Asymptotic guarantees} \label{subsec:asymp_results}
%---------------------------------------------------------------
Similar to the asymptotic results for TD algorithms, \cite{tsitsiklis1999optimal} show that the variant of Q-learning detailed above in Algorithm \ref{algo:project_TD_opt_stop} converges to the unique solution, $\theta^*$, of the projected Bellman equation,
\[
\Phi \theta = \Pi_D F \Phi \theta.
\]
This results crucially relies on the fact that the projected Bellman operator $\Pi_D F(\cdot)$ is a contraction with respect to $\norm{\cdot}_D$ with modulus $\gamma$. The analogous result for our study of TD(0) was stated in Lemma \ref{lemma:BO_prop}. \cite{tsitsiklis1999optimal} also give error bounds for the limit of convergence with respect to $Q^*$, the optimal Q-function. In particular, it can be shown that
\[
\norm{\Phi \theta^* - Q^*}_D \leq \frac{1}{\sqrt{1-\gamma^2}} \norm{\Pi_D Q^* - Q^*}_D
\]
where the left hand side measures the error between the estimated and the optimal Q-function which is upper bounded by the \emph{representational power} of the linear approximation architecture, as given on the right hand side. In particular, if $Q^*$ can be represented as a linear combination of the feature vectors then there is no approximation error and the algorithm converges to the optimal Q-function. Finally, one can ask whether the stopping times suggested by this approximate continuation value function, $\Phi \theta^*$, are effective. Let $\tilde{\mu}$ be the policy that stops at the first time $t$ when 
\[
U(s_t) \geq (\Phi \theta^*)(s_t).
\]
Then, for an initial state $s_0$ drawn from the stationary distribution $\pi$, 
\[
 \E\left[V^*(s_0) \right] - \mathbb{E}\left[ V_{\tilde{\mu}}(s_0) \right] \leq \frac{2}{(1-\gamma)\sqrt{1-\gamma^2}} \norm{\Pi_D Q^* - Q^*}_D,
\]
where $V^*$ and $V_{\tilde{\mu}}$ denote the value functions corresponding, respectively, to the optimal stopping policy the approximate stopping policy $\mu$. Again, this error guarantee depends on the choice of feature representation.

%----------------------------------------------------------------------------
\subsection{Finite time analysis} \label{subsec:opt_stopping_finite_anal}
%----------------------------------------------------------------------------
In this section, we show how our results in Sections \ref{sec:iid_sampling}, \ref{sec:markov_chain_analyis} for TD(0) and its projected counterpart can be extended, without any modification, to give convergence bounds for the Q-function approximation algorithm described above. To this effect, we highlight that key lemmas that enable our analysis in Sections \ref{sec:iid_sampling} and \ref{sec:markov_chain_analyis} also hold in this setting. The contraction property of the $F(\cdot)$ operator will be crucial to our arguments here. Convergence rates for an i.i.d.\,noise model, mirroring those established for TD(0) in Theorem \ref{thrm_iid}, can be shown for Algorithm \ref{algo:project_TD_opt_stop}. Results for the Markov chain sampling model, mirroring those established for TD(0) in Theorem \ref{thrm:mc_ana}, can be shown for a projected variant of Algorithm \ref{algo:project_TD_opt_stop}.

First, we give mathematical expressions for the negative gradient. As a general function of $\theta$ and tuple $O_t = (s_t,u(s_t),s'_t)$, the negative gradient can be written as
\begin{eqnarray}
\label{eq:grad_opt_stopping}
g_t(\theta) = \Big( u(s_t) + \gamma \max{ \{U(s'_t), \phi(s'_{t})^\top \theta\} } - \phi(s_t)^\top\theta \Big) \phi(s_t).
\end{eqnarray}
The negative expected gradient, when the tuple $(s_t, u(s_t), s'_t)$ follows its steady-state behavior, can be written as
\[ 
\bar{g}(\theta) = \sum_{s, s' \in \Sc} \pi(s) \Pc(s'|s) \left(u(s)+\gamma \max{ \{U(s'), \phi(s')^\top \theta\} } - \phi(s)^\top\theta \right)\phi(s).
\]
Additionally, using $ \sum_{s' \in \Sc} \Pc(s'|s) \left(u(s)+\gamma \max{ \{U(s'), \phi(s')^\top \theta\} } \right) = \left(F \Phi \theta \right)(s)$, it is easy to show 
\[ 
\bar{g}(\theta) = \Phi^\top D (F \Phi\theta - \Phi \theta).
\]
Note the close similarity of this expression with its counterparts for TD learning (see Section \ref{sec:algorithm} and Appendix \ref{appendix_b:TD_lambda}); the only difference is that the appropriate Bellman operator(s) for TD learning, $T_{\mu}(\cdot)$, has been replaced with the appropriate Bellman operator $F(\cdot)$ for this optimal stopping problem. 
%--------------------------------------------------------------------------------------%

\subsubsection{Analysis with i.i.d.\,noise}
In this section, we show how to analyze the Q-learning algorithm under an i.i.d.\,observation model, where the random tuples observed by the algorithm are sampled i.i.d.\,from the stationary distribution of the Markov process. All our ideas follow the presentation in Section \ref{sec:iid_sampling}, a careful understanding of which reveals that Lemmas \ref{lemma:expected_gradient} and \ref{lemma:norm_bd_grad} form the backbone of our results. Recall that Lemma \ref{lemma:expected_gradient} establishes how, at any iterate $\theta$, TD updates point in the descent direction of  $\|\theta^* - \theta\|^2_2$. Lemma \ref{lemma:norm_bd_grad} bounds the expected norm of the stochastic gradient, thus giving a control over system noise. 

In Lemmas \ref{lemma:expected_gradient_opt_stopping} and \ref{lemma:norm_bd_grad_opt_stopping}, given below, we show how  exactly the same results also hold for the Q-function approximation algorithm under the i.i.d.\,sampling model. With these two key lemmas, convergence bounds shown in Theorem \ref{thrm_iid} follows by repeating the analysis in Section \ref{sec:iid_sampling}. 
\begin{restatable}[]{lem}{lemmaoptstopp}
[\textbf{\cite{tsitsiklis1999optimal}}]
\label{lemma:expected_gradient_opt_stopping}
Let $V_{\theta^{*}}$ be the unique fixed point of $\Pi_{D} F(\cdot)$, i.e. $V_{\theta^*} = \Pi_{D}F V_{\theta^*}$. Then, for any $\theta \in \mathbb{R}^d$,
\begin{equation*}
(\theta^* - \theta)^\top \bar{g}(\theta) \,\, \geq \,\, (1-\gamma) \norm{V_{\theta^*} - V_{\theta}}^2_{D}.
\end{equation*}
\end{restatable}
\begin{proof}
This property is a consequence of the fact that $\Pi_D F (\cdot)$ is a contraction with respect to $\norm{\cdot}_{D}$ with modulus $\gamma$. It was established by \cite{tsitsiklis1999optimal} in the process of proving their Lemma 8. For completeness, we provide a standalone proof in Appendix \ref{appendix:a1}.
\end{proof}
%--------------------------------------------------------------------------------------%
\begin{restatable}[]{lem}{}
\label{lemma:norm_bd_grad_opt_stopping}
We use notation and proof strategy mirroring the proof of Lemma \ref{lemma:norm_bd_grad}. For any fixed $\theta \in \mathbb{R}^d$, $\mathbb{E}\|g_t(\theta) \|^2_2 \leq  2\sigma^2 + 8\| V_{\theta}-V_{\theta^*}\|^2_D$
where $\sigma^2 = \mathbb{E}\left[ \| g_t(\theta^*) \|^2_2 \right]$. 
\end{restatable}
\begin{proof}
For brevity of notation, set $\phi= \phi(s_t), \, \phi'=\phi(s'_t)$ and $U'=U(s')$. Define $\xi = (\theta^*-\theta)^\top \phi$ and $\xi' = (\theta^*-\theta)^\top \phi'$. By stationarity $\xi$ and $\xi'$ have the same marginal distribution and $\E[\xi^2] = \| V_{\theta^*} - V_{\theta}\|_D^2$. Using the formula for $g_{t}(\theta)$ in Equation \eqref{eq:grad_opt_stopping}, we have        
\begin{eqnarray}
\E\left[ \|g_t(\theta) \|^2_2\right] &\leq& 2\E \left[ \|g_t(\theta^*)\|^2_2 \right] + 2\E \left[ \|g_t(\theta)-g_t(\theta^*) \|^2_2 \right] \nonumber \\ 
&=& 2\sigma^2 + 2\mathbb{E} \left[ \left\| \phi \left( \phi^\top (\theta^*-\theta) - \gamma \left[ \max{ \left(U',{\phi'}^\top \theta^* \right) } - \max{ \left(U', {\phi'}^\top \theta \right)}\right]  \right) \right\|^2_2 \right] \nonumber \\
&\leq& 2\sigma^2 + 2\mathbb{E} \left[ \left\| \phi \left( \abs{\phi^\top (\theta^*-\theta)} + \gamma \abs{\max{ \left(U',{\phi'}^\top \theta^* \right) } - \max{ \left(U', {\phi'}^\top \theta \right)}} \right) \right\|^2_2 \right] \nonumber \\
&\leq& 2\sigma^2 + 2\mathbb{E} \left[ \left\| \phi \left(\abs{\phi^\top (\theta^*-\theta)} + \gamma \abs{{\phi'}^\top (\theta^*-\theta)} \right) \right\|^2_2 \right] \label{eq:opt_stop_max_prop} \\
&\leq& 2\sigma^2 + 2\mathbb{E} [\abs{\xi + \gamma \xi'}^2] \nonumber \\
&\leq& 2\sigma^2 + 4 \left( \E\left[ |\xi|^2 \right] + \gamma^2 \E\left[ |\xi'|^2 \right] \right)  \nonumber \\
&=& 2\sigma^2 + 4(1=\gamma^2) \| V_{\theta} - V_{\theta^*}\|^2_D \leq 2\sigma^2 + 8 \| V_{\theta} - V_{\theta^*}\|^2_D, \nonumber
\end{eqnarray}
where we used the assumption that features are normalized so that $\| \phi \|^2_2 \leq 1$ almost surely. Additionally, in going to Equation \eqref{eq:opt_stop_max_prop}, we used that $\abs{\max{(c_1,c_3)}-\max{(c_2,c_3)}} \leq \abs{c_1-c_2}$ for any scalars $c_1,c_2$ and $c_3$.
\end{proof}
%--------------------------------------------------------------------------------------%

\subsubsection{Analysis under the Markov chain model}
Analogous to Section \ref{sec:markov_chain_analyis}, we analyze a projected variant of Algorithm \ref{algo:project_TD_opt_stop} under the Markov chain sampling model. Let $\Theta_{R}= \{\theta \in \mathbb{R}^d : \| \theta\|_2 \leq R \}$.  Starting with an initial guess of $\theta_0 \in \Theta_R$, the algorithm updates to the next iterate by taking a gradient step followed by projection onto $\Theta_R$, so iterates satisfy the stochastic recursion $\theta_{t+1} = \Pi_{2,R} (\theta_t + \alpha_t g_t(\theta_t))$. We make the similar structural assumptions to those in  Section \ref{sec:markov_chain_analyis}. In particular, assume the feature vectors and the continuation, termination rewards to be uniformly bounded, with $\|\phi(s) \|_{2}  \,\, \leq 1$ and $\max\{\abs{u(s)}, \abs{U(s)}\} \leq r_{\max}$ for all $s\in \mathcal{S}$. We assume $r_{\max} \leq R$, which can always be ensured by rescaling rewards or the projection radius. 

We first show a uniform bound on the gradient norm.
\begin{restatable}[]{lem}{lemmaoptnorm}
\label{lemma:mc_opt_stop_norm_grad} 
Define $G = (r_{\max} + 2R)$. With probability 1, $\norm{g_t(\theta)}_2 \leq G \,\, \text{for all} \,\, \theta \in \Theta_R$. 
% This also implies a uniform norm bound on the expected gradient, $\norm{\bar{g}(\theta_t)}_2 \leq G \,\, \forall \,\, \theta_t$.
\end{restatable} 
\begin{proof}
We start with the mathematical expression for the stochastic gradient, 
\begin{eqnarray*}
g_t(\theta) = \Big( u(s_t) + \gamma \max{ \{U(s'_t), \phi(s'_{t})^\top \theta\} } - \phi(s_t)^\top\theta \Big) \phi(s_t).
\end{eqnarray*}
As $r_{\max} \leq R$, we have: $\max{ \{U(s'_t), \phi(s'_{t})^\top \theta\} } \leq \max{ \{U(s'_t), \norm{\phi(s_{t}')}_2 \norm{\theta}_2 \} } \leq R$. Then, 
\begin{eqnarray*}
\norm{g_t(\theta)}^2_2 &=& \left( u(s_t) + \gamma \max{ \{U(s'_t), \phi(s'_{t})^\top \theta\} } - \phi(s_t)^\top\theta \right)^2 \norm{\phi(s_t)}^2 \\
&\leq& \left( r_{\max} + \gamma R - \phi(s_t)^\top \theta \right)^2 \\
&\leq& \left( r_{\max} + \gamma R + \norm{\phi(s_t)}_2 \norm{\theta}_2 \right)^2 \leq \left( r_{\max} + 2 R \right)^2= G^2.
\end{eqnarray*}
We used here that the basis vectors are normalized, $\norm{\phi(s_t)}_2 \leq 1$ for all $t$.
\end{proof}
If we assume the Markov process $(s_0, s_1, \ldots)$ satisfies Assumption \ref{as:6}, then Lemma \ref{lemma:mc_opt_stop_norm_grad} paves the way to show exactly the same convergence bounds as given in Theorem \ref{thrm:mc_ana}. For this, we refer the readers to Section \ref{sec:markov_chain_analyis} and Appendix \ref{appendix_a}, where we show all the key lemmas and a detailed proof of Theorem \ref{thrm:mc_ana}. One can mirror the same proof, using Lemmas \ref{lemma:expected_gradient_opt_stopping} and \ref{lemma:mc_opt_stop_norm_grad} in place of Lemmas \ref{lemma:expected_gradient} and \ref{lemma:mc_grad_bias_bd}, which apply to TD(0). In particular, note that we can use Lemma \ref{lemma:mc_opt_stop_norm_grad} along with some basic algebraic inequalities to show the gradient bias, $\zeta_t(\theta)$, to be Lipschitz and bounded. This, along with the information-theoretic arguments of Lemma \ref{lemma:mc_exp_bd} enables the exact same upper bound on the gradient bias as shown in Lemma \ref{lemma:mc_grad_bias_bd}. Combining these with standard proof techniques for SGD \citep{lacoste2012simpler,nemirovski2009robust} shows the convergence bounds for Q-learning.

%-----------------------------------------
\section{Conclusions} \label{sec:conc}
%-----------------------------------------
In this paper we provide a simple finite time analysis of a foundational and widely used algorithm known as temporal difference learning. Although asymptotic convergence guarantees for the TD method were previously known, characterizing its data efficiency stands as an important open problem.  %The key to our analysis is lemma \eqref{lemma:expected_gradient}, a \textit{convexity} like property in optimization parlance. 
Our work makes a substantial advance in this direction by providing a number of explicit finite time bounds for TD, including in the much more complicated case where data is generated from a single trajectory of a Markov chain. Our analysis inherits the simplicity of and elegance enjoyed by SGD analysis and can gracefully extend to different variants of TD, for example TD learning with eligibility traces (TD($\lambda$)) and Q-function approximation for optimal stopping problems. Owing to the close connection with SGD, we believe that optimization researchers can further build on our techniques to develop principled improvements to TD. 

There are a number of research directions one can take to extend our work. First, we use a projection step for analysis under the Markov chain model, a choice we borrowed from the optimization literature to simplify our analysis. It will be interesting to find alternative ways to add regularity to the TD algorithm and establish similar convergence results; we think analysis without the projection step is possible if one can show that the iterates remain bounded under additional regularity conditions. Second, the $\tilde{\mathcal{O}}(1/T)$ convergence rate we showed used step-sizes which crucially depends on the minimum eigenvalue $\omega$ of the feature covariance matrix, which would need to be estimated from samples. While such results are common in optimization for strongly convex functions, very recently \cite{lakshminarayanan2018linear} showed  TD(0) with iterate averaging and \emph{universal constant step-sizes} can attain an $\tilde{\mathcal{O}}(1/T)$ convergence rate in the i.i.d.\,sampling model. Extending our analysis for problem independent, robust step-size choices is a research direction worth pursuing.

\setlength{\bibsep}{4pt plus 0.4ex}
\bibliographystyle{plainnat}
\bibpunct{(}{)}{;}{a}{,}{,}
\bibliography{bibfile}
%-----------------------------------------------------

\newpage
\appendix

\section{Analysis of Projected TD(0) under the Markov chain sampling model} \label{appendix_a}
In this section, we complete the proof of Theorem \ref{thrm:mc_ana}. The first subsection restates the theorem, as well as the two key lemmas from Section \ref{sec:markov_chain_analyis} that underly the proof. The second subsection contains a proof of Theorem \ref{thrm:mc_ana}. Finally, Subsection \ref{appendix_a:mc_noise_bd} contains the proof of a technical result, Lemma \ref{lemma:mc_noise_bd}, which was omitted from the main text but we need for the proof. 

\subsection{Restatement of the theorem and key lemmas from the main text}
\MCBound*

The key to our proof is the following lemmas, which were establish in Section \ref{sec:markov_chain_analyis}. Recall the definition of the gradient error $\zeta_t(\theta) \equiv  \left(g_t(\theta) - \bar{g}(\theta)\right)^\top(\theta - \theta^*)$. 

\lemmamcgraddecomp*
\MCGradientBound*

\subsection{Proof of Theorem \ref{thrm:mc_ana}.}
\label{appendix_a:proof_thrm_mc}
We now complete the proof of Theorem \ref{thrm:mc_ana}. The proof directly uses Lemma \ref{lemma:mc_grad_decomp} and Lemma \ref{lemma:mc_grad_bias_bd}.  

\begin{proof}
From Lemma \ref{lemma:mc_grad_decomp}, we have
\begin{equation}\label{eq:mc_basic_rec}
\mathbb{E}\left[ \norm{\theta^* - \theta_{t+1}}_2^2 \right] \leq \mathbb{E} \left[ \norm{\theta^* - \theta_t}_2^2 \right] - 2 \alpha_t (1-\gamma) \mathbb{E}\left[\norm{V_{\theta^*} - V_{\theta_t}}^2_{D} \right] + 2\alpha_t \mathbb{E}\left[ \zeta_t(\theta_t) \right] + \alpha_t^2 G^2.
\end{equation}
\paragraph{Proof of part (a):} We first show the analysis for a constant step-size and iterate averaging. Considering $\alpha_t = \alpha_0 = 1/\sqrt{T}$ in Equation \eqref{eq:mc_basic_rec}, rearranging terms and summing from $t=0$ to $t=T-1$, we get
\begin{eqnarray*}
2 \alpha_0 (1-\gamma) \sum_{t=0}^{T-1} \mathbb{E}\left[\norm{V_{\theta^*} - V_{\theta_t}}^2_{D} \right] \leq \sum_{t=0}^{T-1} \left( \mathbb{E} \left[ \norm{\theta^* - \theta_t}_2^2 \right] - \mathbb{E}\left[ \norm{\theta^* - \theta_{t+1}}_2^2 \right] \right) + G^2 
+ 2\alpha_0 \sum_{t=0}^{T-1} \mathbb{E}\left[ \zeta_t(\theta_t) \right]. 
\end{eqnarray*}
Using Lemma \ref{lemma:mc_grad_bias_bd} (in which $\alpha_{t^*}=\alpha_0$ in this case) and simplifying, we find
\begin{eqnarray*}
\sum_{t=0}^{T-1} \mathbb{E}\left[\norm{V_{\theta^*} - V_{\theta_t}}^2_{D} \right] &\leq& 
\frac{\norm{\theta^* - \theta_0}_2^2 + G^2}{2\alpha_0 (1-\gamma)} + \frac{T \cdot 2G^2 (2 + 3\tau^{\mix}(1/\sqrt{T})) \alpha_0}{(1-\gamma)}\\
&=& \frac{\sqrt{T}\left(\norm{\theta^* - \theta_0}_2^2 + G^2\right)}{2(1-\gamma)} +  \frac{\sqrt{T} \cdot 2G^2 (2 + 3\tau^{\mix}(1/\sqrt{T}))}{(1-\gamma)}. \label{eq:mc_part_a_rec} 
\end{eqnarray*}
This gives us our desired result,
\begin{eqnarray*}
\mathbb{E}\left[\norm{V_{\theta^*} - V_{\bar{\theta}_T}}^2_{D}\right] \leq \frac{1}{T} \sum_{t=0}^{T-1} \mathbb{E}\left[\norm{V_{\theta^*} - V_{\theta_t}}^2_{D}\right] &\leq& \frac{\norm{\theta^* - \theta_0}_2^2 + G^2 \left(9+12\tau^{\mix}(1/\sqrt{T})\right)}{2\sqrt{T}(1-\gamma)}.
\end{eqnarray*}
\paragraph{Proof of part (b):} The proof is analogous to part (b) of Theorem \ref{thrm_iid}.  Consider a constant step-size of $\alpha_0 < 1/(2 \omega  (1-\gamma))$. Starting with Equation \eqref{eq:mc_basic_rec} and applying Lemma \ref{lemma:strong_conv}, which showed $\norm{V_{\theta^*} - V_{\theta}}^2_{D} \geq \omega  \norm{\theta^* - \theta}_2^2$ for all $\theta$, we get
\begin{eqnarray*}
\mathbb{E} \left[\norm{\theta^*-\theta_{t+1}}_2^2\right] &\leq& \left(1-2\alpha_0(1-\gamma)\omega \right) \mathbb{E} \left[\norm{\theta^*-\theta_t}_2^2\right] + \alpha_0^2 G^2 + 2\alpha_0 \mathbb{E}\left[\zeta_t(\theta_t)\right] \nonumber \\
&\leq& \left(1-2\alpha_0(1-\gamma)\omega \right) \mathbb{E} \left[\norm{\theta^*-\theta_t}_2^2\right] + \alpha_0^2 G^2 \left(9+12\tau^{\mix}(\alpha_0) \right), \label{eq:theta_rec}
\end{eqnarray*}
where we used Lemma \ref{lemma:mc_grad_bias_bd} to go to the second inequality. Iterating over this inequality gives us our final result. For any $T \in \mathbb{N}_0$,
\begin{eqnarray*}
\mathbb{E} \left[\norm{\theta^*-\theta_{T}}_2^2\right] &\leq& \left(1-2\alpha_0(1-\gamma)\omega \right)^T \norm{\theta^*-\theta_0}_2^2 + \alpha_0^2 G^2 \left(9+12\tau^{\mix}(\alpha_0) \right) \sum_{t=0}^{\infty}\left(1-2\alpha_0(1-\gamma)\omega \right)^t \\
&\leq& \left(e^{-2\alpha_0(1-\gamma)\omega T}\right) \norm{\theta^*-\theta_0}_2^2 + \frac{\alpha_0 G^2\left(9+12\tau^{\mix}(\alpha_0) \right)}{2(1-\gamma)\omega}.
\end{eqnarray*}
The second inequality above follows by solving the geometric series and then using the fact that $\left(1-2\alpha_0(1-\gamma)\omega \right) \leq e^{-2\alpha_0(1-\gamma)\omega}$.

\paragraph{Proof of part (c):} We now show the analysis for a linearly decaying step-size using Equation \eqref{eq:mc_basic_rec} as our starting point. We again use Lemma \ref{lemma:strong_conv}, which showed $\norm{V_{\theta^*} - V_{\theta}}^2_{D} \geq \omega  \norm{\theta^* - \theta}_2^2$ for all $\theta$, to get,
\begin{eqnarray*}
\mathbb{E}\left[\norm{V_{\theta^*} - V_{\theta_t}}^2_{D}\right] \leq \frac{1}{(1-\gamma)\alpha_t} \left((1-(1-\gamma)\omega \alpha_t) \mathbb{E}\left[\norm{\theta^* - \theta_t}_2^2\right] - \mathbb{E}\left[\norm{\theta^* - \theta_{t+1}}_2^2\right] + \alpha_t^2 G^2 \right) + \\
\frac{2}{(1-\gamma)} \mathbb{E}\left[\zeta_t(\theta_t)\right]. 
\end{eqnarray*}
Consider a decaying step-size $\alpha_t = \frac{1}{\omega (t+1) (1-\gamma)}$, simplify and sum from $t=0$ to $T-1$ to get
\begin{eqnarray}
\sum_{t=0}^{T-1} \mathbb{E}\left[\norm{V_{\theta^*} - V_{\theta_t}}^2_{D}\right] &\leq& \underbrace{-\omega T \mathbb{E}\left[ \norm{\theta^* - \theta_T}_2^2 \right]}_{< 0} + \frac{G^2}{\omega (1-\gamma)^2} \sum_{t=0}^{T-1} \frac{1}{t+1} + \frac{2}{(1-\gamma)} \sum_{t=0}^{T-1} \mathbb{E}\left[\zeta_t(\theta_t)\right]. \nonumber\\ \label{eq:mc_ana_rec_partb}
\end{eqnarray}
To simplify notation, for the remainder of the proof put $\tau=\tau^{\mix}(\alpha_T)$. We can decompose the sum of gradient errors as
\begin{eqnarray}
\sum_{t=0}^{T-1} \mathbb{E}\left[\zeta_t(\theta_t)\right] = \sum_{t=0}^{\tau} \mathbb{E}\left[\zeta_t(\theta_t)\right] + \sum_{t=\tau + 1}^{T-1} \mathbb{E}\left[\zeta_t(\theta_t)\right].
\end{eqnarray}
We will upper bound each term. In each case we use that, since $\alpha_t = \frac{1}{\omega (t+1) (1-\gamma)}$, 
\[ 
\sum_{t=0}^{T-1}\alpha_t = \frac{1}{\omega(1-\gamma)} \sum_{t=0}^{T-1} \frac{1}{(t+1)} \leq   \frac{1+\log T}{\omega(1-\gamma)}.
\]
Combining this with Lemma \ref{lemma:mc_grad_bias_bd} gives, 
\[
\sum_{t=0}^{\tau} \mathbb{E}\left[\zeta_t(\theta_t)\right] \leq \sum_{t=0}^{\tau} \left( 6G^2 \sum_{i=0}^{t-1} \alpha_i \right) \leq  \tau \left( 6G^2 \sum_{i=0}^{T-1} \alpha_i \right) \leq \frac{6G^2 \tau}{\omega(1-\gamma)} (1 + \log T).
\]
Similarly, using Lemma \ref{lemma:mc_grad_bias_bd}, we have
\[
\sum_{t=\tau+ 1}^{T-1} \mathbb{E}\left[\zeta_t(\theta_t)\right] \leq 2G^2 \left(2 + 3\tau\right) \sum_{t=\tau + 1}^{T-1} \alpha_{t-\tau} \leq 2G^2 \left(2 + 3\tau\right) \sum_{t=1}^{T-1} \alpha_{t} \leq \frac{2G^2 \left(2 + 3\tau\right)}{\omega(1-\gamma)} \left(1 + \log T \right).
\]
Combining the two parts, we get 
\[
\sum_{t=0}^{T-1} \mathbb{E}\left[\zeta_t(\theta_t)\right] \leq \frac{4G^2 \left(1+3\tau\right)}{\omega(1-\gamma)} (1 + \log T).
\]
Using this in conjunction with Equation \eqref{eq:mc_ana_rec_partb} we give final result.
\begin{eqnarray*}
\mathbb{E}\left[\norm{V_{\theta^*} - V_{\bar{\theta}_T}}^2_{D}\right] \leq \frac{1}{T} \sum_{t=0}^{T-1} \mathbb{E}\left[\norm{V_{\theta_t} - V_{\theta^*}}^2_{D}\right] \leq \frac{G^2}{\omega T(1-\gamma)^2} \left(1+\log T\right) + \frac{2}{T(1-\gamma)} \sum_{t=0}^{T-1} \mathbb{E}\left[\zeta_t(\theta_t)\right].
\end{eqnarray*}
Simplifying and substituting $\tau=\tau^{\mix}(\alpha_T)$, we get 
\begin{eqnarray*}
\mathbb{E}\left[\norm{V_{\theta^*} - V_{\bar{\theta}_T}}^2_{D}\right] &\leq& \frac{G^2}{\omega T(1-\gamma)^2} \left(1+\log T\right) + \frac{8G^2 \left(1 + 3\tau^{\mix}(\alpha_T)\right)}{\omega T(1-\gamma)^2} (1 + \log T) \\
&\leq& \frac{ G^2 \left(9+24\tau^{\mix}(\alpha_T)\right)}{\omega T(1-\gamma)^2} \left(1+\log T\right).
\end{eqnarray*}
Additionally, Equation \eqref{eq:mc_ana_rec_partb} also gives us a convergence rate of $\mathcal{O}(\log T/T)$ for the iterate $\theta_T$ itself:
\begin{eqnarray*}
\label{eq:mc_theta_result}
\mathbb{E}\left[\norm{\theta^* - \theta_T}_2^2 \right] &\leq& \frac{ G^2 \left(9+24\tau^{\mix}(\alpha_T)\right)}{\omega^2 T(1-\gamma)^2} \left(1+\log T\right).
\end{eqnarray*}
\end{proof}

\subsection{Proof of Lemma \ref{lemma:mc_noise_bd}}
\label{appendix_a:mc_noise_bd}
\lemmagradbias*
\begin{proof}
The first claim follows from a simple argument using Lemma \ref{lemma:mc_grad_norm}.
\[
\abs{\zeta_t (\theta)} = \abs{ \left(g_t(\theta) - \bar{g}(\theta)\right)^\top(\theta - \theta^*)} \leq \left(\norm{g_t(\theta)}_2 + \norm{\bar{g}(\theta)}_2 \right)\left( \norm{\theta}_2 + \norm{\theta^*}_2\right)  \leq 4GR \leq 2G^2,
\]
where the first inequality follows from the triangle inequality and the Cauchy-Schwartz inequality, and the final inequality uses that $R\leq G/2$ by definition of $G=r_{\max}+2R$. 

To establish the second claim, consider the following inequality for any vectors $(a_1, b_1, a_2, b_2)$:
\[
\abs{a_1^\top b_1-a_2^\top b_2} = \abs{a_1^\top(b_1-b_2) + b_2^\top(a_1-a_2)} \leq \norm{a_1}\norm{b_1-b_2} + \norm{b_2}\norm{a_1-a_2}.
\]
This follows as a direct application of Cauchy-Schwartz. It implies that for any $\theta, \theta'\in \Theta_R$, 
\begin{eqnarray*}
\abs{\zeta_{t}(\theta) - \zeta_{t}(\theta')} &=& \abs{\left(g_t(\theta) - \bar{g}(\theta)\right)^\top(\theta - \theta^*) -  \left(g_t(\theta') - \bar{g}(\theta')\right)^\top(\theta' - \theta^*)} \\
&\leq& \norm{g_t(\theta) - \bar{g}(\theta)}_2 \norm{\theta - \theta'}_2 + \norm{\theta' - \theta^*}_2 \norm{\left(g_t(\theta) - \bar{g}(\theta)\right) -  \left(g_t(\theta') - \bar{g}(\theta')\right)}_2 \\
&\leq& 2G\norm{\theta - \theta'}_2 + 2R \left( \norm{g_t(\theta) - g_t(\theta')}_2 +  \norm{\bar{g}(\theta) - \bar{g}(\theta')}_2 \right) \\
&\leq& 2G\norm{\theta - \theta'}_2 + 8R \norm{\theta - \theta'}_2 \\
&\leq& 6G\norm{\theta - \theta'}_2.
\end{eqnarray*}
where we used that $R \leq G/2$ by the definition of $G$. We also used that both $g_t(\cdot)$ and $\bar{g}(\cdot)$ are 2-Lipschitz functions which is easy to see. Starting with $g_t(\theta) = ( r_t + \gamma \phi(s_{t}')^\top \theta - \phi(s_t)^\top \theta ) \phi(s_t)$, consider
\begin{eqnarray*}
\norm{g_t(\theta)-g_t(\theta')}_2 &=& \norm{\phi(s_t) \left( \gamma \phi(s_{t}') - \phi(s_t) \right)^\top(\theta - \theta')}_2 \\
&\leq& \norm{\phi(s_t)}_2 \norm{\left( \gamma \phi(s_{t}') - \phi(s_t) \right)}_2 \norm{(\theta - \theta')}_2 \\
&\leq& 2 \norm{(\theta - \theta')}_2.
\end{eqnarray*}
Similarly, following Equation \eqref{eq: expectation form of gbar}, we have $\norm{\bar{g}(\theta) - \bar{g}(\theta')}_2 = \norm{\E \left[ \phi \left( \gamma \phi' - \phi \right) \right]^\top (\theta - \theta')}_2$, where $\phi = \phi(s)$ is the feature vector of a random initial state $s\sim \pi$, $\phi'=\phi(s')$ is the feature vector of a random next state drawn according to $s' \sim \Pc(\cdot \mid s)$. Therefore, 
\begin{eqnarray*}
\norm{\bar{g}(\theta) - \bar{g}(\theta')}_2 \leq \norm{\phi \left( \gamma \phi' - \phi \right)^\top(\theta - \theta')} \leq 2 \norm{(\theta - \theta')}_2.
\end{eqnarray*}

\end{proof}

% \pagebreak
% \newpage\phantom{blabla}
% \newpage\phantom{blabla}

\newpage

\section{Analysis of Projected TD($\lambda$) under Markov chain observation model} 
\label{appendix_b:TD_lambda}
In this section, we give a detailed proof of the convergence bounds presented in Theorem \ref{thrm:lambda_mc_ana}. Subsection \ref{appendix_b:key_lemmas} details our proof strategy along with key lemmas which come together in Subsection \ref{appendix_b:proof_mc_lambda} to establish the results. We begin by providing mathematical expressions for TD($\lambda$) updates.

\paragraph{Stationary distribution of TD($\lambda$) updates:} Recall that the projected TD($\lambda$) update at time $t$ is given by:
\[
\theta_{t+1} = \Pi_{2,R}\left( \theta_t + \alpha_t x_t(\theta_t, z_{0:t}) \right)
\]
where $\Pi_{2,R}(\cdot)$ denotes the projection operator onto a norm ball of radius $R < \infty$ and $x_t(\theta_t, z_{0:t})$ is the update direction. Let us now give explicit mathematical expressions for $x_{t}(\theta, z_{0:t})$ and its steady-state mean $\bar{x}(\theta)$. Note that these are analogous to the expressions for the negative gradient $g_{t}(\theta)$ and its steady-state expectation $\bar{g}(\theta)$ for TD(0). At time $t$, as a general function of (non-random) $\theta$ and the tuple $\mathcal{O}_t = (s_t,r_t, s'_t)$ along with the eligibility trace term $z_{0:t}$, we have
\begin{eqnarray*}
x_t(\theta, z_{0:t}) = \left( r_t + \gamma \phi(s'_{t})^\top \theta - \phi(s_t)^\top \theta \right) z_{0:t} = \delta_t(\theta) z_{0:t} \,\, \quad \forall \,\, \theta\in \mathbb{R}^d.
\end{eqnarray*}
The asymptotic convergence of TD($\lambda$) is closely related to the expected value of $x_t(\theta, z_{0:t})$ under the steady-state behavior of $(O_t, z_{0:t})$,
\[
\bar{x}(\theta)=\lim_{t \rightarrow \infty} \mathbb{E} \left[ \delta_t(\theta) z_{0:t} \right]. 
\]
Rather than take this limit, it will be helpful in our analysis to think of an equivalent \emph{backward view} by constructing a stationary process with mean $\bar{x}(\theta)$. Consider a stationary sequence of states $(\ldots, s_{-1}, s_{0}, s_1, \ldots)$ and set $z_{-\infty:t}=\sum_{k=0}^\infty (\gamma \lambda)^k \phi(s_{t-k})$. Then the sequence $(x_{0}(\theta, z_{-\infty:0}), x_1(\theta, z_{-\infty:1}), \ldots)$ is stationary, and we have 
% it's expectation under the steady-state distribution is $\bar{x}(\theta)$.
\begin{eqnarray}
\label{eq:expec_grad_TD_lambda}
\bar{x}(\theta) = \mathbb{E}\left[ \delta_t(\theta) z_{-\infty:t} \right].
\end{eqnarray}
It should be emphasized that $\bar{x}(\theta)$ and the states $(\ldots, s_{-2}, s_{-1})$ are introduced only for the purposes of our analysis and are never used by the algorithm itself. However, this turns out to be quite useful as it is easy to show \citep{van1998learning} that 
\begin{eqnarray}
\label{eq:expect_result_TD_lambda}
\bar{x}(\theta) = \Phi^\top D \left( T_{\mu}^{(\lambda)} \Phi \theta - \Phi \theta \right),
\end{eqnarray}
where $\Phi$ is the feature matrix and $( T_{\mu}^{(\lambda)} \Phi \theta - \Phi \theta )$ denotes the Bellman error defined with respect to the Bellman operator $T_{\mu}^{(\lambda)}(\cdot)$, corresponding to a policy $\mu$. Careful readers will notice the stark similarity between Equation \eqref{eq:expect_result_TD_lambda} and Equation \eqref{eq:Expected grad matrix form}. Exploiting the property that $\Pi_D T_{\mu}^{\lambda}(\cdot)$ is also a contraction operator, one can easily show a result equivalent to Lemma \ref{lemma:expected_gradient}, thus quantifying the progress we make by taking steps in the direction of $\bar{x}(\theta)$. The rest of our proof essentially shows how to control for the observation noise, i.e. the fact that we use $x_t(\theta, z_{0:t})$ rather than $\bar{x}(\theta)$ to make updates. To remind the readers of the results, we first restate Theorem \ref{thrm:lambda_mc_ana} below.

% \paragraph{Restatement of Theorem \ref{thrm:lambda_mc_ana}}
% \label{appendix_b:thrm_mc_lambda}
\LambdaMCBound*

%-------------------------------------------------------------------------------------------------------------------%
%-------------------------------------------------------------------------------------------------------------------%
\subsection{Proof strategy and key lemmas}
\label{appendix_b:key_lemmas}
We now describe our proof strategy and give key lemmas used to establish Theorem \ref{thrm:lambda_mc_ana}. Throughout, we consider the Markov chain observation model with Assumption \ref{as:6} and study the Projected TD ($\lambda$) algorithm applied with parameter $R \geq \norm{\theta^*}_2$ and step-size sequence $(\alpha_0, \ldots, \alpha_T)$. To simplify our exposition, we introduce some notation below.

\paragraph{Notation:} We specify the notation used throughout this section. Define the set $\Theta_{R}= \{\theta \in \mathbb{R}^d : \| \theta\|_2 \leq R \}$, so $\theta_t \in \Theta_R$ for each $t$ because of the algorithm's projection step. Next, we generically define $z_{l:t} = \sum_{k=0}^{t-l} (\gamma \lambda)^k \phi(s_{t-k})$ for any lower limit $l \leq t$. Thus, $z_{l:t}$ denotes the eligibility trace as a function of the states $(s_l, \ldots, s_t)$. Next, we define $\zeta_t(\theta, z_{l:t})$ as a general function of $\theta$ and $z_{l:t}$,
\begin{eqnarray}
\label{eq:error_fn_lambda}
\zeta_t(\theta, z_{l:t}) = \left(\delta_t(\theta) z_{l:t} - \bar{x}(\theta) \right)^\top (\theta - \theta^*).
\end{eqnarray}
Here, the subscript $t$ in $\zeta_t$ encodes the dependence on the tuple $O_t = (s_t,r_t,s'_t)$ which is used to compute the Bellman error, $\delta_t(\cdot)$ at time $t$. Finally, we set $B:=(r_{\max} + 2R)/(1 - \gamma \lambda)$ which implies $B > R/2$, a fact we use many times in our proofs to simplify constant terms. As a reminder, note that our bounds depend on the mixing time, which we defined in Section \ref{sec:td_lambda} as
\[
\tau^{\mix}_{\lambda}(\epsilon) = \max \{ \tau^{\text{MC}}(\epsilon), \tau^{\text{Algo}}(\epsilon) \},
\]
where $\tau^{\text{MC}}(\epsilon) = \min \{ t \in \mathbb{N}_0 \,\, | \,\, m \rho^t \leq \epsilon \}$ and $\tau^{\text{Algo}}(\epsilon) = \min \{ t \in \mathbb{N}_0 \,\, | \,\, (\gamma \lambda)^t \leq \epsilon \}$.

\paragraph{Proof outline:}
The analysis for TD($\lambda$) can be broadly divided into three parts and closely mimics the steps used to prove TD(0) results. 
\begin{enumerate}
\item As a first step, we do an error decomposition, similar to the result shown in Lemma \ref{lemma:mc_grad_decomp}. This is enabled by two key lemmas, which are analogues of Lemma \ref{lemma:expected_gradient} and Lemma \ref{lemma:mc_grad_norm} for Projected TD(0). The first one spells out a clear relationship of how the updates following $\bar{x}(\theta)$ point in the descent direction of $\| \theta^* - \theta \|^2_2$ while the second one upper bounds the norm of the update direction, $x_t(\theta, z_{0:t})$, by the constant $B$ (as defined above).
\item The error decomposition that we obtain from Step 1 can be stated as:
\[
\E[\norm{\theta^* - \theta_{t+1}}_2^2] \leq \E[\norm{\theta^* - \theta_t}_2^2] - 2 \alpha_t (1-\kappa) \E[\norm{V_{\theta^*} - V_{\theta_t} }^2_{D}] + 2\alpha_t \E[\zeta_t(\theta_t, z_{0:t})] + \alpha_t^2 B^2. 
\]
In the second step, we establish an upper bound on the bias term, $\E \left[ \zeta_t(\theta_t, z_{0:t}) \right]$, which is the main challenge in our proof. Recall that the dependent nature of the state transitions may result in strong coupling between the tuples $\mathcal{O}_{t-1}$ and $\mathcal{O}_t$ under the Markov chain observation model. Therefore, this bias in update direction can potentially be non-zero. Presence of the eligibility trace term, $z_{0:t}$, which is a function of the entire history of states, $(s_0, \ldots, s_t)$, further complicates the analysis by introducing subtle dependencies. 

To control for this, we use information-theoretic techniques shown in Lemma \ref{lemma:mc_exp_bd} which exploit the geometric ergodicity of the MDP, along with the geometric weighting of state features in the eligibility trace term. Our result essentially shows that the bias scales the noise in update direction by a factor of the mixing time. Mathematically, for a constant step-size $\alpha$, we show that $\mathbb{E}\left[\alpha \zeta_t(\theta_t, z_{0:t}) \right] \approx B^2 (6 + 12 \tau^{\mix}_{\lambda}(\alpha)) \alpha^2$. We show a similar result  for decaying step-sizes as well. 
\item In the final step, we combine the error decomposition from Step 1 and the bound on the bias from Step 2, to establish finite time bounds on the performance of Projected TD($\lambda$) for different step-size choices. We closely mimic the analysis of \cite{nemirovski2009robust} for a constant, aggressive step-size of $(1/\sqrt{T})$ and the proof ideas of \cite{lacoste2012simpler} for decaying step-sizes.
\end{enumerate}

\subsubsection{Error decomposition under Projected TD($\lambda$)}
We first prove Lemmas \ref{lemma:expected_grad_td_lambda} and \ref{lemma:norm_grad_lambda} which enable the error decomposition shown in Lemma \ref{lemma:mc_lambda_grad_decomp}. 
%---------------------------------------------------------------------------------------------------%
\begin{restatable}[]{lem}{lemmabasictdlambda}
[\textbf{\cite{tsitsiklis1997analysis}}]
\label{lemma:expected_grad_td_lambda}
Let $V_{\theta^{*}}$ be the unique fixed point of $\Pi_{D} T_{\mu}^{(\lambda)}(\cdot)$ i.e. $V_{\theta^*} = \Pi T_{\mu}^{(\lambda)} V_{\theta^*}$. For any $\theta \in \mathbb{R}^d$, 
\begin{equation*}
(\theta^{*} - \theta)^\top \bar{x}(\theta) \,\, \geq \,\, (1-\kappa) \norm{V_{\theta^*} - V_{\theta}}^2_{D}.
\end{equation*}
\end{restatable}
\begin{proof}
We use the definition of $\bar{x}(\theta) = \langle \Phi^\top,T_{\mu}^{(\lambda)} \Phi \theta - \Phi \theta \rangle_D$ as shown in Equation \eqref{eq:expect_result_TD_lambda} along with the fact that $\Pi_{D} T_{\mu}^{(\lambda)} (\cdot)$ is a contraction with respect to $\| \cdot \|_{D}$ with modulus $\kappa$. See Appendix \ref{appendix:a1} for a complete proof.
\end{proof} 
%---------------------------------------------------------------------------------------------------%
\begin{restatable}[]{lem}{lemmasixth}
\label{lemma:norm_grad_lambda}
For all $\theta \in \Theta_R$, $\norm{x_t(\theta, z_{0:t})}_2 \leq B$ with probability 1. Additionally, $ \norm{\bar{x}(\theta)}_2 \leq B$.
% \[
% \norm{x_t(\theta, z_{0:t})}_2 \leq B \,\,\, \text{and} \,\,\, \norm{\bar{x}(\theta)}_2 \leq B.
% \]
\end{restatable}
\begin{proof}
See Subsection \ref{appendix_b:proof_supp_lemma_lambda} for a complete proof.
\end{proof}
%---------------------------------------------------------------------------------------------------
The above two lemmas can be easily combined to establish a recursion for the error under projected TD($\lambda$) that holds for each sample path. 
\begin{restatable}[]{lem}{lemmamclambdagraddecomp}
\label{lemma:mc_lambda_grad_decomp} 
With probability 1, for every $t\in \mathbb{N}_0$,
\[
\norm{\theta^* - \theta_{t+1}}_2^2 \leq \norm{\theta^* - \theta_t}_2^2 - 2 \alpha_t (1-\kappa) \norm{V_{\theta^*} - V_{\theta_t} }^2_{D} + 2\alpha_t \zeta_t(\theta_t, z_{0:t}) + \alpha_t^2 B^2. 
\]
\end{restatable}
\begin{proof}
The Projected TD($\lambda$) algorithm updates the parameter as: $\theta_{t+1} = \Pi_{2,R} [\theta_t + \alpha_t x_t(\theta_t,z_{0:t})] \,\,\, \forall \,\, t \in \mathbb{N}_0$. This implies, 
\begin{eqnarray*}
\norm{\theta^* - \theta_{t+1}}_2^2 &=& \norm{\theta^* - \Pi_{2,R}(\theta_t + \alpha_t x_t(\theta, z_{0:t}))}_2^2 \\
&=& \norm{\Pi_{2,R}(\theta^*) - \Pi_{2,R}(\theta_t + \alpha_t x_t(\theta_t, z_{0:t}))}_2^2 \\
&\leq& \norm{\theta^* - \theta_t - \alpha_t x_t(\theta_t, z_{0:t})}_2^2 \\
&=& \norm{\theta^* - \theta_t}_2^2 - 2\alpha_t x_t(\theta_t, z_{0:t})^\top (\theta^* - \theta_t) + \alpha_t^2 \norm{x_t(\theta_t, z_{0:t})}_2^2 \\
&\leq& \norm{\theta^* - \theta_t}_2^2 - 2\alpha_t x_t(\theta_t, z_{0:t})^\top (\theta^* - \theta_t) + \alpha_t^2 B^2 \\
&=&  \| \theta^* - \theta_t \|_2^2 - 2\alpha_t \bar{x}(\theta_t)^\top (\theta^*-\theta_t) + 2\alpha_t \zeta_t(\theta_t, z_{0:t}) + \alpha_t^2 G^2.\\
&\leq& \norm{\theta^* - \theta_t}_2^2 - 2\alpha_t (1-\kappa) \norm{V_{\theta^*} - V_{\theta}}^2_{D} + 2\alpha_t \zeta_t(\theta_t, z_{0:t}) + \alpha_t^2 B^2.
\end{eqnarray*}
The first inequality used that orthogonal projection operators onto a convex set are non-expansive, the second used Lemma \ref{lemma:norm_grad_lambda} together with the fact $\|\theta_t\|_2 \leq R$ due to projection, and the third used Lemma \ref{lemma:expected_grad_td_lambda}.  Note that we used $\zeta_t(\theta_t, z_{0:t})$ to simplify the notation for the error in the update direction. Recall the definition of the error function from Equation \eqref{eq:error_fn_lambda} which implies,
\begin{eqnarray*}
\zeta_t(\theta_t, z_{0:t}) = (\delta_t(\theta_t)z_{0:t} - \bar{x}(\theta_t))^\top (\theta_t - \theta^*) = (x_t(\theta_t, z_{0:t}) - \bar{x}(\theta_t))^\top (\theta_t - \theta^*).
\end{eqnarray*}
\end{proof}

%-------------------------------------------------------------------------------------------------------------------%
%-------------------------------------------------------------------------------------------------------------------%
\subsubsection{Upper bound on the bias in update direction.}
We give an upper bound on the expected error in the update direction, $\E[\zeta_t(\theta_t, z_{0:t})]$, which as explained above, is the key challenge for our analysis. For this, we first establish some basic regularity properties of the error function $\zeta_t(\cdot, \cdot)$ in Lemma \ref{lemma:lip_lambda} below. In particular, part (a) shows boundedness, part (b) shows that it is Lipschitz in the first argument and part (c) bounds the error due to truncation of the eligibility trace. Recall that $z_{l:t}$ denotes the eligibility trace as a function of the states $(s_l, \ldots, s_t)$.
%-------------------------------------------------------------------------------------------------------------------%
\begin{restatable}[]{lem}{lemmaliplambda}
\label{lemma:lip_lambda}
Consider any $l \leq t$ and any $\theta, \theta' \in \Theta_R$. With probability 1,
\begin{enumerate}[label=(\alph*)]
\item $\abs{\zeta_t(\theta,z_{l:t})} \leq 2B^2$.
\item $\abs{\zeta_t (\theta,z_{l:t}) - \zeta_t (\theta',z_{l:t})} \leq 6B \norm{(\theta - \theta^')}_2.$
% $\zeta_t(\cdot, z_{l:t}): \Theta_R \rightarrow \mathbb{R}$ is $L$-Lipschitz with $L=6B$. 
\item The following two bounds also hold,
\begin{eqnarray*}
\abs{\zeta_t(\theta,z_{0:t}) - \zeta_t(\theta,z_{t-\tau:t})} &\leq& B^2(\gamma \lambda)^\tau \,\, \text{for all} \,\, \tau \leq t, \\
\abs{\zeta_t(\theta,z_{0:t}) - \zeta_t(\theta,z_{-\infty:t})} &\leq& B^2(\gamma \lambda)^t.
\end{eqnarray*}

\end{enumerate}
\end{restatable}
\begin{proof}
We essentially use the uniform bound on $x_t(\theta, z_{0:t})$ and $\bar{x}(\theta)$ as stated in Lemma \ref{lemma:norm_grad_lambda} to show this result. See Subsection \ref{appendix_b:proof_supp_lemma_lambda} for a detailed proof.
\end{proof}
%-------------------------------------------------------------------------------------------------------------------%
Lemma \ref{lemma:lip_lambda} can be combined with Lemma \ref{lemma:mc_exp_bd} to give an upper bound on the bias term, $\E\left[ \zeta_t(\theta_t, z_{0:t}) \right]$, as shown below.
\begin{restatable}[]{lem}{}
Consider a non-increasing step-size sequence, $\alpha_0 \geq \alpha_1 \ldots \geq \alpha_T$. Then the following hold. 
\label{lemma:lambda_grad_bias_bd}
\begin{enumerate}[label=(\alph*)]
\item For $\,\, 2 \tau^{\mix}_{\lambda}(\alpha_T) < t \leq T$,
\begin{eqnarray*}
\E\left[ \zeta_t(\theta_t, z_{0:t}) \right] \leq 6B^2 (1 + 2\tau^{\mix}_{\lambda}(\alpha_T)) \alpha_{t- 2\tau^{\mix}_{\lambda}(\alpha_T)}.
\end{eqnarray*}
\item For $\,\, 0 \leq t \leq 2 \tau^{\mix}_{\lambda}(\alpha_T)$,
\begin{eqnarray*}
\E\left[ \zeta_t(\theta_t, z_{0:t}) \right] \leq 6B^2\left(1+2\tau^{\mix}_{\lambda}(\alpha_T)\right)\alpha_0 + B^2(\gamma \lambda)^{t}.
\end{eqnarray*}
\item For all $ t \in \mathbb{N}_0$,
\begin{eqnarray*}
\E\left[ \zeta_t(\theta_t, z_{0:t}) \right] \leq 6B^2 \sum_{i=0}^{t-1} \alpha_i + B^2(\gamma \lambda)^{t}
\end{eqnarray*}
\end{enumerate}
\end{restatable}

\begin{proof}
We proceed in two cases below. Throughout the proof, results from Lemma \ref{lemma:lip_lambda} are applied using the fact that $\theta_t \in \Theta_R$, because of the algorithm's projection step.
\paragraph{Case (a):} Let $t > 2\tau$ and consider the following decomposition for all $\tau \in \{0,1,\ldots, t/2\}$. We show an upper bound on each of the three terms separately.
\begin{eqnarray*}
\mathbb{E}\left[ \zeta_t(\theta_t, z_{0:t}) \right] \leq \abs{\mathbb{E}\left[ \zeta_t(\theta_t, z_{0:t}) \right] - \mathbb{E}[\zeta_t(\theta_{t-2\tau}, z_{0:t})]} + \abs{\mathbb{E}[\zeta_t(\theta_{t-2\tau}, z_{0:t})] - \mathbb{E}[\zeta_t(\theta_{t-2\tau}, z_{t-\tau:t})]} + \abs{\mathbb{E}[\zeta_t(\theta_{t-2\tau}, z_{t-\tau:t})]}.
\end{eqnarray*}
\underline{\it Step 1: Use regularity properties of the error function to bound first two terms.}
\vspace{1mm}\\
We relate $\zeta_{t}(\theta_t,z_{0:t})$ and $\zeta_{t}(\theta_{t-\tau},z_{0:t})$ using the Lipschitz property shown in part (b) of Lemma \ref{lemma:lip_lambda} to get,
\begin{eqnarray}
\label{eq:decomp_a_1}
\abs{\zeta_t(\theta_t, z_{0:t}) - \zeta_t(\theta_{t-2\tau}, z_{0:t})} = 6B \norm{\theta_t - \theta_{t-2\tau}}_2 \leq 6B^2 \sum_{i=t-2\tau}^{t-1} \alpha_i.
\end{eqnarray}
Taking expectations on both sides gives us the desired bound on the first term. The last inequality used the norm bound on update direction as shown in Lemma \ref{lemma:norm_grad_lambda} to simplify,
\[
\norm{\theta_t - \theta_{t-2\tau}}_2 \leq \sum_{i=t-2\tau}^{t-1} \norm{\Pi_{2,R}\left( \theta_{i+1} + \alpha_i x_i(\theta_i,z_{0:i}) \right) -\theta_i}_2 \leq \sum_{i=t-2\tau}^{t-1} \alpha_i \norm{x_i(\theta_i, z_{0:i})}_2 \leq B \sum_{i=t-2\tau}^{t-1} \alpha_i.
\]
% \underline{\it Step 2: Uniformly bound $\abs{\zeta_{t}(\theta_{t-2\tau},z_{0:t}) - \zeta_{t}(\theta_{t-2\tau},z_{t-\tau:t})}$.} \vspace{1mm}\\
Similarly, by part (c) of Lemma \ref{lemma:lip_lambda}, we have a bound on the second term.
\begin{eqnarray}
\label{eq:decomp_a_2}
\abs{\E[\zeta_{t}(\theta_{t-2\tau},z_{0:t})] - \E[\zeta_{t}(\theta_{t-2\tau},z_{t-\tau:t})]} \leq B^2 (\gamma \lambda)^\tau.
\end{eqnarray}
\underline{\it Step 2: Use information-theoretic arguments to upper bound $\E[\zeta_{t}(\theta_{t-2\tau},z_{t-\tau:t})]$.} \vspace{1mm}\\
We will essentially use Lemma \ref{lemma:mc_exp_bd} to upper bound $\E[\zeta_{t}(\theta_{t-2\tau},z_{t-\tau:t})]$. We first introduce some notation to highlight subtle dependency issues. Note that $\zeta_t(\theta_{t-2\tau}, z_{t-\tau:t})$ is a function of $(\theta_{t-2\tau},s_{t-\tau},\ldots,s_{t-1},O_t)$. To simplify, let $Y_{t-\tau:t} = (s_{t-\tau},\ldots,s_{t-1},O_t)$. Define,
\[
f(\theta_{t-2\tau},Y_{t-\tau:t}) := \zeta_{t}(\theta_{t-2\tau},z_{t-\tau:t}).
\]
Consider random variables $\theta'_{t-2\tau}$ and $Y'_{t-\tau:t}$ drawn independently from the marginal distributions of $\theta_{t-2\tau}$ and $Y_{t-\tau:t}$, so $\Prob(\theta'_{t-2\tau}=\cdot, Y'_{t-\tau:t}=\cdot)=\Prob(\theta_{t-\tau}=\cdot)\otimes \Prob(Y_{t-\tau:t}=\cdot)$. By Lemma \ref{lemma:lip_lambda} we have that $\abs{f(\theta, Y_{t-\tau:t})} \leq 2B^2$ for all $\theta \in \Theta_R$ with probability 1. As 
\[ 
\theta_{t-2\tau} \rightarrow s_{t-2\tau} \rightarrow s_{t-\tau} \rightarrow s_t \rightarrow O_t
\]
form a Markov chain, a direct application of Lemma \ref{lemma:mc_exp_bd} gives us: 
\begin{eqnarray}
\label{eq:coupling_first_piece_bd}
\abs{\E[f(\theta_{t-2\tau},Y_{t-\tau:t}) ] - \E[ f(\theta'_{t-2\tau},Y'_{t-\tau:t}) ]} \leq 4B^2 m\rho^\tau.
\end{eqnarray}
We also have the following bound for all fixed $\theta \in \Theta_R$. Using $\bar{x}(\theta) = \E[\delta_t(\theta) z_{-\infty:t}]$, we get
\begin{eqnarray*}
\label{eq:coupling_second_piece_bd}
\E[f(\theta, Y_{t-\tau:t})] =  \left(\E[\delta_t(\theta)z_{t-\tau:t}] - \bar{x}(\theta)\right)^\top (\theta - \theta^*) 
% &\leq&  \abs{\left(\E[\delta_t(\theta) z_{-\infty:t-\tau}]\right)^\top (\theta - \theta^*)}  \\
&\leq& \abs{\left( \delta_t(\theta) z_{-\infty:t-\tau} \right)^\top (\theta - \theta^*)} \leq B^2 (\gamma \lambda)^\tau
\end{eqnarray*}
Combining the above with Equation \eqref{eq:coupling_first_piece_bd}, we get
\begin{eqnarray}
\label{eq:decom_a_3}
\abs{\E[\zeta_{t}(\theta_{t-2\tau},z_{t-\tau:t})]} &=& \abs{\E[f(\theta_{t-2\tau},Y_{t-\tau:t}) ]} \nonumber \\
&\leq& \abs{\E[f(\theta_{t-2\tau},Y_{t-\tau:t}) ] - \E[ f(\theta'_{t-2\tau},Y'_{t-\tau:t}) ]} + \abs{\E[ f(\theta'_{t-2\tau},Y'_{t-\tau:t}) ]} \nonumber \\
&\leq& 4B^2 m\rho^\tau + \abs{\E[ \E[ f(\theta'_{t-2\tau},Y'_{t-\tau:t}) | \theta'_{t-2\tau}] ]} \nonumber \\
&\leq& 4B^2 m\rho^\tau + B^2 (\gamma \lambda)^\tau.
\end{eqnarray}
\underline{\it Step 3. Combine terms to show part (a) of our claim.} \vspace{1mm}\\
Taking $\tau = \tau^{\mix}_{\lambda}(\alpha_T)$ and combining Equations \eqref{eq:decomp_a_1}, \eqref{eq:decomp_a_2} and \eqref{eq:decom_a_3} establishes the first claim.
\begin{eqnarray*}
\mathbb{E}\left[ \zeta_t(\theta_t, z_{0:t}) \right] \leq 6B^2 \sum_{i=t-2\tau}^{t-1} \alpha_i + 4B^2 m\rho^\tau + 2B^2(\gamma \lambda)^{\tau} &\leq& 12B^2 \tau^{\mix}_{\lambda}(\alpha_T) \alpha_{t-2\tau^{\mix}_{\lambda}(\alpha_T)} + 6B^2 \alpha_T \\
&\leq& 6B^2(1+2\tau^{\mix}(\alpha_T)) \alpha_{t-2\tau^{\mix}_{\lambda}(\alpha_T)}.
\end{eqnarray*}
Here we used that letting $\tau = \tau^{\mix}_{\lambda}(\alpha_T)$ implies: $\max \{ m\rho^{\tau}, (\gamma \lambda)^{\tau} \} \leq \alpha_T$. Two additional facts which we also use follow from a non-increasing step-size sequence, $\sum_{i=t-2\tau}^{t-1} \alpha_i \leq 2\tau \alpha_{t-2\tau}$ and $\alpha_T \leq \alpha_{t-2\tau}$.
%---------------------------------------------%
\paragraph{Case (b):} Consider the following decomposition for all $t \in \mathbb{N}_0$,
\begin{eqnarray*}
\E \left[ \zeta_t(\theta_t, z_{0:t}) \right] \leq \abs{\E\left[ \zeta_t(\theta_t, z_{0:t}) \right] - \E[\zeta_t(\theta_0, z_{0:t})]} + \abs{\E[\zeta_t(\theta_0, z_{0:t})] - \E[\zeta_t(\theta_0, z_{-\infty:t})]} + \abs{\E[\zeta_t(\theta_0, z_{-\infty:t})]}.
\end{eqnarray*}
\underline{\it Step 1: Use regularity properties of the error function to upper bound the first two terms.} \vspace{1mm}\\
Using parts (b), (c) of Lemma \ref{lemma:lip_lambda} and following the arguments shown in Step 1, 2 of case (a) above, we get
\begin{eqnarray}
\label{eq:decomp_b_1_2}
\abs{\E\left[ \zeta_t(\theta_t, z_{0:t}) \right] - \E[\zeta_t(\theta_0, z_{0:t})]} + \abs{\E[\zeta_t(\theta_0, z_{0:t})] - \E[\zeta_t(\theta_0, z_{-\infty:t})]} \leq 6B^2 \sum_{i=0}^{t-1} \alpha_i + B^2 (\gamma \lambda)^t.
\end{eqnarray}
\underline{\it Step 2: Characterizing $\E[\zeta_{t}(\theta,z_{-\infty:t})]$ for any fixed (non-random) $\theta$.} \vspace{1mm}\\
Recall the definition of $\bar{x}(\theta)$ from Equation \eqref{eq:expec_grad_TD_lambda}. For any fixed (non-random) $\theta$, we have $\bar{x}(\theta) = \E[\delta_t(\theta)z_{-\infty:t}]$. Therefore,
\begin{eqnarray}
\label{eq:decomp_b_3}
\E[\zeta_{t}(\theta_0,z_{-\infty:t})] = \left( \E[\delta_t(\theta_0)z_{-\infty:t}] - \bar{x}(\theta_0) \right)^\top (\theta_0 - \theta^*) = 0.  
\end{eqnarray}
\underline{\it Step 3. Combine terms to show parts (b), (c) of our claim.} \vspace{1mm}\\
Combining Equations \eqref{eq:decomp_b_1_2} and \eqref{eq:decomp_b_3} establishes part (c) which states,
\[
\E \left[ \zeta_t(\theta_t, z_{0:t}) \right] \leq 6B^2 \sum_{i=0}^{t-1} \alpha_i + B^2 (\gamma \lambda)^t \qquad \forall \,\,\, t \in \mathbb{N}_0.
\]
We establish part (b) by using that the step-size sequence is non-increasing which implies: $\sum_{i=0}^{t-1} \alpha_i \leq t\alpha_0$. For all $t \leq 2\tau^{\mix}_{\lambda}(\alpha_T)$, we have the following loose upper bound.
\[
\mathbb{E}\left[ \zeta_t(\theta_t, z_{0:t}) \right] \, \leq \, 6B^2t\alpha_0 + B^2(\gamma \lambda)^{t} \, \leq \,6B^2\left(1+2\tau^{\mix}_{\lambda}(\alpha_T))\right)\alpha_0 + B^2(\gamma \lambda)^{t}.
\]
\end{proof}

%----------------------------------------------------------------------------------------%
\subsection{Proof of Theorem \ref{thrm:lambda_mc_ana}}
\label{appendix_b:proof_mc_lambda}
In this subsection, we establish convergence bounds for Projected TD($\lambda$) as stated in Theorem \ref{thrm:lambda_mc_ana} using Lemmas \ref{lemma:mc_lambda_grad_decomp} and \ref{lemma:lambda_grad_bias_bd}. From Lemma \ref{lemma:mc_lambda_grad_decomp} we have,
\begin{eqnarray}
\label{eq:before_proof_td_lambda}
\mathbb{E} \left[ \norm{\theta^* - \theta_{t+1}}_2^2 \right] \leq \mathbb{E} \left[ \norm{\theta^* - \theta_t}_2^2 \right] - 2 \alpha_t (1-\kappa) \mathbb{E}\left[\norm{V_{\theta^*} - V_{\theta_t} }^2_{D} \right] + 2\alpha_t \mathbb{E}\left[ \zeta_t(\theta_t, z_{0:t}) \right] + \alpha_t^2 B^2. 
\end{eqnarray}
Equation \eqref{eq:before_proof_td_lambda} will be used as a starting point for analyzing different step-size choices.
%-------------------------%
\paragraph{Proof of part (a):} Fix a constant step-size of $\alpha_0 = \ldots = \alpha_t = 1/\sqrt{T}$ in Equation \eqref{eq:before_proof_td_lambda}, rearrange terms and sum from $t=0$ to $t=T-1$, we get
\begin{eqnarray*}
2 \alpha_0 (1-\kappa) \sum_{t=0}^{T-1} \mathbb{E}\left[\norm{V_{\theta^*} - V_{\theta_t}}^2_{D} \right] \leq \sum_{t=0}^{T-1} \left( \mathbb{E} \left[ \norm{\theta^* - \theta_t}_2^2 \right] - \mathbb{E}\left[ \norm{\theta^* - \theta_{t+1}}_2^2 \right] \right) + B^2 + 2\alpha_0 \sum_{t=0}^{T-1} \mathbb{E}\left[ \zeta_t(\theta_t, z_{0:t}) \right]. 
\end{eqnarray*}
Using Lemma \ref{lemma:lambda_grad_bias_bd} where $\alpha_{t-2\tau^{\mix}_{\lambda}(\alpha_T)}=\alpha_0$ along with the fact that $(\gamma \lambda) < 1$, we simplify to get
\begin{eqnarray*}
\sum_{t=0}^{T-1} \mathbb{E}\left[\norm{V_{\theta^*} - V_{\theta_t}}^2_{D} \right] &\leq& 
\frac{\norm{\theta^* - \theta_0}_2^2 + B^2}{2\alpha_0 (1-\kappa)} + \frac{T \cdot 6B^2(1+2\tau^{\mix}_{\lambda}(1/\sqrt{T})) \alpha_0}{(1-\kappa)} + \frac{1}{(1-\kappa)}\sum_{t=0}^{2\tau^{\mix}_{\lambda}(1/\sqrt{T})} B^2(\gamma \lambda)^t \\
&\leq& \frac{\sqrt{T}\left(\norm{\theta^* - \theta_0}_2^2 + B^2\right)}{2(1-\kappa)} +  \frac{\sqrt{T} \cdot 6B^2(1+2\tau^{\mix}_{\lambda}(1/\sqrt{T}))}{(1-\kappa)} + \frac{2B^2 \tau^{\mix}_{\lambda}(1/\sqrt{T})}{(1-\kappa)}.
\end{eqnarray*}
Adding these terms, we conclude
\begin{eqnarray*}
\mathbb{E}\left[\norm{V_{\theta^*} - V_{\bar{\theta}_T}}^2_{D}\right] \leq \frac{1}{T} \sum_{t=0}^{T-1} \mathbb{E}\left[\norm{V_{\theta^*} - V_{\theta_t}}^2_{D}\right] &\leq& \frac{\norm{\theta^* - \theta_0}_2^2 + B^2 \left(13+28\tau^{\mix}_{\lambda}(1/\sqrt{T})\right)}{2\sqrt{T}(1-\kappa)}.
\end{eqnarray*}
%-----------------------------%
\paragraph{Proof of part (b):} For a constant step-size of $\alpha_0 < 1/(2 \omega  (1-\kappa))$, we show that the expected distance between the iterate $\theta_T$ and the TD($\lambda$) limit point, $\theta^*$ converges at an exponential rate below some level that depends on the choice of step-size and $\lambda$. Starting with Equation \eqref{eq:before_proof_td_lambda} and applying Lemma \ref{lemma:strong_conv} which shows that $\norm{V_{\theta^*} - V_{\theta}}^2_{D} \geq w \norm{\theta^* - \theta}_2^2$ for any $\theta$, we have that for all $t > 2 \tau^{\mix}_{\lambda}(\alpha_0)$,
\begin{eqnarray*}
\mathbb{E} \left[\norm{\theta^*-\theta_{t+1}}_2^2\right] &\leq& \left(1-2\alpha_0(1-\kappa)\omega \right) \mathbb{E} \left[\norm{\theta^*-\theta_t}_2^2\right] + \alpha_0^2 B^2 +  2\alpha_0 \mathbb{E}\left[ \zeta_t(\theta_t, z_{0:t}) \right] \nonumber \\
&\leq& \left(1-2\alpha_0(1-\kappa)\omega \right) \mathbb{E} \left[\norm{\theta^*-\theta_t}_2^2\right] + \alpha_0^2 B^2 \left(13+24\tau^{\mix}_{\lambda}(\alpha_0) \right), \label{eq:theta_rec}
\end{eqnarray*}
where we used part (a) of Lemma \ref{lemma:lambda_grad_bias_bd} for the second inequality. Iterating over it gives us our final result. For any $T > 2 \tau^{\mix}_{\lambda}(\alpha_0)$,  
\begin{eqnarray*}
\mathbb{E} \left[\norm{\theta^*-\theta_{T}}_2^2\right] &\leq& \left(1-2\alpha_0(1-\kappa)\omega \right)^T \norm{\theta^*-\theta_0}_2^2 + B^2 \Big(\alpha_0^2 \left(13+24\tau^{\mix}_{\lambda}(\alpha_0)\right) \Big) \sum_{t=0}^{\infty}\left(1-2\alpha_0(1-\kappa)\omega \right)^t \\
&\leq& \left(e^{-2\alpha_0(1-\kappa)\omega T}\right) \norm{\theta^*-\theta_0}_2^2 + \frac{B^2\Big(\alpha_0\left(13+24\tau^{\mix}_{\lambda}(\alpha_0) \right) \Big)}{2(1-\kappa)\omega}.
\end{eqnarray*}
The second inequality follows by solving the geometric series and using that $\left(1-2\alpha_0(1-\kappa)\omega \right) \leq e^{-2\alpha_0(1-\kappa)\omega}$.
% By definition, for any $T \geq \tau^{\mix}(\alpha_0)$, we have $(\gamma \lambda)^T \leq \alpha_0$ which implies that $2\alpha_0B^2 (\gamma \lambda)^T \leq 2\alpha_0^2 B^2$. Using this, we get  
% \begin{eqnarray*}
% \mathbb{E} \left[\norm{\theta^*-\theta_{T}}_2^2\right] &\leq& \left(1-2\alpha_0(1-\kappa)\omega \right)^T \norm{\theta^*-\theta_0}_2^2 + \alpha_0^2 B^2 \left(15+16\tau^{\mix}(\alpha_0) \right) \sum_{t=0}^{\infty}\left(1-2\alpha_0(1-\kappa)\omega \right)^t \\
% &\leq& \left(e^{-2\alpha_0(1-\kappa)\omega T}\right) \norm{\theta^*-\theta_0}_2^2 + \frac{\alpha_0 B^2\left(15+16\tau^{\mix}(\alpha_0) \right)}{2(1-\kappa)\omega}.
% \end{eqnarray*}
% The second inequality above follows by solving the geometric series and then using the fact that $\left(1-2\alpha_0(1-\kappa)\omega \right) \leq e^{-2\alpha_0(1-\kappa)\omega}$.
%----------------------------%
\paragraph{Proof of part (c):} Consider a decaying step-size of $\alpha_t = 1/(\omega (t+1) (1-\kappa))$. We start with Equation \eqref{eq:before_proof_td_lambda} and use Lemma \ref{lemma:strong_conv} which showed $\mathbb{E}\left[\norm{V_{\theta^*} - V_{\theta}}^2_{D}\right] \geq w \mathbb{E}\left[ \norm{\theta^* - \theta}_2^2 \right]$ for all $\theta$ to get,
\begin{eqnarray*}
\mathbb{E}\left[\norm{V_{\theta^*} - V_{\theta_t}}^2_{D}\right] \leq \frac{1}{(1-\kappa)\alpha_t} \left((1-(1-\kappa)\omega \alpha_t) \mathbb{E}\left[\norm{\theta^* - \theta_t}_2^2\right] - \mathbb{E}\left[\norm{\theta^* - \theta_{t+1}}_2^2\right] + \alpha_t^2 B^2 \right) + \\
\frac{2}{(1-\kappa)} \mathbb{E}\left[ \zeta_t(\theta_t, z_{0:t}) \right]. 
\end{eqnarray*}
Substituting $\alpha_t = \frac{1}{\omega (t+1) (1-\kappa)}$, simplify and sum from $t=0$ to $T-1$ to get,
\begin{eqnarray}
\sum_{t=0}^{T-1} \mathbb{E}\left[\norm{V_{\theta^*} - V_{\theta_t}}^2_{D}\right] &\leq& -\omega T \mathbb{E}\left[ \norm{\theta^* - \theta_T}_2^2 \right] + \frac{B^2}{\omega (1-\kappa)^2} \sum_{t=0}^{T-1} \frac{1}{t+1} + \frac{2}{(1-\kappa)} \sum_{t=0}^{T-1} \mathbb{E}\left[ \zeta_t(\theta_t, z_{0:t}) \right] \nonumber\\ 
&\leq& \underbrace{-\omega T \mathbb{E}\left[ \norm{\theta^* - \theta_T}_2^2 \right]}_{< 0} + \frac{B^2 (1+ \log T)}{\omega (1-\kappa)^2} + \frac{2}{(1-\kappa)} \sum_{t=0}^{T-1} \mathbb{E}\left[ \zeta_t(\theta_t, z_{0:t}) \right], \label{eq:lambda_mc_ana_rec_partb}
\end{eqnarray}
where we used that $\sum_{t=0}^{T-1} \frac{1}{t+1} \leq (1+\log T)$. To simplify notation, we put $\tau = \tau^{\mix}_{\lambda}(\alpha_T)$ for the remainder of the proof. We use Lemma \ref{lemma:lambda_grad_bias_bd} 
to upper bound the total bias, $\sum_{t=0}^{T-1} \mathbb{E}\left[ \zeta_t(\theta_t, z_{0:t}) \right]$ which can be decomposed as:
\begin{eqnarray}
\sum_{t=0}^{T-1} \mathbb{E}\left[ \zeta_t(\theta_t, z_{0:t}) \right] = \sum_{t=0}^{2\tau} \mathbb{E}\left[ \zeta_t(\theta_t, z_{0:t}) \right] + \sum_{t=2\tau + 1}^{T-1} \mathbb{E}\left[ \zeta_t(\theta_t, z_{0:t}) \right].
\end{eqnarray}
First, note that for a decaying step-size $\alpha_t = \frac{1}{\omega (t+1) (1-\gamma)}$ we have
\[
\sum_{t=0}^{T-1} \alpha_t = \frac{1}{\omega (1-\gamma)} \sum_{t=0}^{T-1} \frac{1}{(t+1)} \leq \frac{1 + \log T}{\omega (1-\gamma)}.
\]
We will combine this with Lemma \ref{lemma:lambda_grad_bias_bd} to upper bound each term separately. First,
\begin{eqnarray*}
\sum_{t=0}^{2\tau} \mathbb{E}\left[ \zeta_t(\theta_t, z_{0:t})\right] &\leq& \sum_{t=0}^{2\tau} \left( 6B^2 \sum_{i=0}^{t-1} \alpha_i \right) + \sum_{t=0}^{2\tau} B^2(\gamma \lambda)^t \\
&\leq& \frac{6B^2}{\omega(1-\kappa)} \sum_{t=0}^{2\tau} \sum_{i=0}^{T-1} \frac{1}{(i+1)} + 2B^2 \tau \,\, \leq \,\, \frac{14B^2 \tau}{\omega(1-\kappa)} (1 + \log T),
\end{eqnarray*}
where we used the fact that $\omega, \kappa, (\gamma \lambda) < 1$. Similarly, 
\begin{eqnarray*}
\sum_{t=2\tau+1}^{T-1} \mathbb{E}\left[ \zeta_t(\theta_t, z_{0:t}) \right] &\leq& 6B^2 (1 + 2 \tau) \sum_{t=2\tau+1}^{T-1} \alpha_{t-2\tau} \leq 6B^2 (1 + 2\tau) \sum_{t=0}^{T-1} \alpha_t \leq \frac{6B^2(1 + 2\tau)}{\omega(1-\kappa)} \left(1 + \log T \right).
\end{eqnarray*}
Combining the two parts, we get 
\begin{eqnarray*}
\sum_{t=0}^{T-1} \mathbb{E}\left[ \zeta_t(\theta_t, z_{0:t})\right] \leq \frac{B^2 (6+26\tau)}{\omega(1-\kappa)} (1 + \log T).
\end{eqnarray*}
Using this in conjunction with Equation \eqref{eq:lambda_mc_ana_rec_partb} we get our final result,
\begin{eqnarray*}
\mathbb{E}\left[\norm{V_{\theta^*} - V_{\bar{\theta}_T}}^2_{D}\right] \leq \frac{1}{T} \sum_{t=0}^{T-1} \mathbb{E}\left[\norm{V_{\theta_t} - V_{\theta^*}}^2_{D}\right] \leq \frac{B^2}{\omega T(1-\kappa)^2} \left(1+\log T\right) + \frac{2}{T(1-\kappa)} \sum_{t=0}^{T-1} \mathbb{E}\left[ \zeta_t(\theta_t, z_{0:t}) \right].
\end{eqnarray*}
Simplifying and putting back $\tau = \tau^{\mix}_{\lambda}(\alpha_T)$, we get 
\begin{eqnarray*}
\mathbb{E}\left[\norm{V_{\theta^*} - V_{\bar{\theta}_T}}^2_{D}\right] &\leq& \frac{B^2}{\omega T(1-\kappa)^2} \left(1+\log T\right) + \frac{2B^2 \left(6+26\tau^{\mix}_{\lambda}(\alpha_T)\right)}{\omega T(1-\kappa)^2} (1 + \log T) \\
&\leq& \frac{ B^2 \left(13+52\tau^{\mix}(\alpha_T)\right)}{\omega T(1-\kappa)^2}(1+\log T).
\end{eqnarray*}
Additionally, Equation \eqref{eq:lambda_mc_ana_rec_partb} implies a convergence rate of $\mathcal{O}(\log T/T)$ for the iterate $\theta_T$ itself:
\begin{eqnarray*}
\label{eq:mc_theta_result}
\mathbb{E}\left[\norm{\theta^* - \theta_T}_2^2 \right] &\leq& \frac{ B^2 \left(13+52\tau^{\mix}_{\lambda}(\alpha_T)\right)}{\omega^2 T(1-\kappa)^2} \left(1+\log T\right).
\end{eqnarray*}

%-----------------------------------------------------------------------------------------------------------------%
%-----------------------------------------------------------------------------------------------------------------%
\subsection{Proof of supporting lemmas.}
\label{appendix_b:proof_supp_lemma_lambda}
In this subsection, we provide standalone proofs of Lemma \ref{lemma:norm_grad_lambda} and \ref{lemma:lip_lambda} used above. 
\lemmasixth*
\begin{proof}
We start with the mathematical expression for $x_t(\theta, z_{0:t})$.
\begin{eqnarray*}
x_t(\theta, z_{0:t}) = \delta_t(\theta) z_{0:t} \,\, \Rightarrow \,\, \norm{x_t(\theta, z_{0:t})}_2 = \abs{\delta_t(\theta)} \norm{z_{0:t}}_2.
\end{eqnarray*}
We give an upper bound on both $\abs{\delta_t(\theta)}$ and $\norm{z_{0:t}}_2$. Starting with the definition of $\delta_t(\theta)$ and using that $\norm{\phi(s_t)}_2 \leq 1 \,\, \forall \,\, t$ along with $\norm{\theta}_2 \leq R$, we get
\begin{eqnarray*}
\abs{\delta_t(\theta)} =\abs{ r_t + \gamma \phi(s_{t}')^\top \theta - \phi(s_t)^\top \theta } \leq r_{\max} + \norm{\phi(s_{t}')}_2 \norm{\theta}_2 + \norm{\phi(s_t)}_2 \norm{\theta}_2 \leq \left( r_{\max} + 2 R \right).
\end{eqnarray*}
Next, 
\begin{eqnarray*}
\norm{z_{0:t}}_2^2 = \norm{\sum_{k=0}^t (\gamma \lambda)^{k} \phi(s_{t-k})}_2^2 \leq \left(\sum_{k=0}^t (\gamma \lambda)^{k}\right)^2 \leq \left(\sum_{k=0}^\infty (\gamma \lambda)^{k}\right)^2 = \frac{1}{(1-\gamma \lambda)^2}.
\end{eqnarray*}
Combining these two implies the first part of our claim.
\begin{eqnarray*}
\norm{x_t(\theta, z_{0:t})}_2 = \abs{\delta_t(\theta)} \norm{z_{0:t}}_2 \leq \frac{\left( r_{\max} + 2 R \right)}{(1-\gamma \lambda)} = B.
\end{eqnarray*}
Note that can easily show an upper bound $\norm{\delta_t(\theta)z_{l:t}}_2 \leq B$ for any pair $(\theta, z_{l:t})$ with $l \leq t$. Consider,
\begin{eqnarray*}
\norm{z_{l:t}}_2^2 &\leq& \norm{z_{-\infty:t}}_2^2 \leq \left(\sum_{k=0}^\infty (\gamma \lambda)^{k}\right)^2 = \frac{1}{(1-\gamma \lambda)^2} \\
\Rightarrow \,\, \norm{\delta_t(\theta)z_{l:t}}_2 &=& \abs{\delta_t(\theta)} \norm{z_{l:t}}_2 \leq \frac{\left( r_{\max} + 2 R \right)}{(1-\gamma \lambda)} = B.
\end{eqnarray*}
Taking $l \rightarrow -\infty$ implies that $\norm{\delta_t(\theta)z_{-\infty:t}}_2 \leq B$. As $\bar{x}(\theta) = \E \left[ \delta_t(\theta)z_{-\infty:t} \right]$, we also have a uniform norm bound on the expected updates, $\norm{\bar{x}(\theta)}_2 \leq B$, as claimed.
\end{proof}

%--------------------------------------------------------------------------------------------------------------------%
\lemmaliplambda*
\begin{proof}
\,Throughout, we use the assumption that basis vectors are normalized i.e. $\norm{\phi(s_t)}_2 \leq 1 \,\, \forall \,\, t$.
%-----------------------------------------------------------------------------------%
\paragraph{Part (a):} We show a uniform norm bound on $\zeta_t(\theta, z_{l:t}) \,\, \forall \,\, \theta \in \Theta_{R}$. First consider the following:
\begin{eqnarray*}
\norm{\delta_t(\theta)z_{l:t}}_2 = \abs{\delta_t(\theta)} \norm{z_{l:t}}_2 &\leq& \abs{r_t + \gamma \phi(s'_t)^\top\theta - \phi(s_t)^\top\theta}  \norm{\sum_{k=0}^{t-l} (\gamma \lambda)^k \phi(s_{t-k})}_2 \\
&\leq& \abs{r_t + \norm{\phi(s_{t}')}_2 \norm{\theta}_2 + \norm{\phi(s_t)}_2 \norm{\theta}_2} \norm{\sum_{k=0}^{\infty} (\gamma \lambda)^k \phi(s_{t-k})}_2 \\
&\leq& \frac{\left( r_{\max} + 2 R \right)}{(1-\gamma\lambda)} = B.
\end{eqnarray*}
Using this along with the fact that $\norm{\theta - \theta^*}_2 \leq 2R \leq B$ and $\norm{\bar{x}(\theta)}_2 \leq B \,\, \text{for all} \,\, \theta \in \Theta_R$, we get
\begin{eqnarray*}
\abs{\zeta_t (\theta, z_{l:t})} = \abs{\left( \delta_t(\theta) z_{l:t} - \bar{x}(\theta) \right)^\top (\theta - \theta^*)} &\leq& \norm{\delta_t(\theta) z_{l:t} - \bar{x}(\theta)}_2 \norm{(\theta - \theta^*)}_2 \\
&\leq& \left(\norm{\delta_t(\theta) z_{l:t}}_2 + \norm{\bar{x}(\theta)}_2\right) \norm{(\theta - \theta^*)}_2 \\
&\leq& 2B \norm{(\theta - \theta^*)}_2 \leq 2B^2.
\end{eqnarray*}

%-----------------------------------------------------------------------------------%
\paragraph{Part (b):} To show that $\zeta_t(\cdot, z_{l:t})$ is $L$-Lipschitz, consider the following inequality for any four vectors $(a_1, b_1, a_2 ,b_2)$, which follows as a direct application of Cauchy-Schwartz.
\[
\abs{a_1^\top b_1-a_2^\top b_2} = \abs{a_1^\top(b_1-b_2) + b_2^\top(a_1-a_2)} \leq \norm{a_1}_2\norm{b_1-b_2}_2 + \norm{b_2}_2\norm{a_1-a_2}_2.
\]
This implies, 
\begin{eqnarray*}
\abs{\zeta_{t}(\theta, z_{l:t}) - \zeta_{t}(\theta', z_{l:t})} &=& \abs{\left(\delta_t(\theta)z_{l:t} - \bar{x}(\theta)\right)^\top(\theta - \theta^*) -  \left(\delta_t(\theta')z_{l:t} - \bar{x}(\theta')\right)^\top(\theta' - \theta^*)} \nonumber \\
&\leq& \norm{\delta_t(\theta)z_{l:t} - \bar{x}(\theta)}_2 \norm{\theta - \theta'}_2 + \norm{\theta' - \theta^*}_2 \norm{\left(\delta_t(\theta)z_{l:t} - \bar{x}(\theta)\right) -  \left(\delta_t(\theta')z_{l:t} - \bar{x}(\theta')\right)}_2 \nonumber \\
&\leq& 2B\norm{\theta - \theta'}_2 + 2R \Big[ \norm{z_{l:t} \left(\delta_t(\theta) - \delta_t(\theta')\right)}_2 +  \norm{\bar{x}(\theta) - \bar{x}(\theta')}_2 \Big] \label{eq:lambda_lip_1}\\
&\leq& 2B\norm{\theta - \theta'}_2 + \frac{8R}{(1-\gamma \lambda)} \norm{\theta - \theta'}_2 \label{eq:lambda_lip_2}\\
&\leq& 6B\norm{\theta - \theta'}_2 \nonumber ,
\end{eqnarray*}
where the last inequality follows as $\frac{R}{1-\gamma\lambda} \leq B/2$ by definition. In the penultimate inequality, we used that $\norm{z_{l:t} \left(\delta_t(\theta) - \delta_t(\theta')\right)}_2 \leq \frac{2}{(1-\gamma\lambda)} \norm{\theta - \theta'}_2$ which is easy to prove. Consider,
\begin{eqnarray*}
\norm{z_{l:t} \left(\delta_t(\theta) - \delta_t(\theta')\right)}_2 &\leq& \norm{z_{l:t}}_2 \abs{\left(\delta_t(\theta) - \delta_t(\theta')\right)} \\
&\leq& \norm{\sum_{k=0}^\infty (\gamma\lambda)^k \phi(s_{t-k})}_2 \abs{\left(\delta_t(\theta) - \delta_t(\theta')\right)} \\
&\leq& \frac{1}{(1-\gamma\lambda)}
\abs{\left(\gamma \phi(s'_t) - \phi(s_t) \right)^\top (\theta - \theta')} \\
&\leq& \frac{\left(\norm{\phi(s'_t)}_2 + \norm{\phi(s_t)}_2\right)}{(1-\gamma\lambda)} \norm{\theta - \theta'}_2 \leq \frac{2}{(1-\gamma\lambda)} \norm{\theta - \theta'}_2.
\end{eqnarray*}
As $\bar{x}(\theta) = \E \left[ \delta_t(\theta) z_{-\infty:t} \right]$, this also implies that $\norm{\bar{x}(\theta) - \bar{x}(\theta')}_2 \leq \frac{2}{(1-\gamma\lambda)} \norm{\theta - \theta'}_2$ which completes the proof.

%-----------------------------------------------------------------------------------%
\paragraph{Part (c):} To show that $\abs{\zeta_t(\theta,z_{0:t}) - \zeta_t(\theta,z_{t-\tau:t})} \leq B^2(\gamma \lambda)^\tau \,\,\text{for all} \,\, \theta \in \Theta_R$ and $\tau \leq t$, we use that $\norm{\theta - \theta^*}_2 \leq 2R \leq B$. 
\begin{eqnarray*}
\abs{\zeta_t(\theta,z_{0:t}) - \zeta_t(\theta,z_{t-\tau:t})} &=& \abs{\left(\delta_t(\theta) z_{0:t} - \delta_t(\theta) z_{t-\tau:t} \right)^\top (\theta - \theta^*)} \\
&\leq& \abs{\delta_t(\theta)} \norm{z_{0:t} - z_{t-\tau:t}}_2 \norm{\theta - \theta^*}_2 \\
&\leq& \abs{r_t + \gamma \phi(s'_t)^\top\theta - \phi(s_t)^\top\theta} \norm{\sum_{k=\tau}^\infty (\gamma \lambda)^k \phi(s_{t-k})}_2 B \\
&\leq& \abs{r_t + 2 \norm{\theta}_2} \cdot \frac{(\gamma \lambda)^\tau}{(1 - \gamma \lambda)} \cdot B \\
&\leq& B \frac{(r_{\max} + 2 R)}{(1 - \gamma \lambda)} (\gamma \lambda)^\tau = B^2 (\gamma \lambda)^\tau.
\end{eqnarray*}
Similarly, 
\begin{eqnarray*}
\abs{\zeta_t(\theta, z_{0:t}) - \zeta_t(\theta, z_{-\infty:t})} &\leq& \abs{\delta_t(\theta) (z_{0:t} - z_{-\infty:t})^\top (\theta - \theta^*)} \\
&\leq& \abs{\delta_t(\theta)} \norm{z_{0:t} - z_{-\infty:t}}_2 \norm{\theta - \theta^*}_2 \\
&\leq& \abs{\left( r_t + \gamma \phi(s'_{t})^\top \theta - \phi(s_t)^\top \theta \right)} \norm{\sum_{k=t}^\infty (\gamma \lambda)^k \phi(s_{t-k})}_2 B \\
&\leq& B \frac{(r_{\max} + 2R)}{(1-\gamma \lambda)} (\gamma \lambda)^{t} \leq B^2(\gamma \lambda)^{t}. 
\end{eqnarray*}
\end{proof}

\newpage
%-------------------------------------------------
\section{Proofs of Lemmas \ref{lemma:expected_gradient_opt_stopping} and \ref{lemma:expected_grad_td_lambda}} \label{appendix:a1}
%-------------------------------------------------
In this section, we give a combined proof of Lemmas \ref{lemma:expected_gradient_opt_stopping} and \ref{lemma:expected_grad_td_lambda} which quantify the progress of the expected updates towards the limit point $\theta^*$ for TD($\lambda$) and the Q-function approximation algorithm. These lemmas can be stated more generally as shown below, instead of using the Bellman operators $F(\cdot)$ and $T^{(\lambda)}(\cdot)$.

\begin{restatable}[]{lem}{lemmageneralexoectedgrad}
\label{lemma:expected_gradient_general}
Let $\Pi_D H(\cdot)$ be a contraction with respect to $\norm{\cdot}_D$ with modulus $\gamma$ and let $V_{\theta^{*}}$ be the unique fixed point of $\Pi_{D} H(\cdot)$, i.e. $V_{\theta^*} = \Pi_{D}H V_{\theta^*}$. Define $\bar{g}(\theta) = \Phi^\top D \left( H \Phi \theta - \Phi \theta \right)$ for all $\theta\in \mathbb{R}^d$ to be the expected update. Then, 
\begin{equation*}
(\theta^* - \theta)^\top \bar{g}(\theta) \,\, \geq \,\, (1-\gamma) \norm{V_{\theta^*} - V_{\theta}}^2_{D}.
\end{equation*}
\end{restatable}
\begin{proof}
We have
\begin{eqnarray}
(\theta^{*} - \theta)^\top \bar{g}(\theta) &=&(\theta^{*}-\theta) ^\top \Phi^\top D \left( H \Phi \theta - \Phi \theta \right) \nonumber \\
&=& \langle \Phi (\theta^{*}-\theta) , \left( H \Phi \theta - \Phi \theta \right) \rangle_{D} \nonumber \\
&=& \langle \Pi_D \Phi (\theta^{*}-\theta ) , \left( H \Phi \theta - \Phi \theta \right) \rangle_{D}  \label{eq:vr_0} \\
&=& \langle \Phi (\theta^{*}-\theta)  , \Pi_D\left( H \Phi \theta - \Phi \theta \right) \rangle_{D} \label{eq:vr_1} \\
&=& \langle \Phi (\theta^{*}-\theta)  , \Pi_D H \Phi \theta - \Phi \theta \rangle_{D} \nonumber \\
&=& \langle \Phi (\theta^{*}-\theta)  , \Pi_D H \Phi \theta - \Phi\theta^*  +\Phi\theta^* - \Phi \theta \rangle_{D} \nonumber \\
&=& \| \Phi(\theta^* -\theta) \|_D^2 - \langle \Phi (\theta^{*}-\theta)  , \Phi\theta^* - \Pi_D H \Phi \theta \rangle_{D}  \nonumber \\
&\geq&  \| \Phi(\theta^* -\theta) \|_D^2 -   \| \Phi(\theta^* -\theta) \|_D \cdot  \|\Pi_D H \Phi \theta - \Phi\theta^*\|_{D}  \nonumber \\
&\geq& \| \Phi(\theta^* -\theta) \|_D^2  - \gamma \cdot \norm{\Phi (\theta^{*}-\theta)  }_{D}^2 \label{eq:vr_2} \\
&=& (1-\gamma) \cdot \norm{\Phi (\theta^{*}-\theta)  }_{D}^2 = (1-\gamma) \cdot \norm{V_{\theta^*} - V_{\theta}}^2_{D}, \nonumber
\end{eqnarray}
where in going to Equation \eqref{eq:vr_0}, we used that $\forall \,\, \mathbf{x} \in \text{Span}(\Phi)$, we have $\Pi_D \, \mathbf{x}=\mathbf{x}$. In Equation \eqref{eq:vr_1}, we used that the projection matrix $\Pi_D$ is symmetric. In going to Equation \eqref{eq:vr_2}, we used that that $\Pi_D H(\cdot)$ is a contraction operator with modulus $\gamma$ with $\Phi\theta^*$ as its fixed point, which implies that $\norm{\Pi_D H \Phi \theta -\Phi \theta^{*}}_D = \norm{\Pi_D H \Phi \theta - \Pi_{D} H \Phi \theta^{*}}_D \leq \gamma \norm{\Phi \theta - \Phi \theta^*}_D$.
\end{proof}
%--------------------------------------------------------------------------------------------%
\end{document}